\documentclass[11pt]{article}
\usepackage[margin=1.0in]{geometry}
\usepackage{graphicx}
\usepackage{amsmath,amssymb,amsfonts,amsthm,mathrsfs}
\usepackage{color}
\usepackage[small,bf]{caption}
\usepackage{makecell}
\newcommand{\mylabel}[2]{#2\def\@currentlabel{#2}\label{#1}}
\makeatother
\usepackage{xcolor}
\definecolor{dgreen}{rgb}{0,0.5,0}
\usepackage[colorlinks=true,linkcolor=blue,citecolor=dgreen]{hyperref}
\usepackage{cases}
\usepackage{xspace}
\usepackage{adjustbox}
\usepackage{multirow}
\usepackage{subcaption}

\usepackage{enumitem}
\usepackage[normalem]{ulem}
\setlist{nosep}

\newcommand{\nc}{\newcommand}
\nc{\DMO}{\DeclareMathOperator}
\nc{\st}{\star}
\nc\m[2]{m_{#1}(#2)}
\nc{\ReduceTree}{\texttt{ReduceTree}\xspace}
\nc{\PolyPriLearn}{\texttt{PolyPriLearn}\xspace}
\nc{\PPPLearn}{\texttt{PolyPriPropLearn}\xspace}
\nc{\GenericLearner}{\texttt{GenericLearner}\xspace}
\nc{\BR}{\mathbb{R}}
\nc{\BM}{\mathbb{M}}
\nc{\BT}{\mathbb{T}}
\nc{\BN}{\mathbb{N}}
\nc{\BZ}{\mathbb{Z}}
\nc{\ep}{\varepsilon}
\DMO{\height}{ht}
\renewcommand{\epsilon}{\varepsilon}
\nc{\ra}{\rightarrow}
\DMO{\Err}{err}
\DMO{\Est}{Est}
\DMO{\good}{good}
\DMO{\negpt}{neg-pt}
\DMO{\VV}{{V}}
\DMO{\LL}{{L}}
\DMO{\finsupp}{fin}
\nc{\fin}{{\finsupp}}
\nc{\err}[2]{\Err_{#1}(#2)}
\nc{\rca}{\mathscr{B}}
\nc{\bt}{b}
\nc{\hMLp}{\hat\ML'}

\nc{\Rprot}{R}
\nc{\Sprot}{S}
\nc{\Aprot}{A}
\nc{\Pprot}{P}
\nc{\Pdist}{D}
\nc{\Qdist}{F}
\nc{\pdist}{d}
\nc{\qdist}{f}
\renewcommand{\^}[1]{^{(#1)}}
\nc{\DP}{differentially private\xspace}

\nc{\SD}{\mathscr{D}}
\nc{\la}{\lambda}
\DMO{\KL}{KL}
\DMO{\Unif}{Unif}
\nc{\nn}{\varnothing}
\DMO{\SOA}{SOA}
\nc{\soa}[2]{\SOA_{#1}(#2)}
\nc{\soaf}[1]{\SOA_{#1}}
\nc{\gRes}[2]{\hat\MG({#1},{#2})}
\nc{\emp}{\hat P_{S_n}}
\DMO{\Red}{red}
\DMO{\Irred}{irred}
\nc{\Ired}{I^{\Red}}
\nc{\Iirred}{I^{\Irred}}

\DMO{\ssmp}{ssmp}
\DMO{\agg}{agg}
\DMO{\final}{final}

\nc{\pp}{p}
\nc{\PP}{P}
\nc{\QQ}{Q}
\nc{\DD}{D}

\DMO{\RAPPOR}{{RAPPOR}}
\nc{\RAP}{\RAPPOR}
\DMO{\RR}{RR}
\nc{\MD}{\mathcal{D}}
\nc{\ML}{\mathcal{L}}
\nc{\di}{P}
\nc{\MO}{\mathcal{O}}
\nc{\MM}{\mathcal{M}}
\nc{\MZ}{\mathcal{Z}}
\nc{\MU}{\mathcal{U}}
\nc{\MP}{\mathcal{P}}
\nc{\poly}{\mathrm{poly}}
\DMO{\treesum}{TreeSum}
\DMO{\lapsum}{LapSum}
\DMO{\checksum}{CheckSum}
\nc{\MDts}{\MD_{\treesum}}
\nc{\MDls}{\MD_{\lapsum}}
\nc{\MDcs}{\MD_{\checksum}}
\nc{\MC}{\mathcal{C}}
\nc{\MT}{\mathcal{T}}
\nc{\MS}{\mathcal{S}}
\nc{\MX}{\mathcal{X}}
\nc{\MY}{\mathcal{Y}}
\nc{\MA}{\mathcal{A}}
\nc{\MB}{\mathcal{B}}
\nc{\MJ}{\mathcal{J}}
\nc{\MF}{\mathcal{F}}
\nc{\MG}{\mathcal{G}}
\nc{\MQ}{\mathcal{Q}}
\nc{\p}{\Pr}
\nc{\E}{\mathbb{E}}
\nc{\tablesize}{s}
\DMO{\Hist}{hist}
\DMO{\Reg}{Reg}
\nc{\hist}{\mathrm{hist}}

\nc{\ba}{\mathbf{a}}
\nc{\bx}{\mathbf{x}}
\nc{\bs}{\mathbf{s}}
\nc{\bv}{\mathbf{v}}
\nc{\bw}{\mathbf{w}}
\nc{\by}{\mathbf{y}}
\nc{\bz}{\mathbf{z}}

\DMO{\sr}{sr}
\DMO{\Med}{Med}
\DMO{\Ber}{Ber}
\DMO{\Bin}{Bin}
\DMO{\Had}{Had}

\nc{\ME}{\mathcal{E}}
\DMO{\View}{View}
\nc{\B}{B}
\nc{\M}{M}
\nc{\ha}{\kappa}
\nc{\hk}{k}
\DMO{\pre}{pre}
\nc{\MH}{\mathcal{H}}
\DMO{\Ldim}{Ldim}
\DMO{\Tdim}{Tdim}
\DMO{\sfat}{sfat}
\DMO{\fat}{fat}
\DMO{\vc}{VCdim}
\DMO{\FO}{FO}
\DMO{\CM}{CM}
\DMO{\hb}{\beta}
\nc{\MW}{\mathcal{W}}
\nc{\MV}{\mathcal{V}}
\nc{\MK}{\mathcal{K}}
\nc{\MN}{\mathcal{N}}

\nc{\BB}{\{0,1\}}
\nc{\bW}{\mathbf{W}}
\nc{\eell}{\ell}
\nc{\EELL}{L}
\nc{\q}{q}

\makeatletter
\newcommand{\customitem}[1]{%
\item[#1]\protected@edef\@currentlabel{#1}%
}
\makeatother

\DeclareMathOperator*{\argmax}{arg\,max}

\makeatletter
\newtheorem*{rep@theorem}{\rep@title}
\newcommand{\newreptheorem}[2]{%
\newenvironment{rep#1}[1]{%
 \def\rep@title{#2~\ref{##1}}%
 \begin{rep@theorem}}%
 {\end{rep@theorem}}}
\makeatother

\newtheorem{theorem}{Theorem}[section]
\newtheorem{corollary}[theorem]{Corollary}
\newtheorem{proposition}[theorem]{Proposition}
\newtheorem{lemma}[theorem]{Lemma}
\newtheorem{informal theorem}[theorem]{Informal Theorem}

\newtheorem{claim}[theorem]{Claim}
\newreptheorem{lemma}{Lemma}

\theoremstyle{definition}
\newtheorem{defn}{Definition}[section]

\newtheorem{remark}{Remark}[section]

\usepackage{xcolor}
\usepackage{cancel}

\newcount\Comments
\newcommand{\badih}[1]{\ifnum\Comments=1\textcolor{red}{[Badih: #1]}\fi}
\newcommand{\noah}[1]{\ifnum\Comments=1\textcolor{purple}{[Noah: #1]}\fi}
\newcommand{\pasin}[1]{\ifnum\Comments=1\textcolor{red}{[Pasin: #1]}\fi}
\newcommand{\ravi}[1]{\ifnum\Comments=1\textcolor{cyan}{[Ravi: #1]}\fi}

\usepackage[boxruled,vlined,nofillcomment,linesnumbered]{algorithm2e}
\usepackage{graphicx}
\usepackage{bbm}
\usepackage{mathabx}
\nc{\One}{\mathbbm{1}}
\SetKwProg{Fn}{}{\string:}{}

\usepackage{thm-restate}

\title{Sample-efficient proper PAC learning with \\ approximate differential privacy}
\author{Badih Ghazi\thanks{Google Research, Mountain View, CA. \texttt{badihghazi@gmail.com, ravi.k53@gmail.com, pasin@google.com}.} \hspace*{0.5cm}
Noah Golowich\thanks{MIT EECS, Cambridge, MA. Supported at MIT by a Fannie \& John Hertz Foundation Fellowship and an NSF Graduate Fellowship.   This work was done while interning at Google Research.  \texttt{nzg@mit.edu}.} \hspace*{0.5cm} 
Ravi Kumar\footnotemark[1]
\hspace*{0.5cm} 
Pasin Manurangsi\footnotemark[1]}
\date{\today}

\begin{document}
\maketitle

\begin{abstract}
In this paper we prove that the sample complexity of properly learning a class of Littlestone dimension $d$ with approximate differential privacy is $\tilde O(d^6)$, ignoring privacy and accuracy parameters. This result answers a question of Bun et al.~(FOCS 2020) by improving upon their upper bound of $2^{O(d)}$ on the sample complexity. Prior to our work, finiteness of the sample complexity for privately learning a class of finite Littlestone dimension was only known for improper private learners, and the fact that our learner is proper answers another question of Bun et al.,~which was also asked by Bousquet et al.~(NeurIPS 2020). Using machinery developed by Bousquet et al., we then show that the sample complexity of sanitizing a binary hypothesis class is at most polynomial in its Littlestone dimension and dual Littlestone dimension. This implies that a class is sanitizable if and only if it has finite Littlestone dimension. An important ingredient of our proofs is a new property of binary hypothesis classes that we call {\it irreducibility}, which may be of independent interest.
\end{abstract}

\section{Introduction}
Machine learning algorithms are often trained on datasets consisting of sensitive data, such as in medical or social network applications. Protecting the privacy of the users' data is  of importance, both from an ethical perspective~\cite{kearns_ethical_2019} and to maintain compliance with an increasing number of laws and regulations~\cite{article29,nissim_bridging_2016,cohen_towards_2020}. The notion of {\it differential privacy}~\cite{dwork2006calibrating, DworkRothBook,vadhan2017complexity} provides a formal framework for controlling the privacy-accuracy tradeoff in numerous settings involving private data release, and it has played a central role in the development of privacy-preserving algorithms. 

In the body of work on private learning algorithms, a significant amount of effort has gone into developing algorithms for the private PAC model~\cite{kasiviswanathan2008what}, namely the setting of differentially private binary classification (see Section~\ref{sec:pac} for a formal definition). Some papers on this fundamental topic include~\cite{kasiviswanathan2008what,beimel_bounds_2014,bun_differentially_2015,feldman_sample_2014,beimel_private_2014,bun_composable_2018,beimel_characterizing_2019,alon_private_2019,kaplan_privately_2020,bun_equivalence_2020,neel_heuristics_2019,bun_computational_2020}. 
A remarkable recent development~\cite{alon_private_2019,bun_equivalence_2020} in this area is the result that a hypothesis class $\MF$ of binary classifiers is learnable with approximate differential privacy (Definition~\ref{def:dp}) if and only if it is online learnable, i.e., has finite Littlestone dimension (Definition~\ref{def:ldim}).  Specifically, Alon et al.~\cite{alon_private_2019} showed that any \DP learning algorithm with at most constant error for a class of Littlestone dimension $d$ must use at least $\Omega(\log^\st d)$ samples.  Conversely, Bun et al.~\cite{bun_equivalence_2020} showed that if $\MF$ has Littlestone dimension $d$, then there is a \DP learning algorithm for $\MF$ with error $\alpha > 0$ using $2^{O(d)} / \alpha$ samples.\footnote{This bound ignores the dependence on the privacy parameters $\ep, \delta$. Moreover, it applies to the realizable setting; a slightly weaker bound was shown in~\cite{bun_equivalence_2020} for the agnostic setting.} 



\subsection{Results}
In this paper, we resolve two open questions posed by Bun et al.~\cite{bun_equivalence_2020} and Bousquet et al.~\cite{bousquet_passing_2019}: first, we introduce a new private learning algorithm with sample complexity \emph{polynomial} in the Littlestone dimension $d$ of the class $\MF$, thus improving exponentially on the bound $2^{O(d)}$ from~\cite{bun_equivalence_2020}. Answering a second question of~\cite{bun_equivalence_2020}, we show how to make our private learner {\it proper} (whereas the learner from~\cite{bun_equivalence_2020} was improper). Whether privately properly learning classes of finite Littlestone dimension is possible was also asked by Bousquet et al.~\cite[Question 1]{bousquet_passing_2019}. Theorem~\ref{thm:pap-pac-informal} states our main result:
\begin{theorem}[Private proper PAC learning; informal version of Theorem~\ref{thm:poly-pri-learn-proper}]
  \label{thm:pap-pac-informal}
Let $\MF$ be a class of hypotheses $f : \MX \ra \{-1,1\}$, of Littlestone dimension $d$. For any $\ep, \delta, \alpha \in (0,1)$, for some $n = \tilde O \left(\frac{d^6}{\ep \alpha^2} \right)$, there is an $(\ep, \delta)$-\DP algorithm which, given $n$ i.i.d.~samples from any realizable distribution $P$ on $\MX \times \{-1,1\}$, with high probability outputs a classifier $\hat f \in \MF$ with classification error over $P$ at most $\alpha$. 
\end{theorem}
The theorem statement above treats the case where the distribution $P$ over $\MX \times \{-1,1\}$ is {\it realizable}, namely that there exists some $f^\st \in \MF$ so that $P$ is supported on pairs $(x, f^\st(x))$. A generic reduction of~\cite{alon_closure_2020} allows us to show essentially the same sample complexity bound as in Theorem~\ref{thm:pap-pac-informal} for the non-realizable (i.e., {\it agnostic}) setting (see Corollary~\ref{cor:agnostic-proper}). We also remark that it is impossible to obtain a sample complexity bound better than $n = O(d)$ in the context of Theorem~\ref{thm:pap-pac-informal} if we insist that the bound depends on the class $\MF$ only through the Littlestone dimension $d$. This follows because for any $d \in \BN$, there are classes $\MF$ whose Littlestone and VC dimensions are both equal to $d$ (for instance, the class of all binary hypotheses on $d$ points), and it is well-known that the VC dimension characterizes the sample complexity of learning a class (in the absence of privacy) ~\cite{vapnik_statistical_1998}.

The question posed in~\cite{bousquet_passing_2019,bun_equivalence_2020} (and answered by Theorem~\ref{thm:pap-pac-informal}) of whether classes of finite Littlestone dimension have private proper learners is motivated by a connection between proper private learning and {\it private query release} established in~\cite{bousquet_passing_2019}. The problem of private query release, or {\it sanitization}~\cite{blum2008learning,beimel_private_2014}, for a class $\MF$ has an extensive history, described in Section~\ref{sec:related-work}. It is defined as follows: given $\alpha > 0$, a {\it sanitizer} with sample complexity $n \in \BN$ is given as input a dataset $S = \{(x_1, y_1), \ldots, (x_n, y_n) \} \in (\MX \times \{-1,1\})^n$. The sanitizer must output a function $\Est : \MF \ra [0,1]$, which is differentially private for the input $S$, so that with high probability, for each $f \in \MF$, $| \Est(f) - \err{S}{f}| \leq \alpha$, where $\err{S}{f} := \frac{1}{n} \cdot | \{ i \in [n] : f(x_i) \neq y_i\} |$. Bousquet et al.~\cite{bousquet_passing_2019} showed that the existence of a private proper learner for a class $\MF$ implies the existence of a sanitizer for $\MF$; as a corollary of their result and of Theorem~\ref{thm:pap-pac-informal} we therefore obtain the following:
\begin{corollary}[Private query release; informal version of Corollary~\ref{cor:sanitizing-formal}]
  \label{cor:sanitizing-informal}
  Let $\MF$ be a class of hypotheses $f : \MX \ra \{-1,1\}$ of Littlestone dimension $d$ and dual Littlestone dimension $d^\st$. For any $\ep, \delta, \alpha \in [0,1]$, there is an $(\ep, \delta)$-\DP algorithm that for some $n = \poly(d, d^\st, 1/\ep, 1/\alpha, \log 1/\delta)$, takes as input a dataset $S$ of size $n$ and outputs a function $\Est : \MF \ra [0,1]$ so that with high probability, for all $f \in \MF$, $|\Est(f) - \err{S}{f} | \leq \alpha$.
\end{corollary}
It is known that the dual Littlestone dimension $d^\st$ of a class $\MF$ is finite if and only if the Littlestone dimension $d$ is finite; in fact, we have $d^\st \leq 2^{2^{d + 2}} - 2$~\cite[Corollary 3.6]{bhaskar_thicket_2017}. Thus, Corollary~\ref{cor:sanitizing-informal} implies that a class $\MF$ is {\it sanitizable} (roughly, that it has a sanitizer with sample complexity $\poly(1/\alpha)$; see Definition~\ref{def:sanitizable} for a formal version) if it has finite Littlestone dimension. The converse, namely that any sanitizable class must have finite Littlestone dimension, follows as a consequence of a result of~\cite{bun_differentially_2015}, as discussed in Section~\ref{sec:sanitization}. Summarizing, we have the following:
\begin{corollary}
  \label{cor:sanitizing-2}
A hypothesis class $\MF$ is sanitizable if and only if it has finite Littlestone dimension.
\end{corollary}

\paragraph{Techniques: irreducibility}

The main technique that allows us to both improve the exponential bound $2^{O(d)}$ on the sample complexity from \cite{bun_equivalence_2020} to a polynomial dependence, and to make the learner proper in Theorem \ref{thm:pap-pac-informal} is a property of hypothesis classes we introduce, called {\it irreducibility} (Section \ref{sec:irreducibility}). Roughly speaking, a binary hypothesis class $\MG$ of Littlestone dimension $d$ on domain $\MX$ is irreducible if any binary tree of bounded depth labeled by elements of $\MX$ has a leaf such that the restriction of $\MG$ to that leaf still has Littlestone dimension $d$. 
The exponential sample complexity bound in \cite{bun_equivalence_2020} arises (in part) for the following reason: the main sub-procedure in their algorithm operates in a sequence of $d = \Ldim(\MF)$ steps, maintaining a class of candidate hypotheses; at the end of the $d$ steps, this class will have Littlestone dimension 0 (i.e., consists of a single hypothesis), and will be the hypothesis output by the sub-procedure.
Each of these $d$ steps decreases the Littlestone dimension of the class of candidate hypotheses by 1 and increases the number of samples needed by a constant factor, leading to $2^{O(d)}$ samples overall. The notion of irreducibility allows us to show that certain intermediate classes of candidate hypotheses can be ``sufficiently stable'' to allow us to output a hypothesis associated with the intermediate class in a private way. 
This allows us to avoid the exponential blowup in $d$ associated with decreasing the Littlestone dimension of the candidate hypotheses all the way to 0. 
We believe that the notion of irreducibility may be useful in other applications. We provide a more detailed overview of our proofs in Section \ref{sec:proof-overview}.

\subsection{Related work}
\label{sec:related-work}
\paragraph{Sample complexity of differentially private learning} The sample complexity of PAC learning with {\it pure} differential privacy (namely, $(\ep, 0)$-differential privacy) is well-understood. The seminal work of Kasiviswanathan et al.~\cite{kasiviswanathan2008what} showed that a finite class $\MF$ consisting of hypotheses $f : \MX \ra \{-1,1\}$ can be learned with pure differential privacy with sample complexity $O(\log |\MF|)$ (in this section we omit dependence on the privacy and accuracy parameters). By the Sauer--Shelah lemma, $\log(|\MF|) \leq O(\vc(\MF) \cdot \log(|\MX|))$; moreover, the multiplicative gap between $\vc(\MF)$, which characterizes the sample complexity of {\it non-private} learning, and $\log |\MF|$, can be as large as $\log |\MX|$. To obtain a more precise result, Beimel et al.~\cite{beimel_characterizing_2019} introduced a complexity measure for a class $\MF$ of binary hypotheses, known as the {\it probabilistic representation dimension} of $\MF$, which they showed to characterize the sample complexity of (improperly) learning $\MF$ with {pure differential privacy}  up to a constant factor (see also~\cite{beimel_bounds_2014}). Feldman and Xiao~\cite{feldman_sample_2014} showed that, in turn, the probabilistic representation dimension is characterized, up to a constant factor, by the one-way public coin communication complexity of an evaluation problem associated to $\MF$. As a corollary of this result, they established that the sample complexity of learning $\MF$ with pure privacy is always at least $\Omega(\Ldim(\MF))$, where $\Ldim(\MF)$ denotes the Littlestone dimension of $\MF$. 

The current understanding of the sample complexity of learning with {\it approximate} differential privacy (namely, $(\ep, \delta)$-differential privacy with $\delta$ negligible as a function of the number of users), which is our focus in this paper, is much less complete. The class of threshold functions on a domain of size $2^d$, which has Littlestone dimension $d$, is known to be learnable with approximate privacy with sample complexity $O((\log^\st d)^{1.5})$~\cite{kaplan_privately_2020}, showing that the sample complexity of learning a class $\MF$ with approximate privacy can be much less than its Littlestone dimension (see also~\cite{bun_differentially_2015,beimel_private_2014,bun_composable_2018}, which obtained weaker bounds). As mentioned previously, the best-known lower bound for the sample complexity of (improperly) privately learning a class of Littlestone dimension $d$ is $\Omega(\log^\st d)$~\cite{alon_private_2019}; our Theorem~\ref{thm:pap-pac-informal} gives the best known upper bound in terms of Littlestone dimension. In a different direction, some recent papers have investigated the sample complexity of privately learning halfspaces~\cite{kaplan_private_2020,kaplan_how_2020,beimel_private_2019}.

\paragraph{Differentially private query release}
The problem of private data release (also known as sanitization; see Section~\ref{sec:dp} for a formal definition) for a binary hypothesis class $\MF$ dates back to Blum et al.~\cite{blum2008learning}, who showed that the sample complexity of private sanitization is bounded above by $O(\vc(\MF) \cdot \log |\MX|) \leq O(\log |\MF| \cdot \log |\MX|)$. This bound was later improved to $\tilde O(\log |\MF| \cdot \sqrt{\log |\MX|})$ by Hardt and Rothblum~\cite{hardt2010multiplicative},\footnote{The $\tilde O$ hides factors logarithmic in $\log |\MF|$ and $\log |\MX|.$} which is known to be essentially the best possible dependence on $|\MF|, |\MX|$ attainable for a broad range of values of $|\MF|, |\MX|$~\cite{bun_fingerprinting_2014}. Many works have developed more fine-grained bounds on the sample complexity of sanitization in terms of geometrical properties of $\MF$~\cite{bhaskara_unconditional_2012,blasiok_towards_2019,edmonds_power_2020,hardt_geometry_2010,nikolov_improved_2015,nikolov_geometry_2012}, and several have additionally studied computational considerations for this problem~\cite{dwork_complexity_2009,dwork_boosting_2010,hardt_simple_2012,roth_interactive_2011}. However, the upper bounds on the sample complexity of sanitization obtained by all of these works scale at least polynomially with either $\log |\MX|$ or $\log |\MF|$; thus, they implicitly assume that $\MX$ or $\MF$ (or both) is finite. In addition to being of purely theoretical interest, establishing sample complexity bounds with no explicit dependence on $|\MX|, |\MF|$ (and thus which can apply when $|\MX|$ and $|\MF|$ are infinite) could lead to significant gains even in cases when they are finite since in many natural settings, $|\MX|, |\MF|$ are exponentially large in parameters such as dimensionality of the data. The question of removing the $\poly\log |\MF|$ factors in existing bounds has also been asked in~\cite{vadhan2017complexity}: Questions 5.24 and 5.25 in~\cite{vadhan2017complexity} ask for a characterization of the sample complexity of sanitization up to ``small'' approximation factors. In the proof of Corollary~\ref{cor:sanitizing-2} it is established that the sample complexity of sanitizing a class of Littlestone dimension $d$ is between $\Omega(\log^\st d)$ and $2^{O(2^{d})}$. This gap is definitely not ``small'' by any means, but for infinite $|\MF|, |\MX|$, it is the first finite approximation factor to the best of our knowledge. 


\paragraph{Online learning and Littlestone dimension} The Littlestone dimension of a hypothesis class $\MF$ is known to be equal to the optimal mistake bound in the realizable setting of online learning~\cite{littlestone_learning_1987,shalev-shwartz_online_2012}. Moreover, it characterizes the optimal regret of an online learning algorithm in the agnostic setting up to a logarithmic factor: the optimal regret $\Reg(T)$ for an online learning algorithm with respect to a class of Littlestone dimension $d$ satisfies $\Omega(\sqrt{dT}) \leq \Reg(T) \leq O(\sqrt{d T \log T})$~\cite{ben-david_agnostic_2009,shalev-shwartz_online_2012}. Therefore, Theorem~\ref{thm:pap-pac-informal} implies that the sample complexity of privately learning a binary hypothesis class $\MF$ is bounded above by a polynomial in the sample complexity of online learning of $\MF$ (in either the realizable or agnostic setting). 

Many prior works have investigated the connection between online and private learnability in slightly different settings from ours.  Inherent stability-type properties of private learning algorithms have been used to show that certain problems have online learning algorithms~\cite{gonen_private_2019,abernethy_online_2019,neel_heuristics_2019,alon_private_2019}.  Bun~\cite{bun_computational_2020} shows that such a reduction is not possible in a generic sense if it is required to be computationally efficient. In the opposite direction,~\cite{agarwal_price_2017,bun_equivalence_2020} develop differentially private algorithms to solve problems which are online learnable.

\subsection{Organization of the paper}  In Section~\ref{sec:prelim}, we review some preliminaries regarding private query release, private PAC learning, and online learning. In Section~\ref{sec:proof-overview}, we outline the proof of Theorem~\ref{thm:pap-pac-informal}. In Section~\ref{sec:irreducibility} we introduce a central notion used in our proof, namely that of {\it irreducibility}, and prove some basic properties of it. In Sections~\ref{sec:improper-learner} and~\ref{sec:proper-learning-finite} we prove Theorem~\ref{thm:pap-pac-informal} and its corollaries for private query release.  Concluding remarks are in Section~\ref{sec:conclusion}.

\section{Preliminaries}
\label{sec:prelim}
We will use the script notation (e.g., $\MF, \MX$) to denote sets (e.g., sets of data points or sets of binary hypotheses). For sets $\MS, \MT$, we write $\MS \subset \MT$ to mean that $\MS$ is a (not necessarily proper) subset of $\MT$.
\subsection{PAC learning}
\label{sec:pac}
We use standard notation and terminology regarding PAC learning (see, e.g.,~\cite{shalev-shwartz_understanding_2014}). Let $\MX$ be an arbitrary set and let $\{ -1,1\}$ be the label set. We suppose throughout the paper that $\MX \times \{-1,1\}$ is endowed with a $\sigma$-algebra $\Sigma$. For $x \in \MX, y \in \{-1,1\}$, let $\delta_{(x,y)}$ denote the point measure at $(x,y)$, i.e., for $A \in \Sigma$, $\delta_{(x,y)}(A)$ is defined to be 1 if $(x,y) \in A$, and 0 otherwise.

A {\it hypothesis} is a function $f : \MX \ra \{-1,1\}$. We write the set of all hypotheses on $\MX$ as $\{-1,1\}^\MX$. An {\it example} is a pair $(x,y) \in \MX \times \{-1,1\}$, and for $n \in \BN$, a {\it dataset} $S_n$ is a set of $n$ examples, $S_n := \{(x_1, y_1), \ldots, (x_n, y_n) \}$. Given such a dataset, define the {\it empirical measure} $\emp := \frac 1n \sum_{i=1}^n \delta_{(x_i, y_i)}$ on $\MX \times \{-1,1\}$. For a distribution $\di$ on $\MX \times \{-1,1\}$, let $\di^n$ be the distribution of $S_n \in (\MX \times \{-1,1\})^n$ consisting of $n$ i.i.d.~draws from $\di$.

\begin{defn}[Error of a hypothesis]
  Let $\di$ be a probability distribution on $\MX \times \{-1,1\}$. The {\it error} (or {\it loss}) of a hypothesis $f : \MX \ra \{-1,1\}$ is defined as
  $$
\err{\di}{f} := \p_{(x,y) \sim \di} \left[ f(x) \neq y \right].
  $$
\end{defn}
The {\it empirical error} of a hypothesis $f$ with respect to a dataset $S_n$ is defined to be $\err{\emp}{f}$. At times we will abbreviate $\err{\emp}{f}$ by writing $\err{S_n}{f}$ instead. In this paper we will consider hypothesis classes $\MF \subset \{-1,1\}^\MX$; to avoid having to make technical measurability assumptions on $\MF, \MX$, we will assume throughout that $\MF$ and $\MX$ are countable.  (We refer the reader to~\cite[Chapter 5]{dudley_uniform_1999} for a discussion of such assumptions in the case that countability does not hold.  We remark that it is necessary to make such measurability assumptions for standard arguments (e.g., regading uniform convergence) to hold even in the non-private case: without such assumptions, there are (uncountably) infinite classes of VC dimension 1, which empirical risk minimization fails to learn~\cite{ben-david_notes_2015}.)

For any $x \in \MX, b \in \{-1,1\}$, write $\MF|_{(x,b)} := \{ f \in \MF : f(x) = b\}$.

\subsection{Differential privacy and sanitization}
\label{sec:dp}
While our main focus in this paper is on PAC learning, we will additionally discuss implications of our results to differentially private data release. Therefore, in the below definition of differential privacy, we allow each user's example to belong to an arbitrary set $\MZ$ (in PAC learning we have $\MZ = \MX \times \{-1,1\}$). 
\begin{defn}[Differential privacy,~\cite{Dwork2006DP}]
  \label{def:dp}
Fix sets $\MZ, \MW$ and $n \in \BN$, and suppose $\MW$ is countable.\footnote{The restriction of countability may be readily removed by fixing a $\sigma$-algebra $\Sigma$ on $\MW$ and letting $A$ be a mapping from $\MZ^n$ to the space $\Delta(\MW)$ of probability measures on the measure space $(\MW, \Sigma)$.} A randomized algorithm $A : \MZ^n \ra \MW$ is {\it $(\ep, \delta)$-\DP} if the following holds: for any datasets $S, S' \in \MZ^n$ differing in a single example and for all subsets $\MT \subset \MW$, 
$$
\p[A(S) \in \MT] \leq e^{\ep} \cdot \p[A(S') \in \MT] + \delta.
$$
\end{defn}
The {\it sanitization} (or {\it private query release}) problem was introduced in~\cite{blum2008learning} and has been central in many works in differential privacy:
\begin{defn}[Sanitization,~\cite{blum2008learning,beimel_private_2014}]
\label{def:sanitizable}
Fix $n \in \BN$ and $\alpha, \beta, \ep, \delta \in (0,1)$, and suppose $\MF \subset \{-1,1\}^\MX$ is a binary hypothesis class. A randomized algorithm $A : (\MX \times \{-1,1\})^n \ra [0,1]^{\MF}$ is an {\it $(n, \alpha, \beta, \ep, \delta)$-sanitizer} if $A$ is $(\ep,\delta)$-\DP and for all datasets $S = ((x_1, y_1), \ldots, (x_n,y_n)) \in (\MX \times \{-1,1\})^n$, $A(S)$ outputs a function $\Est : \MF \ra [0,1]$ so that with probability at least $1-\beta$, for all $f \in \MF$,
$$
\left| \Est(f) - \frac{| \{ i \in [n] : f(x_i) = y_i \}|}{n} \right| \leq \alpha.
$$

Following~\cite{bousquet_passing_2019}, we say that a class $\MF$ is {\it sanitizable} if there exists a bound $n_\MF(\alpha, \beta) = \poly(1/\alpha, 1/\beta)$ 
so that for every $\alpha,\beta > 0$, there exists an algorithm $A$ on datasets of size $n = n_\MF(\alpha,\beta)$ which is an $(n, \alpha, \beta, \ep, \delta)$-sanitizer for some $\ep = O(1)$ and $\delta$ negligible as a function of $n$. 
\end{defn}

\subsection{VC dimension and uniform convergence}
We will denote {\it hypothesis classes}, namely subsets of $\{-1,1\}^\MX$, with the letters $\MF, \MG, \MH$. A class $\MF \subset \{-1,1\}^\MX$ is said to {\it shatter} a set $\{ x_1, \ldots, x_n\} \subset \MX$ if for each choice $(\ep_1, \ldots, \ep_n) \in \{-1,1\}^n$, there is some $f \in \MF$ so that for all $i \in [n]$, $f(x_i) = \ep_i$.
\begin{defn}[VC dimension]
  \label{def:vc-dim}
The {\it VC dimension} of the class $\MF$, denoted $\vc(\MF)$, is the largest positive integer $n$ so that $\MF$ shatters a set of size $n$.
\end{defn}

We need the following standard fact that finite VC dimension is a sufficient condition for uniform convergence with respect to arbitrary distributions:
\begin{theorem}[e.g.,~\cite{bartlett_rademacher_2003}, Theorems 5 \& 6]
  \label{thm:unif-conv}
  Suppose that $\MF$ is countable and $\vc(\MF) =d_{\VV} \geq 1$. Then there is a constant $C_0 \geq 1$ such that for any distribution $\di$ on $\MX \times \{-1,1\}$ and any $\gamma \in (0,1)$, it holds that
  $$
\p_{S_n \sim \di^n} \left[ \sup_{f \in \MF} \left| \err{\di}{f} - \err{\emp}{f} \right| > C_0 \sqrt{\frac{ d_{\VV}  + \log 1/\gamma}{n}} \right] \leq \gamma.
  $$
\end{theorem}

For a class $\MF \subset \{-1,1\}^\MX$ and a distribution $\di$ on $\MX \times \{-1,1\}$, define 
$$
\MF_{\di,\alpha} := \{ f \in \MF : \err{\di}{f} \leq \alpha\}.
$$
Note that for any $0 \leq \alpha \leq \beta \leq 1$, we have $\MF_{\di,\alpha} \subset \MF_{\di,\beta}$.

For any $\gamma > 0$ and $n \in \BN$, write $\alpha(n,\gamma) := C_0 \sqrt{\frac{d_{\VV} + \log 1/\gamma}{n}}$, so that by Theorem~\ref{thm:unif-conv} we have that $\p_{S_n} \left[ \sup_{f \in \MF} \left| \err{\di}{f} - \err{\emp}{f} \right| > \alpha(n,\gamma) \right] \leq \gamma$.

Note that, under the event $\sup_{f \in \MF} \left| \err{\di}{f} - \err{\emp}{f} \right| \leq \alpha_0$, we have that, for each $\alpha \in [0,1]$,
\begin{equation}
  \label{eq:emp-exp-compare}
\MF_{\emp,\alpha -2 \alpha_0} \subset \MF_{\di, \alpha- \alpha_0} \subset \MF_{\emp,\alpha}.
\end{equation}

Given a class $\MF \subset \{-1,1\}^\MX$, its {\it dual class}, denoted by $\MF^\st$, is defined as follows: $\MF^\st \subset \{-1,1\}^\MF$ and is indexed by $\MX$. For each $x \in \MX$, the corresponding function in $\MF^\st$ is the function $x : \MF \ra \{-1,1\}$, defined by $x(f) := f(x)$. The {\it dual VC dimension} of $\MF$, denoted by $\vc^\st(\MF)$, is the VC dimension of $\MF^\st$: i.e., $\vc^\st(\MF) := \vc(\MF^\st)$.

\subsection{Littlestone dimension}
\label{sec:littlestone-prelim}
To introduce the Littlestone dimension, we need some notation regarding binary trees. For a positive integer $t$ and a sequence $\bt_1, \bt_2, \ldots, \bt_t, \ldots \in \{-1,1\}$, write $\bt_{1:t} := (\bt_1, \ldots, \bt_t)$. As a convention, let $\bt_{1:0}$ denote the empty sequence. For $n \in \BN$, an {\it $\MX$-valued binary tree} $\bx$ of {\it depth} $n$ is a collection of partial functions $\bx_t : \{-1,1\}^{t-1} \ra \MX$ for $1 \leq t \leq {n}$, each with nonempty domain, so that for all $\bt_{1:t}$ in the domain of $\bx_{t+1}$, $\bt_{1:t-1}$ is in the domain of $\bx_t$ and $(\bt_1, \ldots, \bt_{t-1}, -\bt_t)$ is in the domain of $\bx_{t+1}$. If $\bx_t$ is a total function for all $t$, then we say that $\bx$ is a {\it complete} tree; otherwise, we say that $\bx$ is {\it incomplete}. By default we will use the term ``tree'' to refer to complete binary trees; when we wish to refer to incomplete trees (or the notion of {\it generalized trees} in Definition~\ref{def:generalized-tree}), we will use the appropriate adjective.

Associated with each sequence $\bt_{1:t} \in \{-1,1\}^t$ so that either $t = 0$ or $b_{1:t-1}$ is in the domain of $\bx_t$, for some $1 \leq t \leq n$, is a {\it node} of the (possibly incomplete) tree. We say that this node is a {\it leaf} if $b_{1:t}$ is not in the domain of $\bx_{t+1}$; in particular, for complete trees, the nodes associated to each $\bt_{1:n} \in \{-1,1\}^{n}$ are the leaves. Suppose $b_{1:t-1}$ is in the domain of $\bx_t$, for some $t$; then the node associated with $b_{1:t-1}$ is not a leaf, and we say that this node is {\it labeled} by $\bx_t(\bt_{1:t-1})$. For any such non-leaf node $v$, the two nodes associated with $(b_1, \ldots, b_{t-1}, -1)$ and $(b_1, \ldots, b_{t-1}, 1)$ are the {\it children} of $v$ {\it corresponding to the bits $-1$ and $1$}, respectively. Note that a node is a leaf if and only if it has no children. Note also that any non-leaf node has exactly 2 children.

A class $\MF \subset \{-1,1\}^\MX$ is said to {\it shatter} a (complete) tree $\bx$ of depth $n$ if for all sequences $(\bt_1, \ldots, \bt_n) \in \{-1,1\}^n$, there is some $f \in \MF$ so that for each $t \in [n]$, $f(\bx_t(\bt_{1:t-1})) = \bt_t$.
\begin{defn}[Littlestone dimension]
  \label{def:ldim}
The {\it Littlestone dimension} of a class $\MF \subset \{-1,1\}^\MX$ is the largest positive integer $n$ so that there exists a tree $\bx$ of depth $n$ that is shattered by $\MF$.
\end{defn}
The Littlestone dimension is known to exactly characterize the optimal mistake bound for online learnability of the class $\MF$ in the realizable setting~\cite{littlestone_learning_1987}, as well as to characterize the optimal regret bound for online learnability of $\MF$ in the agnostic setting up to a logarithmic factor~\cite{ben-david_agnostic_2009}.

Similar to the case for VC dimension, the dual Littlestone dimension of a class $\MF$, denoted by $\Ldim^\st(\MF)$, is the Littlestone dimension of $\MF^\st$: i.e., $\Ldim^\st(\MF) := \Ldim(\MF^\st)$.

\section{Proof overview}
\label{sec:proof-overview}
In this section we overview the proof of Theorem~\ref{thm:pap-pac-informal}. The proof is in two parts:
\begin{enumerate}
\item The first part is a private {\it improper} learner, \PolyPriLearn (Algorithm~\ref{alg:poly-pri-learn}), with sample complexity $\tilde O \left( \frac{d^6}{\ep \alpha^2} \right)$.  The hypothesis $\hat f \in \{-1,1\}^\MX$ output by \PolyPriLearn also satisfies an additional property, namely, it is associated with an {\it irreducible} subclass of $\MF$ (a notion that we introduce and explain below), with high probability.
\item The second part is a technique, \PPPLearn (Algorithm~\ref{alg:poly-pri-prop-learn}), to convert the improper learner from the first part to a proper learner using the irreducibility property of the hypothesis $\hat f$. 
\end{enumerate}
We now elaborate further on the two parts of the proof.

\paragraph{Part 1: Improper learner and irreducibility}
Besides allowing us to convert an improper learner to a proper one, the notion of {\it irreducibility} is central in allowing us to find a private improper learner for $\MF$ with sample complexity polynomial in $\Ldim(\MF)$. 
Before defining irreducibility and explaining how it is useful, we first outline the overall approach. Given a dataset $S_n = \{(x_1, y_1), \ldots, (x_n, y_n)\}$ drawn i.i.d.~from some distribution $P$ over $\MX \times \{-1,1\}$, we will find several subclasses $\hat \MG_1, \ldots, \hat \MG_J \subset \MF$,\footnote{As a general convention we use a hat for quantities that depend on the dataset.} for some $J \in \BN$ so that for each $1 \leq j \leq J$, $\hat \MG_j$ consists entirely of functions with low empirical error on the dataset $S_n$ (this task is performed by the sub-routine \ReduceTree (Algorithm~\ref{alg:reduce-tree}) of \PolyPriLearn). We will then consider the {\it SOA classifier}\footnote{
  As an aside, the SOA classifier achieves the optimal mistake bound in the realizable setting of online learning \cite{littlestone_learning_1987,shalev-shwartz_online_2012}. 
}
for each subclass $\hat\MG_j$; the SOA classifier for a class $\MG$, denoted by $\soaf{\MG} \in \{-1,1\}^\MX$, is defined as follows: for $x \in \MX$, $\soa{\MG}{x} = 1$ if $\Ldim(\MG|_{(x,1)}) \geq \Ldim(\MG|_{(x,-1)})$, and $\soa{\MG}{x} = -1$ otherwise. The crux of the proof rests on two facts: 
\begin{enumerate}
\customitem{(a)} There are $d+1$ ``special'' classifiers $\sigma_1^\st, \ldots, \sigma_{d+1}^\st \in \{-1,1\}^\MX$ (which depend on $P$ but not any particular dataset) so that with high probability, at least one of $\soaf{\hat \MG_1}, \ldots, \soaf{\hat \MG_J}$ is equal to one of $\sigma_1^\st, \ldots, \sigma_{d+1}^\st$. \label{it:special-classes}
\customitem{(b)} For each class $\hat \MG_j$ that is found in the sub-routine \ReduceTree, with high probability it holds that $\soaf{\hat \MG_j}$ has low population error (i.e., $\err{\di}{\soaf{\hat \MG_j}}$ is small). \label{it:soaf-low-error} 
\end{enumerate}
If properties~\ref{it:special-classes} and~\ref{it:soaf-low-error} are given, then the construction of a private learner is fairly straightforward: if $J$ were a constant, then we could draw $m = \tilde O(d)$ independent datasets $S\^1_n, \ldots, S\^m_n$ and use the private stable histogram of~\cite[Proposition 2.20]{bun_simultaneous_2016} together with property~\ref{it:special-classes} to privately output some $\soaf{\MG} \in \{-1,1\}^\MX$ that belongs to $\{\soaf{\hat \MG_1\^i}, \ldots, \soaf{\hat \MG_J\^i}\}$ for many of the independent datasets $S\^i_n$ (we denote the subclasses corresponding to the $i$th dataset, $i \in [m]$, by $\hat \MG_j\^i$, for $j \in [J]$). By property~\ref{it:soaf-low-error}, such $\soaf{\MG}$ would then have low population error. As it turns out, we will only be able to guarantee that $J = 2^{\tilde O(d^2)}$; we can still guarantee sample complexity polynomial in $d$, though, by using a variant of the stable histogram based on the exponential mechanism~\cite{ghazi_differentially_2020} in Algorithm~\ref{alg:poly-pri-learn}. This will necessitate an increase in $m$ by a factor of $\log(2^{\tilde O(d^2)})$, so that we draw a total of $m = \tilde O(d^3)$ independent datasets; each will be of size $\tilde O(d^3)$, leading to the overall sample complexity bound of $\tilde O(d^6)$. 

Next we discuss the proofs of properties~\ref{it:special-classes} and~\ref{it:soaf-low-error} of the subclasses $\hat \MG_1, \ldots, \hat \MG_J$ that the sub-routine \ReduceTree outputs. The proofs of both of these properties depend on irreducibility, which we now define. 
We say that a hypothesis class $\MG \subset \{-1,1\}^\MX$ is {\it irreducible} if for any $x \in \MX$, it holds that for some $b \in \{-1,1\}$, we have $\Ldim(\MF|_{(x,b)}) = \Ldim(\MF)$. Definition~\ref{def:higher-irred} introduces the generalization of {\it $k$-irreducibility} for all $k \in \BN$ (irreducibility corresponds to 1-irreducibility), but in this section we exhibit the main ideas behind the proof using $k = 1$. To explain how we obtain property~\ref{it:special-classes}, first suppose that the following holds, for some fixed $\alpha_\Delta < \alpha_0$:
\begin{equation}
\tag{A}\label{eq:irred-assumption}
\parbox{15cm}{\centering\text{\sl With high probability over the sample $S_n$, it holds that} \\ 
\text{\sl $\Ldim(\MF_{\hat P_{S_n}, \alpha}) = \Ldim(\MF_{\hat P_{S_n}, \alpha - \alpha_\Delta})$ and $\MF_{\hat P_{S_n}, \alpha-\alpha_\Delta}$ is irreducible. }}
\end{equation}
By Theorem~\ref{thm:unif-conv} and (\ref{eq:emp-exp-compare}) with $\alpha_0 = \alpha_\Delta/2$, as long as $n \geq \tilde \Omega \left( \frac{d}{\alpha_\Delta^2} \right)$, then with high probability we have $\MF_{\emp, \alpha-\alpha_\Delta} \subset \MF_{\di, \alpha-\alpha_\Delta/2} \subset \MF_{\emp, \alpha}$, and so 
\begin{equation}
\label{eq:diff-ldim-equal}
\Ldim(\MF_{\di, \alpha-\alpha_\Delta/2}) = \Ldim(\MF_{\emp, \alpha - \alpha_\Delta})
\end{equation}
by (\ref{eq:irred-assumption}). Using irreducibility of $\MF_{\di, \alpha-\alpha_\Delta}$ and (\ref{eq:diff-ldim-equal}), it is straightforward to show (Lemma~\ref{lem:red-hierarchy}) that 
\begin{equation}
\label{eq:soas-stability}
\soaf{\MF_{\di, \alpha-\alpha_\Delta/2}} = \soaf{\MF_{\emp, \alpha-\alpha_\Delta}}.
\end{equation}
Thus, we have shown, assuming (\ref{eq:irred-assumption}), that a quantity that can be computed from the empirical data, namely $\soaf{\MF_{\emp, \alpha-\alpha_\Delta}}$, is equal with high probability to a fixed quantity, namely $\soaf{\MF_{\di,\alpha-\alpha_\Delta/2}}$, which we may take to be, say $\sigma_1^\st$, in property~\ref{it:special-classes}. 

Of course, we must also deal with the case where (\ref{eq:irred-assumption}) does not hold. There are two possible reasons for this: the first is that $\Ldim(\MF_{\emp, \alpha-\alpha_\Delta}) < \Ldim(\MF_{\emp,\alpha})$. In this case, as long as $\alpha_\Delta$ is sufficiently small, we may replace $\alpha$ with $\alpha - \alpha_\Delta$ and recurse (i.e., check if (\ref{eq:irred-assumption}) holds with the new value of $\alpha$, and act accordingly). Since $\Ldim(\MF_{\emp, \alpha}) \leq \Ldim(\MF) \leq d$, the Littlestone dimension can decrease at most $d$ times and therefore it is sufficient to choose $\alpha_\Delta \approx \alpha / d$ (and so we may take $n = \tilde O(d^3)$).

The other reason that (\ref{eq:irred-assumption}) may fail to hold is that $\MF_{\emp, \alpha - \alpha_\Delta}$ is not irreducible. In such a case, by definition of irreducibility, there exists some $x \in \MX$ so that $$\max \{ \Ldim(\MF_{\emp, \alpha - \alpha_\Delta}|_{(x,1)}), \Ldim(\MF_{\emp, \alpha-\alpha_\Delta}|_{(x,-1)}) \} < \Ldim(\MF).$$ The idea is to now make {\it two} recursive calls, replacing $\alpha$ with $\alpha - \alpha_\Delta$ (as before) and using each of the classes $\MF|_{(x,1)}$ and $\MF|_{(x,-1)}$ in place of $\MF$. A clear issue with this approach is that $x$ may depend on the dataset $S_n$, and so the crucial ``stability'' property of (\ref{eq:soas-stability}) may fail to hold in the recursive call, even if (\ref{eq:irred-assumption}) holds with the new $\alpha$ and for the class $\MF|_{(x, \pm1)}$. It turns out that we can amend this issue by replacing irreducibility in (\ref{eq:irred-assumption}) with the stronger property of $k$-irreducibility for $k > 1$; the details can be found in Sections~\ref{sec:irreducibility} and~\ref{sec:reducetree}. 

This process of decreasing the Littlestone dimension by at least 1 and then making some number of ``recursive'' calls results in a tree with at most ${2^{\tilde{O}(d^2)}}$ leaves (Definition~\ref{def:generalized-tree} describes the specific tree structure). Each of these leaves determines a class $\hat \MG_j$, and using a generalization of (\ref{eq:soas-stability}), we can ensure that the classes $\hat \MG_j$ satisfy property~\ref{it:special-classes}. Moreover, we will be able to ensure that for a sufficiently large integer $k$, each $\hat \MG_j$ is $k$-irreducible; this will be enough to show that property~\ref{it:soaf-low-error} above holds via a fairly straightforward argument (carried out in Lemma~\ref{lem:gen-soa} and Claim~\ref{cl:empirical-error}).

\paragraph{Part 2: Making the improper learner proper}
Let $\soaf{\hat \MG} \in \{-1,1\}^\MX$ be the classifier output by the private improper learner \PolyPriLearn described above. The idea to make this learner proper is to find a small set $\hat\MH \subset \MF$ (in particular, of size bounded by $O(\vc^\st(\MF)/\alpha^2)$), such that for any distribution $Q$ over $\MX$, there is some $\hat h \in \hat\MH$ such that $\p_{x \sim Q}[\soa{\hat \MG}{x} \neq \hat h(x)] \leq \alpha$. In particular, this holds for $Q = P$, the true population distribution. Thus, since the improper learner from above guarantees that $\soaf{\hat \MG}$ has low population error under $P$ with high probability, we can choose some $\hat h \in \hat\MH$ with not much higher error using the exponential mechanism on a fresh set of samples of size roughly $\log |\hat\MH| \leq \tilde O(\log \vc^\st(\MF)) \leq \tilde O(\vc(\MF))$ (this is explained in detail in \PPPLearn, Algorithm~\ref{alg:poly-pri-prop-learn}).

It remains to show the existence of a small $\hat \MH \subset \MF$. To do so, we consider the zero-sum game with action spaces $\MF$ and $\MX$, where the row player chooses $h \in \MF$, the column player chooses $x \in \MX$, and the value of the game is $\One[h(x) \neq \soa{\hat\MG}{x}]$. By von Neumann's minimax theorem\footnote{The application of von Neumann's minimax theorem assumes that $\MF, \MX$ are finite; the infinite (countable) case is handled in Appendix~\ref{sec:topology} using basic ideas from topology.}, we have
\begin{equation}
\label{eq:0sum-game}
\inf_{D \in \Delta(\MF)} \sup_{P \in \Delta(\MX)} \E_{x \sim P, h \sim D}[\One[\soa{\hat\MG}{x} \neq h(x)]] = \sup_{P \in \Delta(\MX)} \inf_{D \in \Delta(\MF)} \E_{x \sim P, h \sim D} [\One[\soa{\hat\MG}{x} \neq h(x)]].
\end{equation}
Using the fact that the class $\hat \MG$ corresponding to the classifier $\soaf{\hat\MG}$ output by \PolyPriLearn is $k$-irreducible for a sufficiently large integer $k$, we show in Lemma~\ref{lem:game-value} that the right-hand side of (\ref{eq:0sum-game}) is bounded above by the desired accuracy $\alpha$. Thus the same holds for the left-hand side of (\ref{eq:0sum-game}). Now take a distribution $\hat D \in \Delta(\MF)$ attaining the infimum on the left-hand side of (\ref{eq:0sum-game}); using a uniform convergence argument applied to the {\it dual class} of $\MF$ (Lemma~\ref{lem:like-soa}), we may choose a multiset $\hat \MH \subset \MF$ of size $O(\vc^\st(\MF)/\alpha^2)$ so that the uniform distribution over $\hat \MH$ comes close to the infimum on the left-hand side of (\ref{eq:0sum-game}). Such an $\hat \MH$ satisfies the property we desired.

\section{Irreducibility}
\label{sec:irreducibility}
In this section we make a definition which is central to our algorithm and its analysis, namely that of {\it irreducibility} of a hypothesis class. We then prove some basic properties of irreducible classes. 

Fix some set $\MX$ and a space $\MG$ of hypotheses on $\MX$. 
For any $x \in \MX, b \in \{-1,1\}$, set
$$
\MG|_{(x,b)} := \{ f \in \MG : f(x) = b\}.
$$
For a set $S = \{ (x_1, b_1), \ldots, (x_n, b_n)\}$, similarly set
$$
\MG|_{S} := \bigcap_{i \in [n]} \MG|_{(x_i,b_i)} = \{ f \in \MG : f(x_i) = b_i \ \forall i \in [n]\}.
$$
For $S$ as above, we will at times abuse notation slightly and write $\MG|_S = \MG|_{(x_1, b_1), \ldots, (x_n, b_n)}$.
\begin{defn}[Irreducibility]
  \label{def:higher-irred}
  A class $\MG$ is defined to be {\it $k$-irreducible} if for any depth-$k$ tree $\bx$, there is some choice of bits $b_1, \ldots, b_k \in \{-1,1\}$ such that
  $$
\Ldim(\MG|_{(\bx_1,b_1), (\bx_2(b_1), b_2)\ldots, (\bx_k(b_{1:k-1}), b_k)}) = \Ldim(\MG).
$$
We say that the class $\MG$ is {\it irreducible} if it is $1$-irreducible.
\end{defn}
Note that $k$-irreducibility implies $k'$-irreducibility for $k' < k$. The following lemma shows that the choice of bits $b_1, \ldots, b_k$ in Definition~\ref{def:higher-irred} is unique: 
\begin{lemma}
\label{lem:unique-leaf}
Suppose $\MG$ is $k$-irreducible. Then for any depth-$k$ (possibly incomplete) tree $\bx$, there is a {\it unique} $t \in [k]$ and leaf associated to some $(b_1, \ldots, b_t) \in \{-1,1\}^t$ so that \\
$\Ldim(\MG|_{(\bx_1,b_1), (\bx_2(b_1), b_2)\ldots, (\bx_t(b_{1:t-1}), b_t)}) = \Ldim(\MG)$.
\end{lemma}
\begin{proof}
Since $\MG$ is $k$-irreducible, there is some $t \in [k]$ and leaf associated to some $b_{1:t} \in \{-1,1\}^t$ so that $\Ldim(\MG|_{(\bx_1,b_1), (\bx_2(b_1), b_2)\ldots, (\bx_t(b_{1:t-1}), b_t)}) = \Ldim(\MG)$. Suppose there were some other pair $(t', b'_{1:t'})$ so that $\Ldim(\MG|_{(\bx_1,b_1'), (\bx_2(b_1'), b_2')\ldots, (\bx_t(b'_{1:t'-1}), b'_{t'})}) = \Ldim(\MG)$. Then $\Ldim(\MG|_{(\bx_1, b_1)}) = \Ldim(\MG) = \Ldim(\MG|_{(\bx_1, b_1')})$, and thus $b_1 = b_1'$. We proceed by induction: for $1 \leq s \leq \min\{t,t'\}$, if $b_1 = b_1', \ldots, b_s = b_s'$, then 
\begin{align*}
\Ldim(\MG|_{(\bx_1, b_1), \ldots, (\bx_{s+1}(b_{1:s}), b_{s+1})})
& = \Ldim(\MG|_{(\bx_1, b_1), \ldots, (\bx_s(b_{1:s-1}), b_s)}) \\
& = \Ldim(\MG|_{(\bx_1, b_1), \ldots, (\bx_{s}(b_{1:s-1}), b_s), (\bx_{s+1}(b_{1:s}), b_{s+1}')}),
\end{align*}
and hence $b_{s+1} = b_{s+1}'$. Since $b_{1:t}$ and $b'_{1:t'}$ both are associated to leaves of the tree $\bx$, we must have $t = t'$ and $b_{1:t} = b'_{1:t'}$. 
\end{proof}

The following lemma shows that $k$-irreducibility satisfies a sort of ``monotonicity'' property among classes of the same Littlestone dimension.
\begin{lemma}
  \label{lem:red-order}
Suppose $\MH \subset \MG$, and $\Ldim(\MH) = \Ldim(\MG)$. If $\MH$ is $k$-irreducible, then so is $\MG$. 
\end{lemma}
\begin{proof}
  If $\MH$ is irreducible, then for any depth-$k$ $\MX$-valued tree $\bx$, we have that for some $b_1, \ldots, b_k \in \{-1,1\}$,
  $$
\Ldim(\MG|_{(\bx_1,b_1), \ldots, (\bx_k(b_{1:k-1}), b_k)}) \geq  \Ldim(\MH|_{(\bx_1,b_1), \ldots, (\bx_k(b_{1:k-1}),b_k)}) = \Ldim(\MH) = \Ldim(\MG),
$$
where the first inequality follows since any $f \in \MH$ with $f(\bx_i(b_{1:i-1})) = b_i$ for $1 \leq i \leq k$ is also in $\MG$. But since $\MG|_{(\bx_1, b_1), \ldots, (\bx_k(b_{1:k-1}),b_k)} \subset \MG$, we have $\Ldim(\MG|_{(\bx_1, b_1), \ldots, (\bx_k(b_{1:k-1}),b_k)}) \leq \Ldim(\MG)$, and so equality holds above, i.e., $\Ldim(\MG|_{(\bx_1, b_1), \ldots, (\bx_k(b_{1:k-1}),b_k)}) = \Ldim(\MG)$. 
\end{proof}

We next define the {\it SOA classifier} associated with a function class $\MG$; the choice of name is due to its similarity to the classifiers used in the standard optimal algorithm (SOA) in online learning~\cite{littlestone_learning_1987,ben-david_agnostic_2009}. 
\begin{defn}[SOA classifier]
\label{def:soa}
  For a class $\MG$, define the function $\soaf{\MG} : \MX \ra \{-1,1\}$ as follows:
  $$
  \soa{\MG}{x} := \begin{cases}
    1 \quad: \Ldim(\MG|_{(x,1)}) \geq \Ldim(\MG|_{(x,-1)}) \\
    -1 \quad : \Ldim(\MG|_{(x,1)}) < \Ldim(\MG|_{(x,-1)}).
  \end{cases}
  $$
\end{defn}

Lemma~\ref{lem:red-hierarchy} establishes an important ``stability-type'' property satisfied by SOA classifiers of irreducible classes.
\begin{lemma}
  \label{lem:red-hierarchy}
  Suppose $\MH \subset \MG, \Ldim(\MH) = \Ldim(\MG)$, and that $\MH$ is irreducible. Then for all $x \in \MX$,
  $
\soa{\MH}{x} = \soa{\MG}{x}.
  $
\end{lemma}
\begin{proof}
  Fix any $x \in \MX$. First suppose that $\soa{\MH}{x} = 1$, i.e., $\Ldim(\MH|_{(x,1)}) \geq \Ldim(\MH|_{(x,-1)})$. Then since $\MH$ is irreducible and $\MH \subset \MG$,
  $$
\Ldim(\MG|_{(x,1)})  \geq \Ldim(\MH|_{(x,1)}) = \Ldim(\MH) = \Ldim(\MG),
$$
which means that $\Ldim(\MG|_{(x,1)}) = \Ldim(\MG)$, and thus $\soa{\MG}{x} = 1$.

Next suppose that $\soa{\MH}{x} = -1$, i.e., $\Ldim(\MH|_{(x,1)}) < \Ldim(\MH|_{(x,-1)})$. Again using irreducibility of $\MG$, we see that
$$
\Ldim(\MG|_{(x,-1)}) \geq \Ldim(\MH|_{(x,-1)}) = \Ldim(\MH) = \Ldim(\MG),
$$
which means that $\Ldim(\MG|_{(x,-1)}) = \Ldim(\MG)$. We must have $\Ldim(\MG|_{(x,1)}) \leq \Ldim(\MG) - 1$, else it would be the case that $\Ldim(\MG) \geq 1 + \Ldim(\MG|_{(x,-1)})$. Hence $\soa{\MG}{x} = -1$. 
\end{proof}
The below lemma implies generalization bounds for the family of hypotheses $\soaf{\MG}$, for $\MG \subset \MF$ that are irreducible of sufficiently high order.
\begin{lemma}
  \label{lem:gen-soa}
  For a class $\MF$ with $\Ldim(\MF) = d$, set
  \begin{equation}
  \label{eq:tildef-d1}
\tilde \MF_{d+1} := \{ \soaf{\MG} : \MG \subset \MF, \MG \text{ is nonempty and $(d+1)$-irreducible}.\}
\end{equation}
Then $\Ldim(\tilde \MF_{d+1}) = d$ as well.
\end{lemma}
Note that $\MF \subset \tilde \MF_{d+1}$, since for any $f \in \MF$, $\{ f \}$ is $k$-irreducible for all $k \in \BN$, and $\soaf{\{ f \}} = f$. It is natural to wonder whether one can upper-bound $\Ldim(\tilde \MF_{d+1})$ if one drops the requirement that $\MG$ is $(d+1)$-irreducible in (\ref{eq:tildef-d1}); in Appendix~\ref{sec:ldim-soa}, we show that this is not possible.
\begin{proof}[Proof of Lemma~\ref{lem:gen-soa}]
  That $\Ldim(\tilde \MF_{d+1}) \geq d$ follows from $\MF \subset \tilde \MF_{d+1}$ 
  To see the upper bound on $\Ldim(\tilde \MF_{d+1})$, suppose for the purpose of contradiction that $\tilde \MF_{d+1}$ shatters an $\MX$-valued tree $\bx$ of depth $d+1$. We will show that $\MF$ also shatters $\bx$, which leads to the desired contradiction.

  Fix any sequence $\bt = (\bt_1, \ldots, \bt_{d+1}) \in \{-1,1\}^{d+1}$. There must be some $\MG \subset \MF$ that is $(d+1)$-irreducible so that for $1 \leq t \leq d+1$, $\soa{\MG}{\bx_t(\bt_{1:t-1})} = \bt_t$, which, by irreducibility of $\MG$, implies that $\Ldim(\MG|_{(\bx_t(\bt_{1:t-1}), \bt_t)})  = \Ldim(\MG)$. Since $\MG$ is $(d+1)$-irreducible, it follows that
  $$
\Ldim(\MG|_{(\bx_1, \bt_1), (\bx_2(\bt_1), \bt_2), \ldots, (\bx_{d+1}(\bt_{1:d}), \bt_{d+1})}) = \Ldim(\MG).
$$
(Indeed, by $(d+1)$-irreducibility of $\MG$, there must be {\it some} sequence $(\bt_1', \ldots, \bt_{d+1}') \in \{-1,1\}^{d+1}$ for which $\Ldim(\MG|_{(\bx_1, \bt_1'), (\bx_2(\bt_1'), \bt_2'), \ldots, (\bx_{d+1}(\bt_{1:d}'), \bt_{d+1}')}) = \Ldim(\MG)$; the smallest $t$ so that $b_t \neq b_t'$ satisfies $\Ldim(\MG|_{(\bx_t(\bt_{1:t-1}), \bt_t)}) = \Ldim(\MG|_{(\bx_t(\bt_{1:t-1}), -\bt_t)}) = \Ldim(\MG)$, which is impossible. Thus $b_t = b_t'$ for all $t$.)
Since $\MG$ is nonempty, there must be some $f \in \MG \subset \MF$ such that for $1 \leq t \leq d+1$, $f(\bx_t(\bt_{1:t-1})) = \bt_t$. It follows that $\MF$ shatters the tree $\bx$, as desired. 
\end{proof}

In Definitions~\ref{def:reducing-array} and~\ref{def:generalized-tree} below, we generalize the notion of tree to include those in which each node may have more than 2 children. The scheme by which we label nodes is somewhat non-standard so as to more closely correspond to the types of trees constructed in Algorithm~\ref{alg:reduce-tree} in the following section.
\begin{defn}[Reducing arrays]
  \label{def:reducing-array}
A {\it reducing array} of {\it depth} $k$ is a collection of $k+1$ tuples $\bt\^j := (\bt\^j_{1}, \ldots, \bt\^j_{j \wedge k}) \in \{-1,1\}^{j \wedge k}$ for $1 \leq j \leq k+1$, which satisfy the following property: $\bt\^{j+1}_{j'} = \bt\^{j}_{j'}$ for all $j' < j \leq k$, and $\bt\^{j+1}_{j} = -\bt\^j_{j}$ for $j \leq k$.\footnote{For real numbers $a,b$, we use the notation $a \wedge b$ and $a \vee b$ to denote $\min\{a, b\}$ and $\max\{a, b\}$ respectively.}
\end{defn}

\begin{defn}[Generalized trees]
  \label{def:generalized-tree}
  A {\it generalized tree} $\bx$ with {\it values in $\MX$} of {\it depth} $d$ and {\it branching factor} $k \in \BN$ is a rooted tree of depth at most $d$\footnote{By depth at most $d$, we mean that the number of edges in the path from the root to any leaf is at most $d$; this aligns with the meaning of depth for binary trees in Section \ref{sec:littlestone-prelim}.} in which each node has at most $k+1$ children. 
  Nodes of the tree without children are called its {\it leaves}. Moreover, the nodes and edges of the tree are labeled as follows:
  \begin{enumerate}
  \item Each non-leaf node $v$ is labeled with an ordered tuple of some number $k_v \leq k$ of points in $\MX$, denoted by $(\bx(v)_1, \ldots, \bx(v)_{k_v}) \in \MX^{k_v}$.
  \item The non-leaf node $v$ has $k_v+1$ children; the edge between $v$ and the $j$th child, $1 \leq j \leq k_v+1$, is labeled by a tuple $\bt\^j$, where the tuples $\bt\^j \in \{-1,1\}^{j \wedge k_v}$ form a reducing array of depth $k_v$ (Definition~\ref{def:reducing-array}).
  \end{enumerate}
  Moreover, for any node $v$ (perhaps a leaf), define $\ba(v) \in (\MX \times \{-1,1\})^\st$ (called the {\it ancestor set} of $v$) as follows: let $v\^1, \ldots, v\^{t-1}$ be the root-to-leaf path for the node $v$ and $v\^t := v$. For each $1 \leq i \leq t-1$, let $\bt\^i \in \{-1,1\}^{k\^i}$ be the label of the edge between $v\^i$ and $v\^{i+1}$, where $k\^i \leq k_{v\^{i}}$ is some positive integer. Then
  $$
\ba(v) := \{(\bx(v\^1)_1, \bt\^1_1), \ldots, (\bx(v\^1)_{k\^1}, \bt\^1_{k\^1})\} \cup \cdots \cup \{(\bx(v\^{t-1})_1, \bt\^{t-1}_1), \ldots, (\bx(v\^{t-1})_{k\^{t-1}}, \bt\^{t-1}_{k\^{t-1}})\}.
$$
The {\it height} of the node $v$ is defined to be 
$k\^1 + \cdots + k\^{t-1}$, where $k\^1, \ldots, k\^{t-1}$ are defined given $v$ as above. Note that the height of $v$ is at least 
the size of (i.e., number of tuples in) the ancestor set $\ba(v)$; the height may be even greater if there are duplicates in $\ba(v)$. 
The {\it height} of the tree $\bx$, denoted by $\height(\bx)$, is the maximum height of any node $v$ of $\bx$. Note that we must have $\height(\bx) \geq d$ if the depth of $\bx$ is $d$. To avoid ambiguity, when we wish to refer to a generalized tree, we will always use the adjective ``generalized''; ``tree'' will continue to mean ``complete binary tree''.
\end{defn}
Figure~\ref{fig:generalized-tree} shows an example of a generalized tree.

\begin{figure}[!ht]
\begin{subfigure}{0.5\textwidth}
\includegraphics[trim=150 244 150 0,clip,scale=0.55]{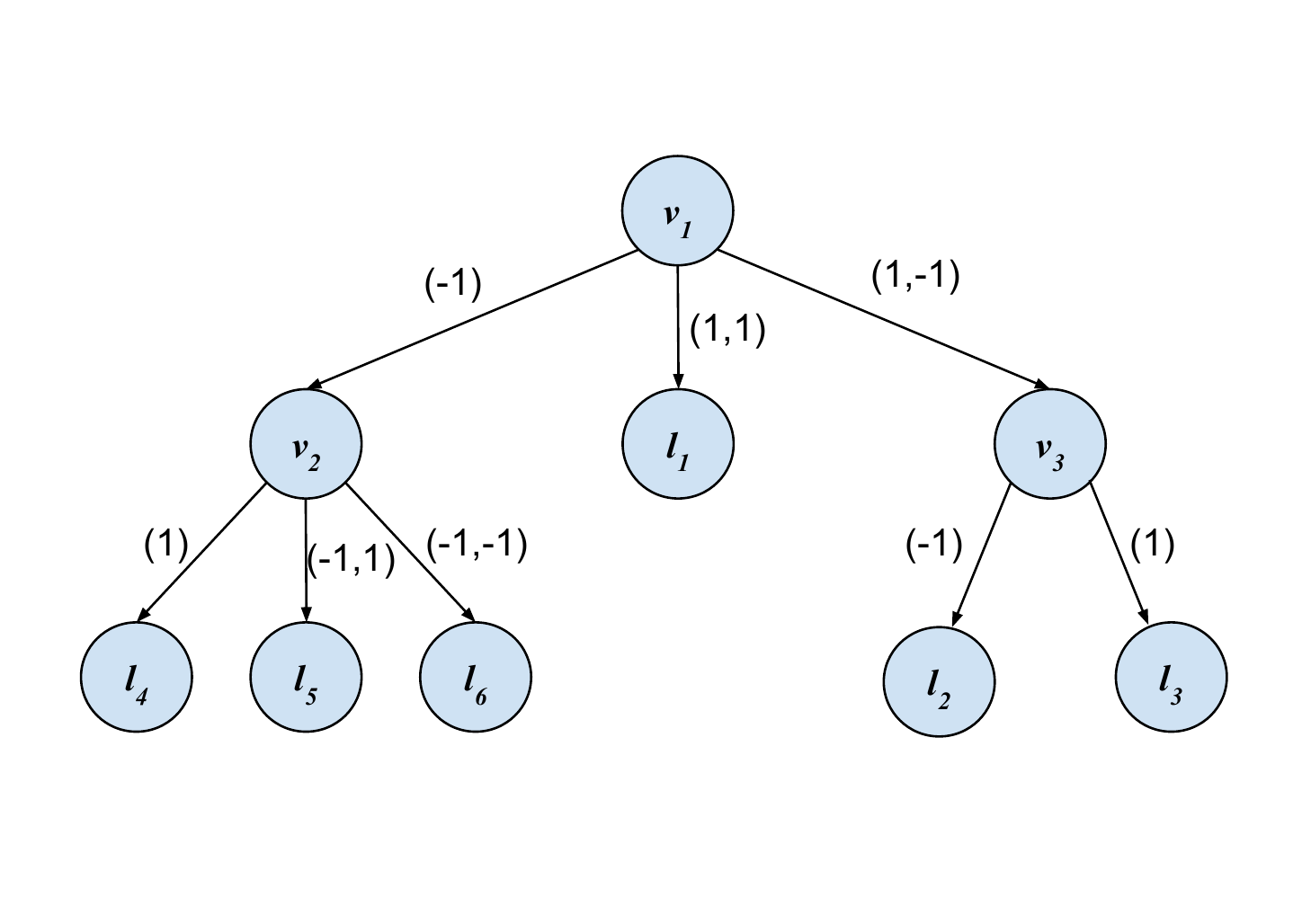}
\caption{Example of a generalized tree
}
\label{fig:generalized-tree}
\end{subfigure}
\begin{subfigure}{0.5\textwidth}
\includegraphics[trim=100 200 100 40,clip,scale=0.55]{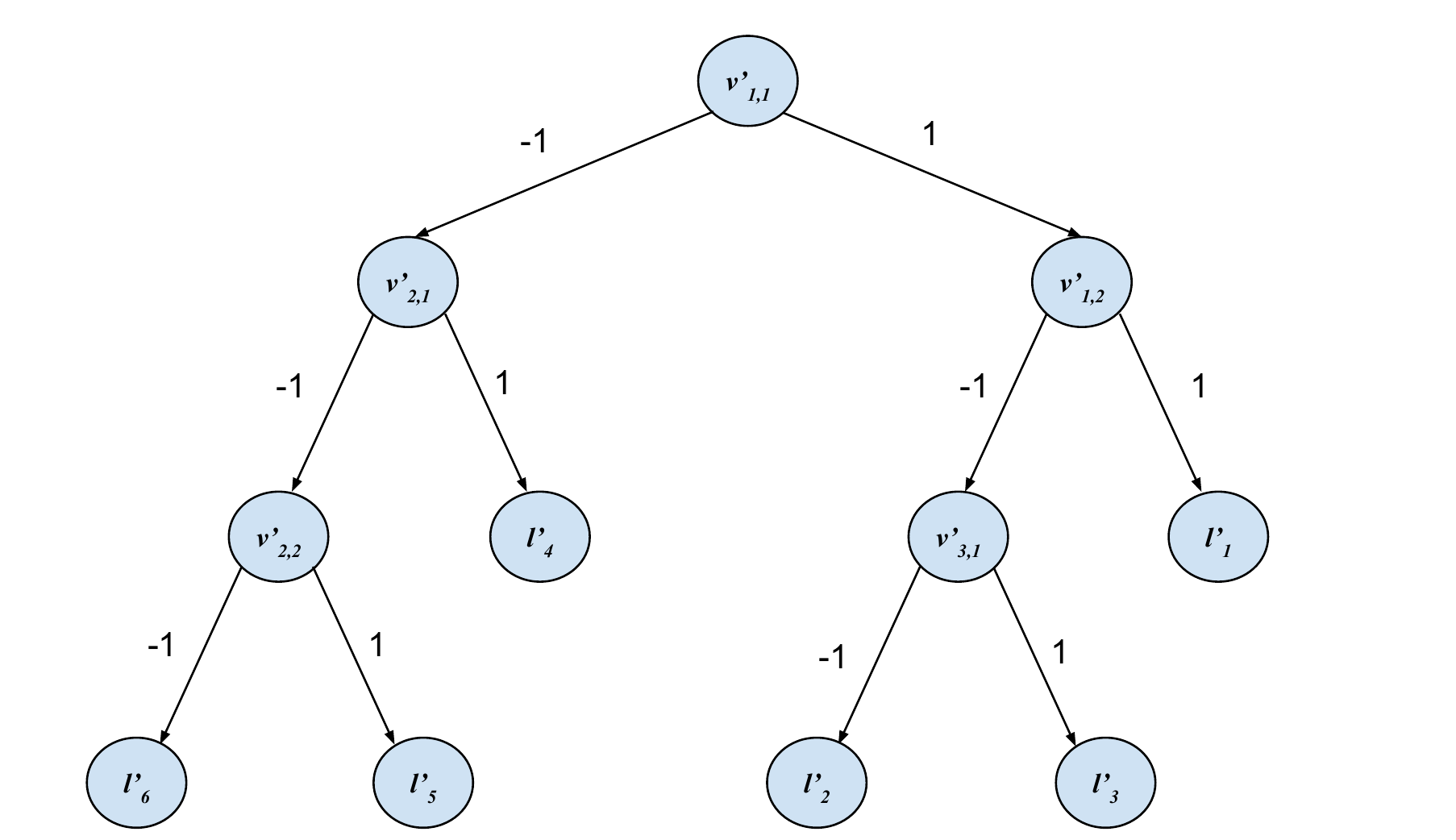}
\caption{Transformed tree
}
\label{fig:transformed-tree}
\end{subfigure}
\caption{(a) Example of a generalized tree $\bx$ of depth 2 and height 3. The tree $\bx$ has 3 non-leaf nodes, $v_1, v_2, v_3$, and 6 leaves, $\ell_1, \ldots, \ell_6$. We have $k_{v_1} = 2, k_{v_2} = 2, k_{v_3} = 1$. The tuples labeling the edges are the tuples of the reducing array corresponding to each non-leaf node $v_i$. For a few examples of ancestor sets, note that $\ba(\ell_5) = \{ (\bx_1(v_1), -1), (\bx_1(v_2), -1), (\bx_2(v_2), 1)\}$, $\ba(\ell_1) = \{(\bx_1(v_1), 1), (\bx_2(v_1), 1) \}$, and $\ba(\ell_3) = \{ (\bx_1(v_1), 1), (\bx_2(v_1), -1), (\bx_1(v_3), 1)\}$. (b) The tree $\bx'$ resulting from applying the transformation in Lemma~\ref{lem:gen-tree-leaf} to $\bx$. For each non-leaf node $v_i$, the $k_{v_i}$ nodes constructed for $v_i$ in $\bx'$ are denoted by $v'_{i,1}, \ldots, v'_{i,k_{v_i}}$. For each leaf $\ell_i$, the corresponding leaf in $\bx'$ is denoted by $\ell_i'$.}
\label{fig:tree-drawing}
\end{figure}

Lemma~\ref{lem:not-k-irred} below explains the choice of name ``reducing array'': it shows that if a class $\MG$ is not $k$-irreducible, then there is a reducing array which can be used to ``reduce the Littlestone dimension of $\MG$'' in a certain sense. As a matter of convention, if the class $\MG$ is empty we write $\Ldim(\MG) = -1$. Note also that $\Ldim(\MG) = 0$ if and only if $\MG$ contains a single hypothesis.
\begin{lemma}
  \label{lem:not-k-irred}
  Suppose that $\MG$ is not $k$-irreducible but is $(k-1)$-irreducible. Then there is a reducing array of depth $k$, denoted by $\bt\^1, \ldots, \bt\^{k+1}$ and a sequence $x_1, \ldots, x_k \in \MX$ so that for all $j \in [k+1]$, we have
  \begin{equation}
    \label{eq:reducing-array-ldim}
0 \leq \Ldim(\MG|_{(x_{1}, \bt\^j_{1}), \ldots, (x_{j \wedge k}, \bt\^j_{j\wedge k})}) < \Ldim(\MG).
  \end{equation}
\end{lemma}
\begin{proof}
  We use induction on $k$. For the base case $k=1$, note that if $\MG$ is not 1-irreducible, there must be some $x_1 \in \MX$ so that $\max\{ \Ldim(\MG|_{(x,1)}), \Ldim(\MG|_{(x,-1)}) \} < \Ldim(\MG)$. Note that as a consequence we must have that 
  $0 \leq \Ldim(\MG|_{(x,b)})$ for each $b \in \{-1,1\}$; indeed, if for some $b$ it were the case that $\MG|_{(x,b)}$ were empty, then $\MG = \MG|_{(x,-b)}$, in which case $\Ldim(\MG) = \Ldim(\MG|_{(x,-b)})$. 

  Now assume that the statement of the lemma is true for $k-1$. If, for each $x_1 \in \MX$, there were some $b_1 \in \{-1,1\}$ so that $\Ldim(\MG|_{(x_1,b_1)}) = \Ldim(\MG)$ and $\MG|_{(x_1, b_1)}$ were $(k-1)$-irreducible, then we would have that $\MG$ is $k$-irreducible, a contradiction. Thus, there is some $x_1 \in \MX$ so that one of the two conditions below holds:
  \begin{itemize}
  \item $0 \leq \Ldim(\MG|_{(x_1,1)}) \vee \Ldim(\MG|_{(x_1,-1)}) < \Ldim(\MG)$. In this case we may choose $x_2 = \cdots = x_k = x_1$ and the unique reducing array $\bt\^1, \ldots, \bt\^{k+1}$ satisfying $\bt\^1_{1} = \bt\^2_2 = \cdots = \bt\^k_{k} = 1${$, \bt\^{k+1}_k = -1$} and obtain that (\ref{eq:reducing-array-ldim}) holds. 
  \item For some $b_1 \in \{-1,1\}$, $\Ldim(\MG|_{(x_1,b_1)}) = \Ldim(\MG)$ but $\MG|_{(x_1,b_1)}$ is not $(k-1)$-irreducible (and is $(k-2)$-irreducible). In this case we must have that $\Ldim(\MG|_{(x_1,-b_1)}) \geq 0$, i.e., $\MG|_{(x_1,-b_1)}$ is nonempty; otherwise we would have that $\MG = \MG|_{(x_1,b_1)}$, which contradicts the fact that $\MG$ is $(k-1)$-irreducible. Now we apply the inductive hypothesis, which guarantees a sequence $\tilde x_2, \ldots, \tilde x_k \in \MX$ together with a reducing array of depth $k-1$, denoted $\tilde \bt\^2, \ldots, \tilde \bt\^{k+1}$, so that for $2 \leq j \leq k+1$, we have
    \begin{equation}
      \label{eq:reducing-array-induction}
0 \leq \Ldim(\MG|_{(x_1,b_1),(\tilde x_2, \bt\^j_{1}),\ldots, (\tilde x_{j \wedge k}, \bt\^j_{j\wedge k})}) < \Ldim(\MG|_{(x_1,b_1)}) = \Ldim(\MG).
\end{equation}
Now set $x_2 := \tilde x_2, \ldots, x_k := \tilde x_k$, and define the reducing array $\bt\^1, \ldots, \bt\^{k+1}$ of depth $k$ by $\bt\^1 := -b_1$, and for $2 \leq j \leq k+1$, $\bt\^j := (b_1, \tilde\bt\^j_{2}, \ldots, \tilde\bt\^j_{j\wedge k})$. Now $0 \leq \Ldim(\MG|_{(x_1,-b_1)}) < \Ldim(\MG)$ together with (\ref{eq:reducing-array-induction}) establishes (\ref{eq:reducing-array-ldim}).
\qedhere
  \end{itemize}
\end{proof}

The next lemma extends Lemma~\ref{lem:unique-leaf} to generalized trees:
\begin{lemma}
\label{lem:gen-tree-leaf}
Suppose that $\bx$ is a generalized tree so that $\height(\bx) \leq k$, and that $\MG$ is $k$-irreducible. Then $\bx$ has a unique leaf $\ell$ so that $\Ldim(\MG|_{\ba(\ell)}) = \Ldim(\MG)$.
\end{lemma}
\begin{proof}
We define a (possibly incomplete) binary tree $\bx'$ of depth $\height(\bx)$, as follows: for each non-leaf node $v$ of $\bx$ whose corresponding reducing array is denoted $\bt\^1, \ldots, \bt\^{k_v+1}$, we create $k_v$ nodes of $\bx'$, labeled by $\bx(v)_1, \ldots, \bx(v)_{k_v}$; we will refer to these nodes by $v_1', \ldots, v_{k_v}'$. For $1 \leq k < k_v$, the node $v_{k+1}'$ is a child of $v_k'$, corresponding to the bit $-\bt\^k_k$. For $1 \leq k \leq k_v$, the child of $v_k'$ corresponding to the bit $\bt\^k_k$ is the child of $v$ in $\bx$ labeled by the tuple $\bt\^k$. Finally, the child of $v_{k_v}'$ corresponding to the bit $-\bt\^{k_v}_{k_v}$ is the child of $v$ in $\bx$ labeled by the tuple $\bt\^{k_v+1}$. An example of the construction of the tree $\bx'$ is shown in Figure~\ref{fig:transformed-tree}.

It is evident that this construction induces a one-to-one mapping between leaves $\ell$ of $\bx$ and corresponding leaves $\ell'$ of $\bx'$. 
For any leaf $\ell$ of $\bx$, notice that its ancestor set $\ba(\ell)$ is equal to the ancestor set of the corresponding leaf $\ell'$ of $\bx'$.\footnote{Since incomplete binary trees are a special case of generalized trees, the definition of ancestor set in Definition~\ref{def:generalized-tree} applies to the tree $\bx'$.} Lemma~\ref{lem:unique-leaf} implies that there is a unique leaf $\ell'$ of $\bx'$ so that $\Ldim(\MG|_{\ba(\ell')}) = \Ldim(\MG)$. Thus there is a unique leaf $\ell$ of $\bx$ so that $\Ldim(\MG|_{\ba(\ell)}) = \Ldim(\MG)$.
\end{proof}

Lemma~\ref{lem:k-reducible-soa-gen} is a key part of the proof that the \ReduceTree algorithm presented in Section~\ref{sec:reducetree} can be used together with the sparse selection protocol of Section~\ref{sec:sparse-selection} to generate an (improper) private learner. Roughly speaking, it gives sufficient conditions for a generalized tree $\bx$ (which will depend on the input dataset) to have some leaf $\hat v$ so that for any hypothesis class $\MJ$ in a certain family of hypothesis classes, it holds that $\soaf{\MJ|_{\ba(\hat v)}} = \soaf{\MJ|_{S^\st}}$, where $S^\st \in (\MX \times \{-1,1\})^\st$ is a collection of $(x,y)$ pairs which will {\it not} depend on the input dataset. The statement of Lemma~\ref{lem:k-reducible-soa-gen} is in fact slightly more general (so that the preceding statement corresponds to the case $\MJ = \MJ'$ in Lemma~\ref{lem:k-reducible-soa-gen}). 
\begin{lemma}
  \label{lem:k-reducible-soa-gen}
  Fix some $k,k' \in \BN$ with $k > k'$ and hypothesis classes $\MH \subset \MG \subset \{-1,1\}^\MX$. 
  Suppose we are given  $S^\st \in (\MX \times \{-1,1\})^{k-k'}$ so that $\MH|_{S^\st}$ is $k$-irreducible, and that
  \begin{equation}
  \label{eq:yat}
  \Ldim(\MG|_{S^\st}) = \Ldim(\MH|_{S^\st}) =: \ell^\st \geq 0.
  \end{equation}
  Suppose that $\bx$ is a generalized tree so that $\height(\bx) \leq k-k'$ and for all leaves $v$ of $\bx$, $\Ldim(\MG|_{\ba(v)}) \leq \ell^\st$. 
  Then there is some leaf $\hat v$ of $\bx$ so that $\soaf{\MJ|_{S^\st}} = \soaf{\MJ'|_{\ba(\hat v)}}$ for all hypothesis classes $\MJ', \MJ$ satisfying $\MH \subset \MJ' \subset \MG$ and $\MH \subset \MJ \subset \MG$. 

  Moreover, the leaf $\hat v$ satisfies:
  \begin{enumerate}
  \item $\Ldim(\MG|_{\ba(\hat v)}) = \Ldim(\MH|_{\ba(\hat v)}) =  \ell^\st$. \label{it:ldim-star}
  \item $\MH|_{\ba(\hat v)}$ is $k'$-irreducible. \label{it:mf-irred} 
  \end{enumerate}
\end{lemma}
\begin{proof}
  The $k$-irreducibility of $\MH|_{S^\st}$, (\ref{eq:yat}), and Lemma~\ref{lem:red-order} gives that $\MG|_{S^\st}$ and $\MJ|_{S^\st}$ are also $k$-irreducible for any $\MJ \supset \MH$.

  As a consequence of the $k$-irreducibility of $\MH|_{S^\st}$ and the fact that $\height(\bx) \leq k - k'$, the following holds: there is some leaf $\hat v$ of $\bx$ so that for any $\MX$-valued tree $\by$ of depth $k'$, there are some $(\bt_1, \ldots, \bt_{k'}) \in \{-1,1\}^{k'}$ such that 
  \begin{equation}
    \label{eq:xktree-1}
\Ldim(\MH|_{S^\st \cup \ba(\hat v) \cup \{ (\by_1, \bt_1), \ldots, (\by_{k'}(\bt_{1:k'-1}), \bt_{k'})\}}) =\Ldim(\MH|_{S^\st}).
\end{equation}
(That such a $\hat v$ exists is an immediate consequence of Lemma~\ref{lem:gen-tree-leaf}; that $\hat v$ does not depend on $\by$ follows from the fact that the leaf guaranteed by Lemma~\ref{lem:gen-tree-leaf} is unique.)

Using the assumption that $\Ldim(\MH|_{S^\st}) = \ell^\st \geq \Ldim(\MG|_{\ba(\hat v)})$, we see that for any $\by$ as above, there exist $b_1, \ldots, b_{k'}$ so that
\begin{align}
  \label{eq:xktree-2}
\Ldim(\MH|_{\ba(\hat v) \cup \{ (\by_1, \bt_1), \ldots, (\by_{k'}(\bt_{1:k'-1}), \bt_{k'})\}})
\geq & \Ldim(\MH|_{S^\st \cup \ba(\hat v) \cup \{  (\by_1, \bt_1), \ldots, (\by_{k'}(\bt_{1:k'-1}), \bt_{k'})\}}) \\
\overset{{\eqref{eq:xktree-1}}}{=}&\Ldim(\MH|_{S^\st}) \nonumber\\
\label{eq:xktree-3}
\geq &  \Ldim(\MG|_{\ba(\hat v)})
\geq  \Ldim(\MH|_{\ba(\hat v)}).
\end{align}
It then follows that the inequalities in (\ref{eq:xktree-2}) and (\ref{eq:xktree-3}) are in fact equalities. For any $x \in \MX$, set $\by$ to be the tree all of whose nodes are labeled by $x$. Then the tuple $(\bt_1, \ldots, \bt_{k'})$ making (\ref{eq:xktree-1}) true (which must be unique) is of the form $(b(x), \ldots, b(x))$ for some $b(x) \in \{-1,1\}$. It follows from (\ref{eq:xktree-2}) and (\ref{eq:xktree-3}) that $\soa{\MH|_{\ba(\hat v)}}{x} = b(x)$. 

From (\ref{eq:xktree-1}) (again with all nodes of the tree $\by$ labeled by $x$) we see also that for any $x \in \MX$,
\begin{equation}
  \Ldim(\MH|_{S^\st}) = \Ldim(\MH|_{S^\st \cup \{(x,b(x))\}}),
\end{equation}
which implies that $\soa{\MH|_{S^\st}}{x} = b(x)$ for all $x \in \MX$. 
By irreducibility of $\MH|_{S^\st}$ and Lemma~\ref{lem:red-hierarchy}, we have that, for all $x \in \MX$ and $\MJ$ satisfying $\MH \subset \MJ \subset \MG$,
\begin{equation}
  \label{eq:xktree-4}
\soa{\MH|_{S^\st}}{x} =  \soa{\MJ|_{S^\st}}{x}.
\end{equation}
Hence, for all $x \in \MX$, $\soa{\MJ|_{S^\st}}{x} = b(x) = \soa{\MH|_{\ba(\hat v)}}{x}$, which establishes the desired equality of SOA hypotheses for $\MJ' = \MH$. Before establishing this for all $\MJ'$ satisfying $\MH \subset \MJ' \subset \MG$, we first show items~\ref{it:ldim-star} and~\ref{it:mf-irred}. 

Using the equalities of (\ref{eq:xktree-2}) and (\ref{eq:xktree-3}) gives that
$\Ldim(\MG|_{\ba(\hat v)}) = \Ldim(\MH|_{\ba(\hat v)})= \ell^\st$, establishing item~\ref{it:ldim-star}. Item~\ref{it:mf-irred} is a direct consequence of the equalities of (\ref{eq:xktree-2}) and (\ref{eq:xktree-3}), since the depth-$k'$ tree $\by$ is arbitrary.

 Items~\ref{it:ldim-star} and~\ref{it:mf-irred} together with Lemma~\ref{lem:red-hierarchy} imply that for any hypothesis class $\MJ'$ satisfying $\MH \subset \MJ' \subset \MG$, we have that 
$$
\soaf{\MJ'|_{\ba(v)}} = \soaf{\MH|_{\ba(v)}} = \soaf{\MJ|_{S^\st}},
$$
as desired. 
\end{proof}

\section{Improper private learner for Littlestone classes}
\label{sec:improper-learner}
In this section we establish a variant of Theorem~\ref{thm:pap-pac-informal} where the learner is only guaranteed to be improper (namely, part 1 of the proof as described in Section~\ref{sec:proof-overview}). In Section~\ref{sec:reducetree} we introduce the algorithm \ReduceTree (defined in Algorithm~\ref{alg:reduce-tree}), which, given a dataset $S_n = ((x_1,y_1), \ldots, (x_n,y_n))$ outputs a set $\hat\MS$ of hypotheses of the form $\soaf{\MG}$ for various classes $\MG \subset \MF$. (To aid our notation in the proofs we also have \ReduceTree output a generalized tree $\hat \bx$ and a set $\hMLp$ of leaves of $\hat \bx$.) In Section~\ref{sec:sparse-selection} we state guarantees for a private sparse selection protocol from~\cite{ghazi_differentially_2020}. In Section~\ref{sec:improper-alg} we will show how to use certain ``stability-type'' properties of the set $\hat\MS$ together with the private sparse selection procedure from Section~\ref{sec:sparse-selection} to privately output a hypothesis which has low population error, which will establish the desired improper learner. (We will then make it proper in Section~\ref{sec:proper-learning-finite}, thus establishing Theorem~\ref{thm:pap-pac-informal} in its entirety.)
\subsection{Building block: \ReduceTree algorithm}
\label{sec:reducetree}

Throughout this section we fix a function class $\MF$ and write $d := \Ldim(\MF)$; we also assume that a given distribution $P$ over $\MX \times \{-1,1\}$ is realizable for $\MF$. The algorithm \ReduceTree takes some $\gamma \in [0,1], n, k' \in \BN$ as parameters (to be specified below). It also takes as input some dataset $S_n \in (\MX \times \{-1,1\})^n$, which is accessed via the empirical distribution $\emp$. 
Recall the function $\alpha(n,\gamma) \in [0,1]$ defined after Theorem~\ref{thm:unif-conv}. 
Let $\alpha_\Delta := 6 \cdot \alpha(n,\gamma)$, and define $E_{\good}$ to be the event
\begin{equation}
\label{eq:egood}
E_{\good} := \left\{ \sup_{f \in \MF} \left| \err{\di}{f} - \err{\emp}{f} \right| \leq \frac{\alpha_\Delta}{6} \right \}.
\end{equation}
Assuming that the dataset $S_n = \{ (x_1, y_1), \ldots, (x_n, y_n) \}$ is distributed i.i.d.~according to $P$, by Theorem~\ref{thm:unif-conv}, $\p_{S_n \sim \di^n} \left[ E_{\good} \right] \geq 1-\gamma$. 

The algorithm \ReduceTree operates as follows. It starts with the class $\MF_{\emp,\alpha_0}$ of hypotheses with empirical error at most $\alpha_0$ ($\alpha_0$ will be chosen so that it is less than the desired error for the output of the private learner). If it is the case that $\Ldim(\MF_{\emp, \alpha_0}) = \Ldim(\MF_{\emp, \alpha_0 - \alpha_\Delta})$ and $\MF_{\emp,\alpha_0 - \alpha_\Delta}$ is irreducible for some appropriately chosen $\alpha_\Delta > 0$, then by Lemmas~\ref{lem:red-order} and~\ref{lem:red-hierarchy}, under the event $E_{\good}$, the classifier $\soaf{\MF_{\emp,\alpha_0 - 2\alpha_\Delta/3}}$ is ``stable'' in the sense that it does not ``depend much'' on the dataset $S_n$. (We leave a formalization of this statement to the proof below.) In this case we can output the set consisting of the single hypothesis, $\hat\MS := \{ \soaf{\MF_{\emp, \alpha_0 - 2\alpha_\Delta/3}} \}$.

If we did not terminate in the above paragraph, then one of the following two statements must hold: (1) it holds that $\Ldim(\MF_{\emp, \alpha_0 - \alpha_\Delta}) < \Ldim(\MF_{\emp, \alpha_0})$, or (2) $\MF_{\emp,\alpha_0 - \alpha_\Delta}$ is not irreducible {and $\Ldim(\MF_{\emp, \alpha_0 - \alpha_\Delta}) = \Ldim(\MF_{\emp, \alpha_0})$}. If (1) holds, then we may simply recurse, i.e., repeat the above process with $\alpha_1 := \alpha_0 - \alpha_\Delta$ replacing $\alpha_0$. Otherwise, (2) holds, so (by Lemma~\ref{lem:not-k-irred} with $k=1$) we can choose some $x \in \MX$ so that $\Ldim(\MF_{\emp, \alpha_0 - \alpha_\Delta}|_{(x,b)}) < \Ldim(\MF_{\emp, \alpha_0 - \alpha_\Delta}) = \Ldim(\MF_{\emp,\alpha_0})$ for each $b \in \{-1,1\}$. In this case, we can repeat the above process twice (once for each $b \in \{-1,1\}$), with $\alpha_1 := \alpha_0 - \alpha_\Delta$ replacing $\alpha_0$ and with $\MF|_{(x,b)}$ replacing $\MF$ for each $b \in \{-1,1\}$. The point $x \in \MX$ becomes the label of the root node of a generalized tree, with two children (which are leaves), corresponding to the bits $\pm 1$. Each step $t \geq 1$ of the above-described recursion builds upon this generalized tree maintained by the algorithm by adding children to some of its current leaves. 
For technical reasons, at depth $t$ of this recursion, we will need to replace the requirement of ``irreducibility'' with that of ``$k' \cdot 2^t$-irreducibility'', for some $k'$ which does not depend on $t$.  The algorithm is guaranteed to terminate because with each increase in $t$, the Littlestone dimension of the current class under consideration decreases, and it can only do so at most $\Ldim(\MF)$ times. Further details may be found in Algorithm~\ref{alg:reduce-tree}.

\begin{algorithm}[!htp]
  \caption{\bf \ReduceTree}\label{alg:reduce-tree}
  \KwIn{Parameters $n, k' \in \BN, \gamma \in [0,1]$, $\alpha_\Delta := 6 \cdot \alpha(n,\gamma)$. Distribution $\emp$ over $\MX$. Hypothesis class $\MF$, with $d := \Ldim(\MF)$.}
  \begin{enumerate}[leftmargin=14pt,rightmargin=20pt,itemsep=1pt,topsep=1.5pt]
\item Initialize a counter $t = 1$ ($t$ counts the depth of the generalized tree constructed at each step of the algorithm).
\item For $1 \leq t \leq d+1$, set $\alpha_t := (d+3-t) \cdot \alpha_\Delta$. 
\item For $1 \leq t \leq d$, set $k_t := k' \cdot 2^{t}$.\label{it:define-kt} 
\item Initialize $\hat \bx\^0 = \{ v_0 \}$ to be a tree with a single (unlabeled) leaf $v_0$. (In general $\hat \bx\^t$ will be the generalized tree produced by the algorithm after step $t$ is completed.)
\item Initialize $\hat \ML_1 = \{ v_0 \}$. (In general $\hat \ML_t$ will be the set of leaves of the tree before step $t$ is started.)
\item For $t \in \{1,2, \ldots, d\}$:
  \begin{enumerate}
  \item For each leaf $v \in \hat \ML_{t}$ and $\alpha \in [0,1]$, set $\gRes{\alpha}{v} := \MF_{\emp, \alpha}|_{\ba(v)}$. (Note that since the only way the tree changes from round to round is by adding children to existing nodes, $\ba(v)$ will never change for a node $v$ that already exists.)
  \item \label{it:sup-ldim-alg} Let $\hat w_t^\st := \max_{v \in \hat \ML_{t}} \Ldim(\gRes{\alpha_t}{v})$ be the maximum Littlestone dimension of any of the classes $\gRes{\alpha_t}{v}.$ 

    Also let $\hat \ML_{t}' := \{ v \in \hat \ML_{t} : \Ldim(\gRes{\alpha_t}{v}) = \hat w_t^\st\}$.
  \item \label{it:ldim-break-step} 
  If there is some $v \in \hat \ML_{t}'$ so that $\Ldim(\hat{\MG}(\alpha_t-\alpha_\Delta, v)) = \Ldim(\hat{\MG}(\alpha_t, v))$ and $\hat{\MG}(\alpha_t - \alpha_\Delta, v)$ is $k_t$-irreducible, then break out of the loop and go to step~\ref{it:tfinal-define}. 
  \item \label{it:red-loop} Else, for each node $v \in \hat\ML_{t}'$:
    \begin{enumerate}
    \item If $\Ldim(\hat{\MG}(\alpha_t-\alpha_\Delta, v)) < \Ldim(\hat{\MG}(\alpha_t, v))$, move on to the next $v$.
    \item \label{it:decrease-ldim} Else, we must have that $\hat{\MG}(\alpha_t - \alpha_\Delta, v)$ is not $k_t$-irreducible. Let $k_v$ be chosen as small as possible so that $\gRes{\alpha_t - \alpha_\Delta}{v}$ is not $k_v$-irreducible; then $k_v \leq k_t$. By Lemma~\ref{lem:not-k-irred}, there is some sequence $x_1, \ldots, x_{k_v} \in \MX$ and reducing array $\bt\^1, \ldots, \bt\^{k_v+1}$ of depth $k_v$ so that for $1 \leq j \leq k_v+1$, it holds that
      \begin{equation}
        \label{eq:alg-dec-ldim}
    0 \leq  \Ldim(\hat{\MG}(\alpha_t - \alpha_\Delta, v)|_{(x_1, \bt\^j_{1}), \ldots, (x_{j \wedge k_v}, \bt\^j_{j \wedge k_v})}) < \Ldim(\hat{\MG}(\alpha_t - \alpha_\Delta, v)).
    \end{equation}
      \item Give $v$ the label $(x_1, \ldots, x_{k_v})$. Construct $k_v+1$ children of $v$ (all leaves of the current tree), with edge labels given by $\bt\^1, \ldots, \bt\^{k_v+1}$.
      \end{enumerate}
      \item Let the current tree (with the additions of the previous step) be denoted by $\hat \bx\^{t}$, and let $\hat \ML_{t+1}$ be the list of the leaves of $\hat \bx\^{t}$, i.e., the nodes which have not (yet) been assigned labels or children.
  \end{enumerate}
\item \label{it:tfinal-define} Let $t_{\final}$ be the final value of $t$ the algorithm {\it completed} the loop of step~\ref{it:red-loop} for before breaking out of the above loop (i.e., if the break at step~\ref{it:ldim-break-step} was taken at step $t$, then $t_{\final} = t-1$; if the break was never taken, then $t_{\final} = d$). Let $\hat w_{t_{\final}+1}^\st$ and $\hat \ML_{t_{\final}+1}'$ be defined as in Step~\ref{it:sup-ldim-alg}.
\item
Output the set $\hMLp := \hat \ML_{t_{\final}+1}'$ of leaves of the tree $\hat \bx\^{t_{\final}}$, and the tree $\hat \bx := \hat \bx\^{t_{\final}}$. 
  Finally, output the set 
  \begin{equation}
    \label{eq:redtree-output}
\hat \MS := \{ \soaf{\gRes{\alpha_{t_{\final} + 1} - 2\alpha_\Delta/3}{v}} : \text{$v \in \hMLp$ and $\gRes{\alpha_{t_{\final}+1} - 2\alpha_\Delta/3}{v}$ is $k'$-irreducible \& nonempty} \}. 
  \end{equation}
  \end{enumerate}
  \end{algorithm}

We say that the dataset $S_n$ is {\it realizable} if there is some $f \in \MF$ so that $\err{\emp}{f} = 0$ (this is the case with probability 1 if $S_n \sim P^n$ and $P$ is realizable). Lemma~\ref{lem:exists-irred-node} states a basic property of the output set $\hMLp$ of \ReduceTree.
\begin{lemma}
  \label{lem:exists-irred-node}
Suppose the input dataset $S_n$ of \ReduceTree is realizable. The set $\hMLp$ output by \ReduceTree satisfies the following property: letting $t = t_{\final} + 1 \in [d+1]$, there is some leaf $v \in \hMLp$ so that $\Ldim(\gRes{\alpha_{t}-\alpha_\Delta}{v}) = \Ldim(\gRes{\alpha_{t}}{v}) \geq 0$ and $\gRes{\alpha_{t} - \alpha_\Delta}{v}$ is $k_{t}$-irreducible.
\end{lemma}
\begin{proof}
  If for some $t$, the algorithm breaks at step~\ref{it:ldim-break-step}, then the conclusion of the lemma is immediate: indeed, the condition to break in step~\ref{it:ldim-break-step} gives that for some $v \in \hat \ML_t' = \hat \ML_{t_{\final}+1}'$ we have $\Ldim(\gRes{\alpha_{t}-\alpha_\Delta}{v}) = \Ldim(\gRes{\alpha_{t}}{v})$ and $\gRes{\alpha_{t} - \alpha_\Delta}{v}$ is $k_{t}$-irreducible. In light of (\ref{eq:tree-divide}) below, since $v$ maximizes $\Ldim(\gRes{\alpha_t-\alpha_\Delta}{v})$ among $v \in \hat \ML_t$, it follows that $\Ldim(\gRes{\alpha_t-\alpha_\Delta}{v}) \geq 0$. 

  Otherwise, the algorithm performs a total of $d$ iterations. We claim that $\hat w_{d+1}^\st = 0$. 
  We first show that for all $t \geq 1$, $\hat w_{t+1}^\st < \hat w_t^\st$. To see this fact, note that each leaf $v$ in $\hat \ML_{t+1}$ belongs to one of the following three categories:
  \begin{itemize}
  \item $v \in \hat \ML_t \backslash \hat \ML_t'$. In this case, we have
    $$
\Ldim(\gRes{\alpha_{t+1}}{v}) \leq \Ldim(\gRes{\alpha_{t}}{v}) < \hat w_t^\st.
$$
\item $v \in \hat \ML_t'$ and $\Ldim(\gRes{\alpha_t-\alpha_\Delta}{v}) < \Ldim(\gRes{\alpha_t}{v})$. Using that $\alpha_{t+1} = \alpha_t - \alpha_\Delta$, we obtain
  $$
\Ldim(\gRes{\alpha_{t+1}}{v}) = \Ldim(\gRes{\alpha_t - \alpha_\Delta}{v}) < \Ldim(\gRes{\alpha_t}{v}) \leq \hat w_t^\st.
  $$
\item $v$ is the $j$th child of some node $u \in \hat \ML_t'$ for some $1 \leq j \leq k_u+1 \leq k_t+1$, as constructed in step~\ref{it:decrease-ldim} of the algorithm. Let the label of the edge between $u$ and $v$ be denoted by $\bt\^j \in \{-1,1\}^{j \wedge k_u}$. Then since $\alpha_{t+1} = \alpha_t - \alpha_\Delta$,
  \begin{align}
 \Ldim(\gRes{\alpha_{t+1}}{v}) & \leq \Ldim(\gRes{\alpha_t - \alpha_\Delta}{v})\nonumber\\
                                                    & = \Ldim(\MF_{\emp, \alpha_t-\alpha_\Delta}|_{\ba(u) \cup \{(\bx\^t(u)_1, \bt\^j_{1}), \ldots, (\bx\^t(u)_{j \wedge k_u}, \bt\^j_{j\wedge k_u}) \}}) \nonumber\\
                                                    & < \Ldim(\gRes{\alpha_t-\alpha_\Delta}{u}) \nonumber\\
                                                    & \leq \hat w_t^\st,\nonumber
  \end{align}
  where the strict inequality follows from (\ref{eq:alg-dec-ldim}). 
  \end{itemize}
  
  Since $\hat w_1^\st \leq d$ as $\gRes{\alpha_1}{v} \subset \MF$, we obtain that $\hat w_{d+1}^\st \leq 0$. 
  Thus all leaves $v$ in $\hat \ML_{t+1}$ satisfy $\Ldim(\gRes{\alpha_{t+1}}{v}) < \hat w_t^\st$, i.e., $\hat w_{t+1}^\st < \hat w_t^\st$. Thus each $v \in \hat \ML_{d+1}$ satisfies $\Ldim(\gRes{\alpha_{d+1}}{v}) \leq 0$, and $\hat \ML_{d+1}' = \hMLp$ is exactly the set of $v \in \hat \ML_{d+1}$ for which $\gRes{\alpha_{d+1}}{v}$ is nonempty.
  Hence 
  \begin{equation}
    \label{eq:tree-divide}
\MF_{\emp,\alpha} = \bigcup_{v \in \hat\ML} \MF_{\emp,\alpha}|_{\ba(v)} = \bigcup_{v \in \hat\ML} \gRes{\alpha}{v}
\end{equation}
for all $\alpha \leq \alpha_{d+1}$. Since we assume $S_n$ is realizable, it follows that $\MF_{\emp, \alpha}$ is nonempty and thus $\max_{v \in \hMLp} \{\Ldim(\gRes{\alpha}{v})\} \geq 0$ for $\alpha \geq 0$. Since $\alpha_{d+1} -\alpha_\Delta \geq 0$, it follows that $\hat w_{d+1}^\st = 0$ and also that there is some $v \in \hMLp$ so that $\Ldim(\gRes{\alpha_{d+1} - \alpha_\Delta}{v}) \geq 0$. Since $\hat w_{d+1}^\st = 0$, we have  $\Ldim(\gRes{\alpha_{d+1}}{v}) \leq 0$, so for this $v \in \hMLp$, $\Ldim(\gRes{\alpha_{d+1}}{v}) = \Ldim(\gRes{\alpha_{d+1}-\alpha_\Delta}{v}) = 0$. The $k_{d+1}$-irreducibility of $\hat \MG(\alpha_{d+1} - \alpha_\Delta, v)$ follows from the fact that a class with Littlestone dimension 0 contains a single function, and is thus $k$-irreducible for all $k \in \BN$. 
\end{proof}

In order to apply Lemma~\ref{lem:k-reducible-soa-gen} in the proof of Lemma~\ref{lem:equal-to-a-leaf} below, we will need an upper bound on $\height(\hat \bx)$ for the tree $\hat \bx$ output by \ReduceTree. Lemma~\ref{lem:height-ub} provides this upper bound; roughly speaking, the growth is exponential in $t$ (recall $k_t = k' \cdot 2^t$ from step~\ref{it:define-kt} of Algorithm~\ref{alg:reduce-tree}) because the tree may grow in height by $k_t$ with each increase of $t$ by 1 (due to step~\ref{it:decrease-ldim} of the algorithm), and in order to satisfy the preconditions of Lemma~\ref{lem:k-reducible-soa-gen} we need to ensure that $k_{t+1}$ is an upper bound on $\height(\hat \bx\^{t})$ for each $t$. 
\begin{lemma}
  \label{lem:height-ub}
  For all $t$ the tree $\hat\bx\^t$ of Algorithm~\ref{alg:reduce-tree} satisfies $\height(\hat \bx\^t) \leq k_{t+1} - k'$. In particular, the tree $\hat \bx$ satisfies $\height(\hat \bx) \leq k_{t_{\final}+1} - k'$. 
\end{lemma}
\begin{proof}
  We prove that $\height(\hat \bx\^t) \leq k_{t+1} - k' = k' \cdot 2^{t+1} - k'$ by induction. For the base case, note that $\height(\hat \bx\^0) = 0 < 2k' - k' = k' \cdot 2^1 - k'$. Since at step~\ref{it:decrease-ldim} of the algorithm, in the $t$th iteration, each leaf is labeled with a tuple of length at most $k_t$, it follows that
  $$
  \height(\hat \bx\^t) \leq \height(\hat \bx\^{t-1}) + k_t \leq k_t - k' + k_t = k_{t+1} - k',
  $$
  for all $t \geq 1$. The lemma statement follows since $\hat \bx = \hat \bx\^{t_{\final}}$. 
\end{proof}

For each $\alpha \in [0,1]$ and $t \in [d+1]$, define the set:
\begin{equation}
\label{eq:big-max}
\MM_{\alpha,t} := \left\{ S \in (\MX \times \{-1,1\})^{k_t - k'}  : \substack{
\text{$\MF_{P,\alpha-\alpha_\Delta/3}|_{S}$ is $k_t$-irreducible and nonempty,}  \\ 
\text{and $\Ldim(\MF_{\di,\alpha-\alpha_\Delta/3}|_{S}) = \Ldim(\MF_{\di,\alpha+\alpha_\Delta/3}|_{S})$}
}  \right\}
\end{equation}
\begin{lemma}
\label{lem:big-max-defined}
Suppose that the event $E_{\good}$ occurs. Then for $t = t_{\final} + 1$, the set $\MM_{\alpha_t - \alpha_\Delta/2, t}$ is nonempty.
\end{lemma}
\begin{proof}
Let $v$ be a node as guaranteed by Lemma~\ref{lem:exists-irred-node}, i.e., so that $\gRes{\alpha_t - \alpha_\Delta}{v}$ is $k_t$-irreducible, and so that $\Ldim(\gRes{\alpha_t - \alpha_\Delta}{v}) = \Ldim(\gRes{\alpha_t}{v}) \geq 0$. Since the event $E_{\good}$ holds,
$$
\gRes{\alpha_t - \alpha_\Delta}{v} = \MF_{\emp, \alpha_t - \alpha_\Delta}|_{\ba(v)} \subset \MF_{\di, \alpha_t - 5\alpha_\Delta/6}|_{\ba(v)} \subset \MF_{\di, \alpha_t - \alpha_\Delta/6}|_{\ba(v)} \subset \MF_{\emp, \alpha_t}|_{\ba(v)} = \gRes{\alpha_t}{v}.
$$
It follows from Lemma~\ref{lem:red-order} that $\MF_{\di, \alpha_t - 5\alpha_\Delta/6}|_{\ba(v)}$ is $k_t$-irreducible and that $\Ldim(\MF_{\di, \alpha_t - \alpha_\Delta/6}|_{\ba(v)}) = \Ldim(\MF_{\di, \alpha_t - 5\alpha_\Delta/6}|_{\ba(v)})$. Since the height of the tree $\hat \bx\^{t-1} = \hat \bx\^{t_{\final}}$ is at most $k_t - k'$ (Lemma~\ref{lem:height-ub}), it follows that the number of tuples in $\ba(v)$ is at most $k_t - k'$; thus, after duplicating some of the tuples in $\ba(v)$ if necessary, we get that $\ba(v) \in \MM_{\alpha-\alpha_\Delta/2,t}$. 
\end{proof}

For any $\alpha \in [0,1],t \in [d+1]$ for which $\MM_{\alpha,t}$ is nonempty, define:
  \begin{align}
  \label{eq:big-arg-max}
  S_{\alpha,t}^\st \in \argmax_{S \in \MM_{\alpha,t}} \left\{ \Ldim(\MF_{\di, \alpha} |_{S})\right\}, \qquad \ell_{\alpha,t}^\st := \max_{S \in \MM_{\alpha,t}} \{ \Ldim(\MF_{\di,\alpha}|_S) \} \geq 0.
\end{align}
Also set
\begin{equation}
\label{eq:sigma-at}
\sigma_{\alpha,t}^\st := \soaf{\MF_{\di,\alpha}|_{S_{\alpha,t}^\st}}. 
\end{equation}
We emphasize here that $\MM_{\alpha,t}$ and $S_{\alpha,t}^\st$ are both independent of the output of the algorithm \ReduceTree (and in particular, they do not depend on the particular input dataset $S_n$).

\begin{lemma}
  \label{lem:equal-to-a-leaf}
  Under the event $E_{\good}$, the following holds: for $t = t_{\final}+1 \in [d+1]$ and some leaf $\hat v \in \hMLp$, we have $\sigma_{\alpha_t - \alpha_\Delta/2, t}^\st = \soaf{\gRes{\alpha_t - 2\alpha_\Delta/3}{\hat v}}$. (In particular, for this $t$, $\sigma_{\alpha_t - \alpha_\Delta/2, t}^\st$ is well-defined, i.e., $\MM_{\alpha_t - \alpha_\Delta/2,t}$ is nonempty.)
  
  Moreover, ${\gRes{\alpha_t - 2\alpha_\Delta/3}{\hat v}}$ is $k'$-irreducible and nonempty, and $\Ldim({\gRes{\alpha_t - 2\alpha_\Delta/3}{\hat v}}) = \ell_{\alpha_t - \alpha_\Delta/2,t}^\st \geq 0$.
\end{lemma}
\begin{proof}
  By Lemma~\ref{lem:exists-irred-node}, for $t = t_{\final} + 1 \in [d+1]$, there is some leaf $v' \in \hMLp$ so that $\Ldim(\MF_{\emp,\alpha_t - \alpha_\Delta}|_{\ba(v')}) = \Ldim(\MF_{\emp, \alpha_t}|_{\ba(v')}) \geq 0$ and $\MF_{\emp, \alpha_t - \alpha_\Delta}|_{\ba(v')}$ is $k_{t}$-irreducible. Under the event $E_{\good}$, for each node $v$ of the tree $\hat \bx$ output by the algorithm, we have that
  \begin{align}
  \MF_{\emp, \alpha_t - \alpha_\Delta}|_{\ba(v)} \subset
  \MF_{\di,\alpha_t - 5 \alpha_\Delta/6}|_{\ba(v)} \subset
    \MF_{\emp, \alpha_t - 4\alpha_\Delta/6}|_{\ba(v)} \nonumber\\
    \label{eq:6-inclusions}
    \subset
    \MF_{\di, \alpha_t - 3\alpha_\Delta/6}|_{\ba(v)} \subset
  \MF_{\emp, \alpha_t - 2\alpha_\Delta/6}|_{\ba(v)} \subset
  \MF_{\di, \alpha_t - \alpha_\Delta/6}|_{\ba(v)} \subset \MF_{\emp, \alpha_t}|_{\ba(v)}.
  \end{align}
  Now we apply Lemma~\ref{lem:k-reducible-soa-gen} with $\MJ = \MJ' = \MF_{\di, \alpha_t - \alpha_\Delta/2}$, $\MH = \MF_{\di, \alpha_t - 5\alpha_\Delta/6}$, $\MG = \MF_{\di, \alpha_t - \alpha_\Delta/6}$, $k = k_t, k' = k'$, $\bx$ equal to the tree $\hat \bx = \hat \bx\^{t_{\final}}$ output by \ReduceTree, and $S^\st = S_{\alpha_t - \alpha_\Delta/2,t}^\st$. Since $t = t_{\final} + 1$, Lemma~\ref{lem:big-max-defined} guarantees that $S^\st$ is well-defined (in particular, that $\MM_{\alpha_t - \alpha_\Delta/2,t}$ is nonempty).  We check that the preconditions of Lemma~\ref{lem:k-reducible-soa-gen} hold: Notice that (\ref{eq:yat}) holds by the definitions (\ref{eq:big-max}) and (\ref{eq:big-arg-max}), and that $\MH|_{S^\st} = \MF_{\di, \alpha_t - \alpha_\Delta/2 - \alpha_\Delta/3}|_{S^\st}$ is $k_t$-irreducible, also by (\ref{eq:big-max}) and (\ref{eq:big-arg-max}). By definition of $\ell_{\alpha,t}^\st$ in (\ref{eq:big-arg-max}), we have
  $$
  \ell_{\alpha_t - \alpha_\Delta/2,t}^\st = \Ldim(\MF_{\di, \alpha_t - 5\alpha_\Delta/6}|_{S^\st}) = \Ldim(\MF_{\di, \alpha_t - \alpha_\Delta/6}|_{S^\st}).
  $$
 Lemma~\ref{lem:height-ub} establishes that $\height(\hat \bx) \leq k_t - k'$. Also, from the guarantee on $v'$ in Lemma~\ref{lem:exists-irred-node}, (\ref{eq:6-inclusions}), and Lemma~\ref{lem:red-order}, we have  $\Ldim(\MF_{\di,\alpha_t-5\alpha_\Delta/6}|_{\ba(v')}) = \Ldim(\MF_{\di,\alpha_t - \alpha_\Delta/6}|_{\ba(v')})$ and $\MF_{\di,\alpha_t - 5\alpha_\Delta/6}|_{\ba(v')}$ is $k_t$-irreducible. Thus $\ba(v') \in \MM_{\alpha_t - \alpha_\Delta/2, t}$, so by definition of $\ell_{\alpha,t}^\st$, 
  $$
\ell_{\alpha_t - \alpha_\Delta/2,t}^\st \geq \Ldim(\MF_{\di, \alpha_t - \alpha_\Delta/2}|_{\ba(v')}).
$$
Moreover, for any other leaf $u$ of the tree $\hat\bx$, we have, by definition of $\hMLp = \hat \ML_{t_{\final} + 1}'$,
$$
\Ldim(\MF_{\di,\alpha_t - \alpha_\Delta/6}|_{\ba(u)}) \leq \Ldim(\MF_{\emp, \alpha_t}|_{\ba(u)}) \leq \Ldim(\MF_{\emp, \alpha_t}|_{\ba(v')}) = \Ldim(\MF_{\di, \alpha_t - \alpha_\Delta/2}|_{\ba(v')}) \leq \ell_{\alpha_t - \alpha_\Delta/2,t}^\st.
$$
(In more detail, the first inequality above holds due to (\ref{eq:6-inclusions}), the second inequality is due to the fact that $v' \in \hMLp = \hat \ML_{t_{\final} +1}'$ (see step~\ref{it:sup-ldim-alg} of \ReduceTree), and the equality holds due to (\ref{eq:6-inclusions}) and $\Ldim(\MF_{\emp,\alpha_t - \alpha_\Delta}|_{\ba(v')}) = \Ldim(\MF_{\emp, \alpha_t}|_{\ba(v')})$). 
Then the hypotheses of Lemma~\ref{lem:k-reducible-soa-gen} hold and letting $\MJ' = \MJ = \MF_{\di, \alpha_t - \alpha_\Delta/2}$, it follows that for some leaf $\hat v$ of $\hat \bx$, we have
$$
\sigma_{\alpha_t - \alpha_\Delta/2, t}^\st = \soaf{\MF_{\di, \alpha_t - \alpha_\Delta/2}|_{S^\st}} = \soaf{\MF_{\di, \alpha_t-\alpha_\Delta/2}|_{\ba(\hat v)}}
$$
as well as $\Ldim(\MF_{\di,\alpha_t-5\alpha_\Delta/6}|_{\ba(\hat v)}) = \Ldim(\MF_{\di,\alpha_t-\alpha_\Delta/6}|_{\ba(\hat v)}) = \ell_{\alpha_t - \alpha_\Delta/2,t}^\st$, and that $\MF_{\di,\alpha_t-5\alpha_\Delta/6}|_{\ba(\hat v)}$ is $k'$-irreducible. From (\ref{eq:6-inclusions}), it follows that $\Ldim(\MF_{\emp,\alpha_t-4\alpha_\Delta/6}|_{\ba(\hat v)}) = \Ldim(\MF_{\emp,\alpha_t-2\alpha_\Delta/6}|_{\ba(\hat v)}) = \ell_{\alpha_t - \alpha_\Delta/2,t}^\st \geq 0$, and that $\MF_{\emp,\alpha_t-4\alpha_\Delta/6}|_{\ba(\hat v)} =\gRes{\alpha_t - 2\alpha_\Delta/3}{\hat v}$ is $k'$-irreducible. 
Then by (\ref{eq:6-inclusions}) and Lemma~\ref{lem:red-hierarchy}, we have
$$
\sigma_{\alpha_t - \alpha_\Delta/2,t}^\st = \soaf{\MF_{\di, \alpha_t - \alpha_\Delta/2}|_{\ba(\hat v)}} = \soaf{\MF_{\emp, \alpha_t - 2\alpha_\Delta/3}|_{\ba(\hat v)}} = \soaf{\gRes{\alpha_t - 2\alpha_\Delta/3}{\hat v}}.
$$

Finally we check that $\hat v \in \hMLp = \hat \ML_t' = \hat \ML_{t_{\final}+1}'$, i.e., all leaves $u$ of the tree $\hat \bx$ satisfy $\Ldim(\MF_{\emp, \alpha_t}|_{\ba(u)}) \leq \Ldim(\MF_{\emp, \alpha_t}|_{\ba(\hat v)})$. This is a consequence of the fact that for all such $u$,
$$
\Ldim(\MF_{\emp,\alpha_t}|_{\ba(\hat v)}) \geq \Ldim(\MF_{\emp,\alpha_t - 2\alpha_\Delta/6}|_{\ba(\hat v)})  =  \ell_{\alpha_t - \alpha_\Delta/2, t}^\st \geq \Ldim(\MF_{\emp,\alpha_t}|_{\ba(v')}) \geq \Ldim(\MF_{\emp,\alpha_t}|_{\ba(u)}),
$$
since $v' \in \hMLp$.
\end{proof}

\begin{lemma}
  \label{lem:s-ub}
The set $\hat \MS$ output by \ReduceTree has size $|\hat \MS| \leq \prod_{t=1}^{d} (k_t + 1)$.
\end{lemma}
\begin{proof}
It suffices to show that for $t \in [d]$, the tree $\bx\^t$ has at most $\prod_{t'=1}^t (k_{t'} + 1)$ leaves. In turn, this is a simple consequence of the fact that $\bx\^0$ has a single leaf, and the tree $\bx\^t$ is formed by adding at most $k_t+1$ leaves to some of the leaves of the tree $\bx\^{t-1}$. 
\end{proof}

\subsection{Building block: sparse selection protocol}
\label{sec:sparse-selection}
We use the following primitive for solving the {\it sparse selection} problem from~\cite{ghazi_differentially_2020}:
\begin{defn}[Sparse selection]
  For $m, \ell \in \BN$, in {\it $(m,\ell)$-sparse selection} problem, there is some (possibly infinite) universe $\MU$, and $m$ users. Each user $i \in [m]$ is given some set $\MS_i \subset \MU$ of size $| \MS_i | \leq \ell$. An algorithm solves the $(m,\ell)$-sparse selection problem with with additive error $\eta$ if it outputs some universe element $\hat u \in \MU$ such that
  \begin{equation}
    \label{eq:sparse-sel-err}
| \{ i : \hat u \in \MS_i \}| \geq \max_{u \in \MU} | \{ i : u \in \MS_i \} | - \eta.
\end{equation}
\end{defn}

Proposition~\ref{prop:sparse-selection} shows that the sparse selection problem can be solved privately with error independent of the size of the universe $\MU$. It can be thought of as an analogue of the private stable histogram of~\cite[Proposition 2.20]{bun_simultaneous_2016} for the problem of private selection. 
\begin{proposition}[\cite{ghazi_differentially_2020}, Lemma 36]
  \label{prop:sparse-selection}
  For $\ep \in (0,1]$, $\delta \in (0,1)$, $\beta \in (0,1)$, there is an $(\ep,\delta)$-differentially private algorithm that given an input dataset to the $(m,\ell)$-sparse selection problem, outputs a universe element $\hat u$ such that with probability at least $1-\beta$, the error of $\hat u$ is 
  $$
O \left( \frac{1}{\ep} \log \left( \frac{m\ell}{\ep \delta \beta}\right) \right).
  $$
\end{proposition}

\subsection{Overall algorithm}
In this section we combine the components of Sections~\ref{sec:reducetree} and~\ref{sec:sparse-selection} to prove the following theorem, which gives an improper learner for hypothesis classes with sample complexity polynomial in the Littlestone dimension.
\label{sec:improper-alg}
\begin{theorem}
  \label{thm:poly-pri-learn}
  Let $\MF$ be a concept class of domain $\MX$ with $d_{\VV} := \vc(\MF), d_{\LL} := \Ldim(\MF)$. For any $\ep, \delta, \eta \in (0,1)$, for some
  $$
n = O \left( \frac{d_{\LL}^5 d_{\VV} \log^2 \left( \frac{d_{\LL}}{\ep \delta \eta \beta}\right)}{\ep\eta^2} \right)
$$
the algorithm \PolyPriLearn (Algorithm~\ref{alg:poly-pri-learn})  
takes as input $n$ i.i.d.~samples from any realizable distribution $\di$, is $(\ep, \delta)$-\DP, and produces a hypothesis $\hat f$ so that $\err{\di}{\hat f} \leq \eta$ with probability at least $1-\beta$. 

Moreover, under the same $(1-\beta)$-probability event, $\hat f = \soaf{\MG}$ for some $\MG \subset \MF$ for which $\MG$ is $\left\lceil\frac{64 C_0 d_{\LL}}{\eta^2} \right\rceil$-irreducible.
\end{theorem}
\begin{remark}
The assertion that $\hat f = \soaf{\MG}$ for some $\MG$ which is $ \left\lceil \frac{64 C_0 d_{\LL}}{\eta^2} \right\rceil$-irreducible is for use in Section~\ref{sec:proper-learning-finite} when we use \PolyPriLearn as a component of a proper private learning algorithm.
\end{remark}

\PolyPriLearn (Algorithm~\ref{alg:poly-pri-learn}) operates as follows. For sufficiently large positive integers $m, n_0$, \PolyPriLearn runs \ReduceTree on $m$ independent samples of size $n_0$ from the distribution $\di$. Each run of \ReduceTree outputs some set $\hat\MS$ of classifiers in $\{-1,1\}^\MX$. \PolyPriLearn then uses the sparse selection protocol of Proposition~\ref{prop:sparse-selection} to choose some classifier that lies in many of the sets $\hat \MS$. 

\begin{algorithm}[!htp]
  \caption{\PolyPriLearn}\label{alg:poly-pri-learn}
  \KwIn{Parameters $\ep, \delta, \eta, \beta \in (0,1)$, i.i.d. samples $(x,y) \in \MX \times \{-1,1\}$ from a realizable distribution $\di$.}
  \begin{enumerate}[leftmargin=14pt,rightmargin=20pt,itemsep=1pt,topsep=1.5pt]
 \item Set $m \gets \frac{C(d_{\LL}^3 \log(1/(\ep\delta\beta\eta)))}{\ep}$, $n_0 \gets \frac{Cd_{\LL}^2 d_{\VV} \log \left( \frac{d_{\LL} m}{\eta \beta}\right)}{\eta^2}$, $n \gets n_0m$, $\alpha_\Delta \gets 6\cdot \alpha(n_0, \beta/(2m))$, where $C > 0$ is a sufficiently large constant. \label{it:overall-params-set}
 
 Also set $k' \gets \max \{\lceil n_0 \cdot (d_{\LL} + 3) \alpha_\Delta \rceil, \left\lceil \frac{64 C_0 d_{\LL}}{\eta^2} \right\rceil\}$, where $C_0$ is the constant of Theorem~\ref{thm:unif-conv}.
  \item For $1 \leq j \leq m$:\label{it:call-reducetree}
    \begin{itemize}
    \item Run the algorithm \ReduceTree with $n = n_0, \gamma = \beta / (2m)$, and the parameters $\alpha_\Delta, k'$ set in step~\ref{it:overall-params-set} (i.e., with a fresh i.i.d.~sample from $\di$). Let its output set $\hat \MS$ (defined in (\ref{eq:redtree-output})) be denoted by $\hat \MS\^j$. 
    \end{itemize}
  \item Run the sparse selection protocol of Proposition~\ref{prop:sparse-selection} on the sets $\hat \MS^{(1)}, \ldots, \hat \MS^{(m)}$, and output the function $\hat f : \MX \ra \{-1,1\}$ that it outputs.
  \end{enumerate}
\end{algorithm}

\begin{proof}[Proof of Theorem~\ref{thm:poly-pri-learn}]
In the proof we will often refer to the parameters $n_0, m, \alpha_\Delta, k'$, which are set in step~\ref{it:overall-params-set} of \PolyPriLearn (Algorithm~\ref{alg:poly-pri-learn}). 
  Notice that by our choice of
  $$
n_0 =  \frac{Cd_{\LL}^2 d_{\VV} \log \left( \frac{d_{\LL} m}{\eta \beta}\right)}{\eta^2},
  $$
as long as $C$ is sufficiently large, 
  we have that $\alpha_\Delta := 6 \cdot \alpha(n_0, \beta/(2m))$ satisfies $(d_{\LL}+3)\cdot \alpha_\Delta < \eta$. Recall the definition of $\alpha_t := (d_{\LL}+3-t) \cdot \alpha_\Delta$ for $1 \leq t \leq d_{\LL}+1$ from \ReduceTree.

  For $1 \leq j \leq m$, let $T\^j := \{ (x\^j_1, y\^j_1), \ldots, (x\^j_{n_0}, y\^j_{n_0}) \}$ be the dataset of size $n_0$ drawn (i.i.d.~from $\di$) in the $j$th iteration of Step~\ref{it:call-reducetree} of \PolyPriLearn. Let $\hat P\^j := \frac{1}{n_0} \sum_{i=1}^{n_0} \delta_{(x\^j_i, y\^j_i)}$ be the empirical measure over $T\^j$. 
  
  We say that a class $\MG \subset \MF$ is a {\it finite restriction subclass (of $\MF$)} if we can write $\MG = \MF|_{(x_1, y_1), \ldots, (x_M, y_M)}$ for some $(x_1, y_1), \ldots, (x_M, y_M) \in \MX \times \{-1,1\}$. Note that the set of all finite restriction subclasses of $\MF$ is countable by our assumption that $\MX$ is countable. It follows that the set of all finite unions of finite restriction subclasses of $\MF$ is also countable. 
  Now define
  $$
\tilde \MF = \MF \cup \{ \soaf{\MG} :  \substack{\text{$\MG \subset \MF$, $\MG$ is nonempty, $(d_{\LL}+1)$-irreducible, } \\ \text{and a finite union of finite restriction subclasses of $\MF$}}\}.
$$
Notice that the set $\hat \MS$ output by \ReduceTree consists entirely of functions in $\tilde \MF$. (This follows since the set $\hat \MS$ consists of hypotheses of the form $\soaf{\gRes{\alpha}{v}}$, where $\gRes{\alpha}{v}$ is $k'$-irreducible: we then use that $d_{\LL} + 1 \leq k'$ and that for any $\alpha \in [0,1]$, and any dataset $S_n$, $\MF_{\hat \di_{S_{n}}, \alpha}$ is the union of at most $2^{n}$ finite restriction subclasses of $\MF$.)
Moreover, $\tilde \MF$ is countable
, and Lemma~\ref{lem:gen-soa} gives that $\vc(\tilde \MF) \leq \Ldim(\tilde \MF) \leq d_{\LL}$. 
Then Theorem~\ref{thm:unif-conv} gives that
\begin{equation}
  \label{eq:e0-prob}
\p\left[ \forall j \in [m] : \sup_{\tilde f \in \tilde \MF} \left| \err{\di}{\tilde f} - \err{\hat P\^j}{\tilde f} \right| \leq \frac{\alpha_\Delta}{6}  \right] \geq 1-(\beta /(2m)) \cdot m = 1-\beta/2.
\end{equation}
Let $E_0$ be the event inside the probability above, namely that for all $j \in [m]$, $\sup_{\tilde f \in \tilde \MF} \left| \err{\di}{\tilde f} - \err{\hat P\^j}{\tilde f} \right| \leq \frac{\alpha_\Delta}{6}$. Since $\tilde \MF \supseteq \MF$, $E_0$ contains the event that $E_{\good}$ simultaneously holds for each dataset $T\^1, \ldots, T\^m$ (recall that $E_{\good}$ was defined for any dataset $T\^j$ in (\ref{eq:egood})).

The bulk of the proof of Theorem~\ref{thm:poly-pri-learn} is to show the following two claims:
  \begin{claim}
    \label{cl:output-star}
Suppose $m > \frac{C d_{\LL}^3 \log \left( \frac{1}{\ep \delta \beta \eta} \right)}{\ep}$ for a sufficiently large constant $C > 0$. There is an event $E_1$ that occurs with probability at least $1-\beta/2$ (over the randomness of the dataset and the algorithm), so that under $E_1 \cap E_0$, \PolyPriLearn outputs a hypothesis $\soaf{\MG}$, for some $\MG \subset \MF$ so that $\MG$ is $k'$-irreducible. Moreover, this hypothesis belongs to $\hat \MS\^j$ for some $j \in [m]$. 
\end{claim}

\begin{claim}
  \label{cl:empirical-error}
Suppose $k'\geq \lceil n_0 \cdot (d_{\LL}+2) \cdot \alpha_\Delta \rceil$. Under the event $E_1 \cap E_0$, the output of \PolyPriLearn has empirical error at most $(d_{\LL} + 2) \cdot \alpha_\Delta$  on at least one of the $m$ datasets $T\^j$ drawn in Step 2 of \PolyPriLearn. 
\end{claim}

Assuming Claims~\ref{cl:output-star} and~\ref{cl:empirical-error}, we complete the proof of Theorem~\ref{thm:poly-pri-learn}. Notice that the assumptions of Claims~\ref{cl:output-star} and~\ref{cl:empirical-error} hold by our choices of $m, k'$ in Step~\ref{it:overall-params-set} of \PolyPriLearn. 
Denote the output of \PolyPriLearn by $\hat f : \MX \ra \{-1,1\}$. By Claim~\ref{cl:empirical-error}, we have that $\err{\hat P\^j}{\hat f} \leq (d_{\LL} + 2) \cdot \alpha_\Delta$ for some $j \in [m]$. By Claim~\ref{cl:output-star}  and the definition of the sets $\hat \MS\^j$ in (\ref{eq:redtree-output}), we have that under the event $E_1 \cap E_0$, $\hat f \in \tilde \MF$; moreover, $\hat f = \soaf{\MG}$ for some $\MG \subset \MF$ which is $\left\lceil \frac{64 C_0 d_{\LL}}{\eta^2} \right\rceil$-irreducible, by the choice of $k'$ in step~\ref{it:overall-params-set} of Algorithm~\ref{alg:poly-pri-learn}. 
By the definition of $E_0$, it follows that under the event $E_0 \cap E_1$, since $\hat f \in \tilde \MF$, we have $\err{\di}{\hat f} \leq (d_{\LL} + 2) \cdot \alpha_\Delta + \alpha_\Delta/6 \leq (d_{\LL} + 3) \cdot \alpha_\Delta \leq \eta$. By (\ref{eq:e0-prob}) and a union bound, $\p[E_0 \cap E_1] \geq 1-\beta$, so $\p[\err{\di}{\hat f} \leq \eta] \geq 1-\beta$, as desired.

That \PolyPriLearn is $(\ep,\delta)$-differentially private follows as an immediate consequence of Proposition~\ref{prop:sparse-selection} and the fact that each data point lies in exactly one $T\^j$. Summarizing, the sample complexity of \PolyPriLearn is
$$
n_0 \cdot m \leq O \left( \frac{d_{\LL}^5 d_{\VV} \log^2 \left( \frac{d_{\LL}}{\ep \delta \eta \beta}\right)}{\ep\eta^2} \right)
$$

Finally we prove Claims~\ref{cl:output-star} and~\ref{cl:empirical-error}.
\begin{proof}[Proof of Claim~\ref{cl:output-star}]
  Notice that for each $j \in [m]$, each element of $\hat \MS\^j$ is of the form $\soaf{\MG}$ for some $\MG \subset \MF$ which is $k'$-irreducible, and thus $(d_{\LL}+1)$-irreducible (as $k' \geq d_{\LL}+1$). It therefore suffices to show that under the event $E_0 \cap E_1$ (for an appropriate choice of $E_1$), \PolyPriLearn outputs some element of some $\hat \MS\^j$, $j \in [m]$.

    For $\alpha \in [0,1], t \in [d_{\LL} + 1]$, recall the definition $\MM_{\alpha,t}$ in (\ref{eq:big-max}), and for those $\alpha,t$ for which $\MM_{\alpha,t}$ is nonempty, the definition of $\sigma_{\alpha,t}^\st$ in (\ref{eq:sigma-at}). 
    By the definition of $\hat \MS\^j$ (see (\ref{eq:redtree-output})) and Lemma~\ref{lem:equal-to-a-leaf}, under the event $E_0$ each $\hat \MS\^j$ contains at least one of $\sigma_{\alpha_t - \alpha_\Delta/2,t}^\st$ for some $t \in [d_{\LL}+1]$ (which is well-defined). 
    By the pigeonhole principle, it follows that some $\sigma_{\alpha_t - \alpha_\Delta/2,t}^\st$ lies in at least $\lceil m/(d_{\LL}+1) \rceil$ sets $\hat \MS\^j$. 

  By Lemma~\ref{lem:s-ub}, we have that
  $$
| \hat \MS\^j | \leq \prod_{t=1}^{d_{\LL}} (k_t+1) = \prod_{t=1}^{d_{\LL}} k' \cdot 2^t = (k')^{d_{\LL}} \cdot 2^{(d_{\LL}+1)d_{\LL}/2} \leq 2^{d_{\LL}^2 + d_{\LL} \log k'}.
$$
Now choose $\nu > 0$ so that the $(m, 2^{d_{\LL}^2 + d_{\LL} \log k'})$-sparse selection protocol of Proposition~\ref{prop:sparse-selection} (with universe $\MU = \tilde \MF$), has error at most $\nu$ on some event $E_1$ with probability at least $1-\beta/2$. By Proposition~\ref{prop:sparse-selection}, we may choose $\nu = \frac{C}{\ep} \log\left( \frac{m 2^{d_{\LL}^2 + d_{\LL} \log k'}}{\ep \delta \beta} \right)$ for a sufficiently large constant $C$.

Summarizing, under the event $E_0 \cap E_1$, as long as $\nu < \lceil m / (d_{\LL} + 1) \rceil$, the hypothesis $\hat f $ output by the sparse selection protocol belongs to some set $\hat \MS\^j$. Since 
$$
k' \leq 200 C_0 \max \left\{ n_0 (d_{\LL} + 3) \alpha_\Delta, \frac{d_{\LL}}{\eta^2} \right\} \leq 200 C_0 \max \left\{ \frac{d_{\LL}}{\eta^2}, n_0 \eta \right\} \leq \frac{200 C_0 d_{\LL}^2 d_{\VV} \log \left( \frac{d_{\LL}{m}}{\eta \beta} \right)}{\eta^2}
$$
to ensure $\nu < \lceil m / (d_{\LL} + 1) \rceil$ it suffices to have
$$
m > \frac{C' (d_{\LL}+1)}{\ep} \left( \log(m) + d_{\LL}^2 + \log \left( \frac{1}{\ep \delta \beta} \right) + d_{\LL} \left( \log(1/\eta) + \log \log \left( \frac{m}{ \beta} \right) \right)\right),
$$
for which it in turn suffices that
$$
m \geq \frac{C'' d_{\LL}^3 \log \left( \frac{1}{\ep \delta \beta \eta} \right)}{\ep} 
$$
for sufficiently large constants $C', C''$.
\end{proof}

\begin{proof}[Proof of Claim~\ref{cl:empirical-error}]
  By Claim~\ref{cl:output-star}, it suffices to show that under the event $E_1 \cap E_0$, each element of $\hat \MS\^j$ has empirical error at most $(d_{\LL} + 2) \cdot \alpha_\Delta$ on the dataset $T\^j$. By definition, each element of $\hat \MS\^j$ is of the form $\soaf{\MF_{\hat \di\^j, \alpha_t - 2\alpha_\Delta/3}|_{\ba(v)}}$ for some node $v$ of the tree $\hat \bx$ output by \ReduceTree for which $\MF_{\hat \di\^j, \alpha_t - 2\alpha_\Delta/3|_{\ba(v)}}$ is nonempty and $k'$-irreducible (see (\ref{eq:redtree-output})). Fix any such element, and write $\hat{\MH} := \MF_{\hat \di\^j, \alpha_t - 2\alpha_\Delta/3}|_{\ba(v)}$. By definition we have that each $f \in \hat{\MH} = \MF_{\hat \di\^j, \alpha_t - 2\alpha_\Delta/3}|_{\ba(v)} \subset \MF_{\hat \di\^j, \alpha_t - 2\alpha_\Delta/3}$ satisfies 
  \begin{equation}
  \label{eq:err-pj}
  \err{\hat \di\^j}{f} \leq \alpha_t - 2\alpha_\Delta/3 \leq \alpha_1 -2\alpha_\Delta/3= (d_{\LL} + 2) \cdot \alpha_\Delta - 2\alpha_\Delta/3.
  \end{equation}

Let $\ell = \lceil n_0 \alpha_\Delta \cdot (d_{\LL} + 2) \rceil$. Suppose for the purpose of contradiction that $$\err{\hat \di\^j}{\soaf{\hat{\MH}}} \geq \alpha_\Delta \cdot (d_{\LL} + 2).$$ Let $i_1, \ldots, i_{\ell} \in [n_0]$ be indices on which $\soaf{\hat{\MH}}$ is incorrect; i.e., for $t \in [\ell]$, we have $\soa{\hat{\MH}}{x\^j_{i_t}} = -y\^j_{i_t}$, i.e., $\Ldim(\hat{\MH}|_{(x\^j_{i_t}, -y\^j_{i_t})}) = \Ldim(\hat{\MH})$. Since $\hat{\MH}$ is $k'$-irreducible and $k' \geq \ell$, it follows that
  $$
\Ldim(\hat{\MH}|_{(x\^j_{i_1}, -y\^j_{i_1}), \ldots, (x\^j_{i_{\ell}}, -y\^j_{i_{\ell}})}) = \Ldim(\hat{\MH}),
$$
and in particular since $\hat{\MH}$ is nonempty there is some $f \in \hat{\MH}$ so that for $t \in [\ell]$, $f(x\^j_{i_t}) = -y\^j_{i_t}$, i.e., $\err{\hat \di\^j}{f} \geq \ell/n_0 > \alpha_\Delta \cdot (d_{\LL} + 2) - 2\alpha_\Delta/3$. This is a contradiction to (\ref{eq:err-pj}).
\end{proof}
\end{proof}

\section{Proper private learner for Littlestone classes}
\label{sec:proper-learning-finite}
In this section we show how to use the improper private learner of Theorem~\ref{thm:poly-pri-learn} to obtain a proper one, thus proving Theorem~\ref{thm:poly-pri-learn-proper} (the formal version of Theorem~\ref{thm:pap-pac-informal}). For simplicity we assume in this section that $\MX, \MF$ are finite. The case in which they are allowed to be infinite is treated in Appendix~\ref{sec:infinite-spaces}. Let $\Delta(\MX), \Delta(\MF)$ be the spaces of probability distributions over $\MX, \MF$, respectively.

\begin{lemma}
  \label{lem:game-value}
Let $C_0$ be the constant of Theorem~\ref{thm:unif-conv}. Fix any $\alpha \in (0,1)$ and $\MG \subset \MF$ which is $\left\lceil \frac{C_0d}{\alpha^2} \right\rceil$-irreducible and suppose $\vc(\MF) \leq d$ for some $d \in \BN$. 
  Then it holds that
  $$
\sup_{P \in \Delta(\MX)} \inf_{D \in \Delta(\MF)} \E_{x \sim P, h \sim D} \left[ \One[\soa{\MG}{x} \neq h(x)] \right] \leq \alpha.
  $$
\end{lemma}
\begin{proof}
  It suffices to show that for any $P \in \Delta(\MX)$, there is some $g \in \MF$ so that $\E_{x \sim P} \left[ \One[\soa{\MG}{x} \neq g(x)] \right] \leq \alpha$. By Theorem~\ref{thm:unif-conv}, for $n = \left\lceil\frac{C_0d}{\alpha^2}\right\rceil$, then with probability at least $1/2$ over a sample $x_1, \ldots, x_n \sim P$, we have
  \begin{equation}
    \label{eq:g-use-uc}
\sup_{f \in \MF} \left| \E_{x \sim P} \left[ \One[\soa{\MG}{x} \neq f(x)] \right] - \frac 1n \sum_{i=1}^n \One[\soa{\MG}{x_i} \neq f(x_i) ] \right| \leq \alpha.
\end{equation}
Fix any sample $x_1, \ldots, x_n$ for which (\ref{eq:g-use-uc}) holds. Since $\MG$ is $n$-irreducible, there exists some $g \in \MG$ so that $g(x_i) = \soa{\MG}{x_i}$ for each $i \in [n]$. Then by (\ref{eq:g-use-uc}), we have that $\E_{x \sim P} \left[ \One[\soa{\MG}{x} \neq g(x) ]\right] \leq \alpha$. 
\end{proof}

\begin{lemma}
  \label{lem:like-soa}
 Let $C_0$ be the constant of Theorem~\ref{thm:unif-conv}. Fix any $\alpha \in (0,1)$ and $\MG \subset \MF$ which is $\left\lceil \frac{C_0d }{\alpha^2} \right\rceil$-irreducible and suppose $\vc(\MF) \leq d$ and $\vc^\st(\MF) \leq d^\st$ for some $d, d^\st \in \BN$. Then there is a set $\MH \subset \MF$, depending only on the function $\soaf{\MG}: \MX \ra \{-1,1\}$, and of size $|\MH| \leq \left\lceil \frac{C_0d^\st}{\alpha^2} \right\rceil$, so that for any distribution $P \in \Delta(\MX)$, it holds that
  \begin{equation}
    \label{eq:minh-ub}
\min_{h \in \MH} \E_{x \sim P} \left[ \One[h(x) \neq \soa{\MG}{x}] \right]  \leq 2\alpha.
  \end{equation}
\end{lemma}
\begin{proof}
  By Lemma~\ref{lem:game-value} and von Neumann's minimax theorem, it holds that
  \begin{align}
    \label{eq:infd}
    & \inf_{D \in \Delta(\MF)} \sup_{P \in \Delta(\MX)}  \E_{x \sim P, h \sim D} \left[ \One[\soa{\MG}{x} \neq h(x)] \right] \\
=&  \sup_{P \in \Delta(\MX)} \inf_{D \in \Delta(\MF)} \E_{x \sim P, h \sim D} \left[ \One[\soa{\MG}{x} \neq h(x)] \right] \leq \alpha.\nonumber
  \end{align}
  Fix some $D \in \Delta(\MF)$ obtaining the infimum in (\ref{eq:infd}); this is possible because $\Delta(\MF)$ is compact. Note that $D$ depends only on $\soaf{\MG} \in \{0,1\}^\MX$ (i.e., it can be written as a function of $\soaf{\MG}$). By Theorem~\ref{thm:unif-conv}, for $m = \left\lceil \frac{C_0 d^\st}{\alpha^2} \right\rceil$, with probability at least $1/2$ over an i.i.d.~sample $h_1, \ldots, h_m \sim D$, we have that\footnote{As remarked by~\cite{moran_sample_2016}, this application of Theorem~\ref{thm:unif-conv} to the dual class can be viewed as a sort of combinatorial and approximate version of Carath\'{e}odory's theorem.}
  \begin{align}
    \label{eq:h-conc}
    & \sup_{x \in \MX} \left| \E_{h \sim D} \left[ \One[\soa{\MG}{x} \neq h(x)] \right] - \frac{1}{m} \sum_{j=1}^m \One[\soa{\MG}{x} \neq h_j(x)] \right| \\
    =&  \sup_{x \in \MX} \left| \E_{h \sim D} \left[\frac{h(x)}{2}\right] - \frac{1}{m} \sum_{j=1}^m \frac{h_j(x)}{2} \right| \leq \alpha\\
    =& \sup_{x \in \MX} \left| \E_{h \sim D} \left[ \One[h(x) \neq 1] \right] - \frac{1}{m} \sum_{j=1}^m \One[h_j(x) \neq 1] \right| \leq \alpha.\nonumber
\end{align}
(To see why the equalities above hold, note that for any $h \in \MF$, if $\soa{\MG}{x} = 1$, then $\One[\soa{\MG}{x} \neq h(x)] = \frac{1-h(x)}{2}$, and if $\soa{\MG}{x} = -1$, then $\One[\soa{\MG}{x} \neq h(x)] = \frac{1+h(x)}{2}$.)
Fix any $h_1, \ldots, h_m$ so that (\ref{eq:h-conc}) holds, and set $\MH := \{ h_1, \ldots, h_m \}$. Write $h \sim_U \MH$ to mean that $h$ is drawn uniformly from $\MH$. Then by (\ref{eq:infd}) and (\ref{eq:h-conc}), we have, for any $P \in \Delta(\MX)$,
\begin{align}
  & \E_{x \sim P, h \sim_U \MH} \left[ \One[\soa{\MG}{x} \neq h(x)] \right] \nonumber\\
  \leq^{(\ref{eq:h-conc})}& \E_{x \sim P, h \sim D} \left[ \One[\soa{\MG}{x} \neq h(x)] \right] + \alpha \nonumber\\
  \leq^{(\ref{eq:infd})} & 2\alpha.\nonumber
\end{align}
(\ref{eq:minh-ub}) is an immediate consequence of the above display.
\end{proof}

\subsection{Private proper learning protocol}
\label{sec:proper-protocol}
Before introducing our private proper learning algorithm, we need the following basic lemma which establishes that the use of the exponential mechanism can output a good hypothesis privately from a class of small size:
\begin{lemma}[Generic Private Learner,~\cite{kasiviswanathan2008what}]
  \label{lem:gen-learner}
  Let $\MH \subset \{-1,1\}^\MX$ be a finite set of hypotheses. For
  $$
n = O \left( \frac{\log|\MH| + \log 1/\beta}{\alpha \ep} \right),
$$
there exists an $(\ep, 0)$-differentially private algorithm $\GenericLearner: (\MX \times \{-1,1\})^n \ra \MH$ so that the following holds. For any distribution $P$ over $\MX \times \{-1,1\}$ so that there exists $h^\st \in \MH$ with
$$
\err{P}{h^\st} \leq \alpha,
$$
on input $S_n := \{(x_1,y_1), \ldots, (x_n,y_n) \} \sim P^n$, \GenericLearner outputs, with probability at least $1-\beta$, a hypothesis $\hat h \in \MH$ so that
$$
\err{P}{\hat h} \leq 2\alpha.
$$
\end{lemma}
The precise formulation of Lemma~\ref{lem:gen-learner} is proved in~\cite[Lemma 16]{bun_equivalence_2020}.

Our algorithm, \PPPLearn, for privately and properly learning a hypothesis class, is presented in Algorithm~\ref{alg:poly-pri-prop-learn}. Given sufficiently many samples from a realizable distribution $\di$, \PPPLearn first runs \PolyPriLearn to come up with a hypothesis of the form $\soaf{\MG} \in \{-1,1\}^\MX$ with low population loss on the distribution $\di$. It then uses the guarantee of Lemma~\ref{lem:like-soa} to come up with a small subclass $\MH \subset \MF$ which is guaranteed to contain a hypothesis that performs nearly as well as $\soaf{\MG}$ on the distribution $\di$. It then privately chooses such a hypothesis $\hat h \in \MH$ using the exponential mechanism (Lemma~\ref{lem:gen-learner}).

\begin{algorithm}[!htp]
  \caption{\PPPLearn}\label{alg:poly-pri-prop-learn}
  \KwIn{Parameters $\ep, \delta, \eta, \beta \in (0,1)$, i.i.d. samples $(x,y) \in \MX \times \{-1,1\}$ from a realizable distribution $\di$.}
  \begin{enumerate}[leftmargin=14pt,rightmargin=20pt,itemsep=1pt,topsep=1.5pt]
  \item Run the algorithm \PolyPriLearn with parameters $\ep, \delta, \eta/4, \beta/2$. 
  Let $\hat f \in \{-1,1\}^\MX$ be its output.\label{it:run-polyprilearn}

    With probability at least $1-\eta/2$, it is then guaranteed that $\hat f = \soaf{\MG}$ for some $k'$-irreducible $\MG \subset \MF$. Choose any such $\MG$.

  \item Choose a set $\MH$ as in Lemma~\ref{lem:like-soa} from the function $\soaf{\MG}$.\label{it:choose-h}
  \item On a fresh sample of size $O \left( \frac{\log| \MH| + \log(1/\beta)}{\eta \ep} \right)$, run the \GenericLearner of Lemma~\ref{lem:gen-learner} with the set $\MH$, and return its output $\hat h$.
  \end{enumerate}

\end{algorithm}

\begin{theorem}[Private proper PAC learning]
  \label{thm:poly-pri-learn-proper}
  Let $\MF$ be a concept class of domain $\MX$ with $d_{\VV} := \vc(\MF), d_{\LL} := \Ldim(\MF)$. For any $\ep, \delta, \eta, \beta \in (0,1)$, for some
  \begin{equation}
    \label{eq:proper-sc}
n = O \left( \frac{d_{\LL}^5 d_{\VV} \log^2 \left( \frac{d_{\LL}}{\ep \delta \eta \beta}\right)}{\ep\eta^2} \right),
\end{equation}
there is an $(\ep, \delta)$-differentially private algorithm $A : (\MX \times \{-1,1\})^n \ra \{-1,1\}^\MX$, which, given $n$ i.i.d.~samples from any realizable distribution $\di$, products a hypothesis $\hat f$ so that $\err{\di}{\hat f} \leq \eta$ with probability at least $1-\beta$.
\end{theorem}
\begin{proof}
  We let the algorithm $A$ be \PPPLearn (Algorithm~\ref{alg:poly-pri-prop-learn}). To establish accuracy, note that the output $\hat f = \soaf{\MG}$ of \PolyPriLearn computed in Step~\ref{it:run-polyprilearn} of Algorithm~\ref{alg:poly-pri-prop-learn} satisfies $\err{P}{\hat f} \leq \eta/4$ with probability at least $1-\beta/2$ over the algorithm and the samples, by Theorem~\ref{thm:poly-pri-learn}. Then using $\alpha = \eta/8$ in Lemma~\ref{lem:like-soa}, we get that for the set $\MH$ produced in Step~\ref{it:choose-h} of Algorithm~\ref{alg:poly-pri-prop-learn}, there is some $h^\st \in \MH$ so that $\err{P}{h^\st} \leq \eta/2$. Then by Lemma~\ref{lem:gen-learner}, the output $\hat h$ of \PPPLearn satisfies $\err{P}{\hat h} \leq \eta$ with probability at least $1-\beta$.

 The  $(\ep, \delta)$-differential privacy of the output $\hat h$ of \PPPLearn follows from the $(\ep, \delta)$-differential privacy of the output $\hat f = \soaf{\MG}$ of \PolyPriLearn with respect to its input samples, the post-processing property of differential privacy, and the $(\ep, 0)$-differential privacy of \GenericLearner with respect to its input samples. (Note that \PolyPriLearn and \GenericLearner are run on different samples.)

  Finally, to see that the claimed upper bound on sample complexity holds, it suffices to upper bound the number of samples used by \GenericLearner by the quantity in (\ref{eq:proper-sc}). This follows since by Lemma~\ref{lem:like-soa}, we have
  $$
\log |\MH| \leq \log \left(O \left( \frac{\vc^\st(\MF) }{\eta^2} \right) \right) \leq O \left( \vc(\MF) + \log 1/\eta \right).
  $$
  (Here we use that for any hypothesis class $\MF$, $\vc^\st(\MF) \leq 2^{\vc(\MF) + 1}$~\cite{assouad_densite_1983}.)
\end{proof}

As a corollary of Theorem~\ref{thm:poly-pri-learn-proper} and~\cite[Theorem 4.16]{beimel_learning_2015} (or~\cite[Theorem 2.4]{alon_closure_2020}, which is a more general result) we get a sample complexity bound for agnostic private proper PAC learning:
\begin{corollary}[Agnostic private proper PAC learning]
\label{cor:agnostic-proper}
  Let $\MF$ be a concept class of domain $\MX$ with $d_{\VV} := \vc(\MF), d_{\LL} := \Ldim(\MF)$. For any $\ep, \delta, \eta, \beta \in (0,1)$, for some
  \begin{equation*}
    n = O \left( \frac{d_{\LL}^5 d_{\VV} \log^2 \left( \frac{d_{\LL}}{\ep \delta \eta \beta}\right)}{\ep\eta^2} \right),
\end{equation*}
there is an $(\ep, \delta)$-differentially private algorithm $A : (\MX \times \{-1,1\})^n \ra \MF$, which, given $n$ i.i.d.~samples from any distribution $\di$, produces a hypothesis $\hat f \in \MF$ so that $\err{\di}{\hat f} \leq \eta + \inf_{f \in \MF} \err{\di}{f}$ with probability at least $1-\beta$. 
\end{corollary}

\subsection{Application to private data sanitization}
\label{sec:sanitization}
In this section we show how to prove Corollaries~\ref{cor:sanitizing-informal} and~\ref{cor:sanitizing-2} using a result of~\cite{bousquet_passing_2019} that shows how to convert a private proper agnostic PAC learner into a sanitizer for a binary hypothesis class. We say that an algorithm $A$ is an {\it $(\alpha, \beta)$-accurate} proper agnostic PAC learner for a class $\MF$ with {\it sample complexity $n$} if for any distribution $P$ over $\MX \times \{-1,1\}$, when given as input $n$ i.i.d.~samples from $P$, the algorithm $A$ produces as output a function $\hat f \in \MF$ so that with probability at least $1-\beta$ over the sample and the randomness in $A$, we have $\err{P}{\hat f} \leq \alpha + \inf_{f \in \MF} \err{P}{f}$. 

\begin{restatable}[Slight strengthening of \cite{bousquet_passing_2019}, Propositions 1 \& 2]{theorem}{BLMSanitizer}
\label{thm:bousquet-sanitizer}
Suppose $\MF \subset \{-1,1\}^\MX$ is a class of VC dimension $d_{\VV}$ and dual Littlestone dimension $d_{\LL}^\st$. Moreover suppose that for any $\alpha', \beta', \ep', \delta' \in (0,1)$, there is some $n_0(\alpha', \beta', \ep', \delta') \in \BN$ so that $\MF$ has a proper PAC learner with sample complexity $n_0(\alpha', \beta', \ep',\delta')$ that is $(\ep', \delta')$-\DP and $(\alpha', \beta')$-accurate. Then there is a (sufficiently large) constant $C > 0$ 
so that for any $\alpha, \beta, \ep, \delta \in (0,1)$, as long as $n \in \BN$ is chosen to satisfy
\begin{equation}
  \label{eq:n-sanitizer}
      n \geq \frac{C}{\ep} \cdot \left( \left(n_0(\alpha/8, \tau_0 \beta/2, 1, \delta) + \frac{\log\left( \frac{d_{\LL}^\st}{\beta \alpha}\right)}{\alpha} \right)\cdot \left(\frac{d_{\LL}^\st \log(d_{\LL}^\st / \alpha)\log(1/\delta)}{\alpha^2}\right)^{1/2} 
    \right),
\end{equation}
where $\tau_0 = \frac{\alpha^2}{C d_{\LL}^\st \log(d_{\LL}^\st/\alpha) \log(1/\delta)}$, 
$\MF$ has a $\left(n, \alpha, \beta, 1, \delta \cdot \frac{\sqrt{C d_{\LL}^\st \log(d_{\LL}^\st/\alpha)}}{\alpha}\right)$-sanitizer.
\end{restatable}
We explain how to derive Theorem \ref{thm:bousquet-sanitizer} using the proof technique in \cite[Propositions 1 \& 2]{bousquet_passing_2019} in Section \ref{sec:sanitizer-quant}. 

As an immediate corollary of Theorem~\ref{thm:bousquet-sanitizer} and Corollary~\ref{cor:agnostic-proper} we obtain the following:
\begin{corollary}[Private sanitization; formal version of Corollary~\ref{cor:sanitizing-informal}]
\label{cor:sanitizing-formal}
Let $\MF$ be a hypothesis class with VC dimension $d_{\VV}$, Littlestone dimension $d_{\LL}$, and dual Littlestone dimension $d_{\LL}^\st$. For any $\alpha, \beta, \ep, \delta \in (0,1)$, for any $n \in \BN$ satisfying
\begin{equation}
  \label{eq:sanitization-sc}
 n \geq  C \cdot \frac{d_{\LL}^5 d_{\VV} \sqrt{d_{\LL}^\st} \log^2 \left( \frac{d_{\LL} d_{\LL}^\st}{\delta \alpha \beta}  \right)\log\left( \frac{d_{\LL}^\st}{ \alpha\delta}\right)}{\alpha^3\ep},
\end{equation}
$\MF$ has a $(n, \alpha, \beta, \ep, \delta)$-sanitizer.
\end{corollary}
We remark that the dependence of (\ref{eq:sanitization-sc}) on $d_{\LL}^\st$, namely $\tilde O(\sqrt{d_{\LL}^\st})$, is tight up to polylogarithmic factors in the sense that for all $d_{\LL}^\st$, there is a class $\MF$ with $\max\{ \vc(\MF), \Ldim(\MF) \} \leq O(\log d_{\LL}^\st)$ and $\Ldim^\st(\MF) = d_{\LL}^\st$, yet the sample complexity of sanitization for $\MF$ is $\tilde \Omega(\sqrt{d_{\LL}^\st})$, by Theorem \ref{thm:bun-fingerprinting} below.

\begin{proof}[Proof of Corollary \ref{cor:sanitizing-formal}]
  By Corollary~\ref{cor:agnostic-proper}, the following holds, for a sufficiently large constant $C > 0$: for any $\alpha', \beta', \ep', \delta' \in (0,1)$, for any $n_0 \geq \frac{C d_{\LL}^5 d_{\VV} \log^2 \left( \frac{d_{\LL}}{\ep' \delta' \alpha' \beta'}\right)}{\ep' (\alpha')^2}$, $\MF$ has a proper agnostic PAC learner with sample complexity $n_0$ that is $(\ep', \delta')$-differentially private and $(\alpha', \beta')$-accurate. We first show that for any $\alpha, \beta, \delta \in (0,1)$, $\MF$ has a $(n, \alpha, \beta, 1, \delta)$-sanitizer for an appropriate value of $n$. To do this, we apply Theorem~\ref{thm:bousquet-sanitizer}. To ensure that the number of samples $n$ is at least the quantity in (\ref{eq:n-sanitizer}), it suffices to have at least
  \begin{align*}
    & C \cdot \left( \frac{d_{\LL}^5 d_{\VV} \log^2 \left( \frac{d_{\LL} d_{\LL}^\st}{\delta \alpha \beta} \right)}{\alpha^2} + \frac{\log \left( \frac{d_{\LL}^\st}{\beta \alpha}\right)}{\alpha} \right) \cdot \frac{\sqrt{d_{\LL}^\st \log(d_{\LL}^\st / \alpha) \log(d_{\LL}^\st/(\alpha\delta))}}{\alpha} \\
    \leq & C \cdot \frac{d_{\LL}^5 d_{\VV} \sqrt{d_{\LL}^\st} \log^2 \left( \frac{d_{\LL} d_{\LL}^\st}{\delta \alpha \beta}  \right)\log\left( \frac{d_{\LL}^\st}{\alpha \delta}\right)}{\alpha^3}
  \end{align*}
  samples, where $C$ is a sufficiently large constant.

  The existence of a $(n, \alpha, \beta, \ep, \delta)$-sanitizer for $\MF$ for any $\alpha, \beta, \ep, \delta \in (0,1)$ and $n$ satisfying (\ref{eq:sanitization-sc}) now follows from Theorem~\ref{thm:unif-conv} and a standard privacy amplification by subsampling argument \cite[Lemma 4.12]{bun_differentially_2015}:\footnote{Similar arguments have been used in, e.g., \cite[Lemma 2.2]{bassily_erm}, \cite{beimel_bounds_2014}, \cite{bun_equivalence_2020}.} in particular, 
  by increasing the number of samples $n$ by a factor of $O(1/\ep)$ and sampling an $O(\ep)$ fraction of the samples, we can convert a $(O(1), \delta)$-differentially private algorithm into an $(\ep, \delta)$-differentially private algorithm. The accuracy loss due to this subsampling can be bounded by a small constant times $\alpha$, by Theorem \ref{thm:unif-conv} and the fact that the number of samples $n$ in (\ref{eq:sanitization-sc}) must be 
  at least $\Omega \left( \frac{\vc(\MF) + \log 1/\beta}{\alpha^2}\right)$. 
\end{proof}


Finally, we may prove Corollary~\ref{cor:sanitizing-2}:
\begin{proof}[Proof of Corollary~\ref{cor:sanitizing-2}]
The fact that finite Littlestone dimension of a class $\MF$ implies sanitizability follows from the fact that for all binary hypothesis classes $\MF$, $\vc(\MF) \leq \Ldim(\MF)$, $\Ldim^\st(\MF) \leq 2^{2^{\Ldim(\MF) + 2}}$~\cite[Lemma 4]{bousquet_passing_2019}, and Corollary~\ref{cor:sanitizing-formal}. For the opposite direction, we use the fact that for any $\MF$, the {threshold dimension} of $\MF$\footnote{The {\it threshold dimension} of $\MF \subset \{-1,1\}^\MX$ is the largest positive integer $T$ so that there are $x_1, \ldots, x_T \in \MX$ and $f_1, \ldots, f_T \in \MF$ so that for $1 \leq i, j \leq t$, $f_i(x_j) = \begin{cases} 1 : i \geq j \\ 0 : i < j \end{cases}$.}, denoted $\Tdim(\MF)$, satisfies $\Tdim(\MF) \geq \lfloor \log \Ldim(\MF)\rfloor$~\cite[Theorem 3]{alon_private_2019}. Thus any $(n, \alpha, \beta, \ep, \delta)$-sanitizer for a class $\MF$ yields a $(n, \alpha, \beta, \ep,\delta)$-sanitizer for the class of thresholds on a linearly ordered domain of size $T$ for any $T \leq \lfloor \log \Ldim(\MF) \rfloor$. But~\cite[Theorem 1]{bun_differentially_2015} yields that any $(n, 1/10, 1/10, 1/10, 1/(50n^2))$-sanitizer for the class of thresholds on a linearly ordered comain of size $T$ must satisfy $n \geq \Omega(\log^\st T)$. Thus, if $\Ldim(\MF)$ is infinite, then there is no $(n, 1/10, 1/10, 1/10, 1/(50n^2))$-sanitizer for the class $\MF$, and so $\MF$ is not sanitizable.   
\end{proof}

\paragraph{Lower bounds}
We end this section by discussing how the sample complexity bound of Corollary~\ref{cor:sanitizing-formal} compares to existing lower bounds for sanitization. 
First, we remark that it follows from fingerprinting-based lower bounds~\cite{bun_fingerprinting_2014} that in general the sample complexity of a sanitizer for a class $\MF$ must grow at least polynomially in the the dual Littlestone dimension of $\MF$:
\begin{theorem}[\cite{bun_fingerprinting_2014}, Theorem 5.8]
\label{thm:bun-fingerprinting}
For any constant $\ell \in \BN$, the following holds for all $d,t \in \BN$ so that $\ell + 2 \leq t \leq d/2$. For $\MX = \{-1,1\}^d$, there is a class $\MF \subset \{-1,1\}^\MX$ so that:
\begin{itemize}
\item $\Ldim(\MF) = \Theta(t \log (d/t))$; 
\item $\Ldim^\st(\MF) = \Theta(d)$, 
\end{itemize}
and so that for all $\ep \in (0,1)$ and $\alpha \geq \tilde \Omega \left( \frac{ d^{-\ell/3 + 1/4}}{\sqrt \ep}\right)$, any $(n, \alpha, 1/100, \ep, 1/(10n))$-sanitizer for $\MF$ must have 
$$
n \geq \tilde \Omega \left( \frac{t \sqrt{d}}{\ep \alpha^2} \right).
$$
\end{theorem}
For any fixed $\ell$, the $\tilde \Omega(\cdot)$ in Theorem~\ref{thm:bun-fingerprinting} hides factors which are inverse polynomial in $\log t, \log d, \log \frac{1}{\ep}, \log \frac{1}{\alpha}$. We also remark that the VC and dual VC dimensions are within constant factors of the Littlestone and dual Littlestone dimensions of the class $\MF$ of Theorem~\ref{thm:bun-fingerprinting} (this will be clear from the proof below). Since~\cite{bun_fingerprinting_2014} does not explicitly compute the Littlestone and dual Littlestone dimensions of the class $\MF$, we give a short proof that the entirety of the claim in Theorem~\ref{thm:bun-fingerprinting} holds, using~\cite[Theorem 5.8]{bun_fingerprinting_2014}:
\begin{proof}[Proof of Theorem~\ref{thm:bun-fingerprinting} using~\cite{bun_fingerprinting_2014}]
We take the class $\MF$ to be the class of $t$-wise conjunctions on $\MX = \{-1,1\}^d$, i.e., the class of all ANDs of $t$ literals on $\{-1,1\}^d$ (for concreteness, view 1 as ``False'' and $-1$ as ``True''). From~\cite[Lemma 6]{littlestone_learning_1987} we have that $\Ldim(\MF) \geq \vc(\MF) \geq \Omega(t \log(d/t))$; also $\Ldim(\MF) \leq O(t \log (d/t))$ since the size of $\MF$ is bounded above by $2^{O(t \log(d/t))}$. For the dual quantity, it is clear that $\Ldim^\st(\MF) \leq d$ since the size of the dual class is $2^d$. Moreover, $\Ldim^\st(\MF) \geq \vc^\st(\MF) \geq d/2$ since the class of $d/2$ functions $x \mapsto x_1 \wedge \cdots \wedge x_{t-1} \wedge x_j$, for $d/2 \leq j \leq d$, is shattered by the dual class $\MX$. Finally,~\cite[Theorem 5.8]{bun_fingerprinting_2014} gives us the fact that there is no $(\ep, 1/(10n))$-\DP algorithm which takes as input a dataset $S$ of size $n$ and outputs some function $\Est : \MF \ra [0,1]$ satisfying $|\Est(f) - \err{S}{f}| \leq \alpha$ for all $f \in \MF$ with probability at least $2/3$.
\end{proof}
By choosing $\ell = 1$, and arbitrary positive integers $t,d$ tending to $\infty$ and satisfying $t \leq d/2$, Theorem~\ref{thm:bun-fingerprinting} rules out a sample complexity bound for sanitization that depends polynomially on only the Littlestone dimension of $\MF$ (such as one in Theorem~\ref{thm:poly-pri-learn-proper} for proper private learning). 
Because of the requirement that $t \leq d/2$ in Theorem~\ref{thm:bun-fingerprinting}, it does not rule out a sample complexity bound that depends polynomially on only the dual Littlestone dimension (and only sub-polynomially on the Littlestone dimension). This latter possibility is ruled out by discrepancy-based lower bounds:
\begin{theorem}[\cite{nikolov_geometry_2012}]
\label{thm:vc-lb}
For any binary hypothesis class $\MF$, any $\alpha < 1/50$ and any $\ep \in (0,1)$, any $(n, \alpha, 1/100, \ep, 0.1)$-sanitizer for $\MF$ must have $n \geq \Omega \left( \frac{\vc(\MF)}{\ep \alpha} \right)$. 
\end{theorem}
For a proof of the precise statement of Theorem~\ref{thm:vc-lb}, see Theorem 5.8 and Proposition 5.11 of~\cite{vadhan2017complexity}. Note that for any positive integer $d$, there is a class $\MF$ for which $\Ldim(\MF)= \vc(\MF) = d$ and $\Ldim^\st(\MF) = \Theta(\log d)$ (for instance, we may take the class of all functions on $d$ distinct points). Thus Theorem~\ref{thm:vc-lb} rules out the existence of a sanitizer with sample complexity polynomial in only dual Littlestone dimension. 

Summarizing, from Theorems~\ref{thm:bun-fingerprinting} and~\ref{thm:vc-lb}, we obtain that Corollary~\ref{cor:sanitizing-formal} is ``best possible up to a polynomial'' in the sense that polynomial dependence on both $d_{\LL}$ {\it and} $d_{\LL}^\st$ is necessary in a worst-case sense. Moreover, when $d_{\LL}, d_{\LL}^\st$ are of the same order, then any sample complexity upper bound must be superlinear $\max\{ d_{\LL}, d_{\LL}^\st\}$ (Theorem~\ref{thm:bun-fingerprinting}). Finally, in light of Theorem \ref{thm:bun-fingerprinting}, the square-root dependence on $d_{\LL}^\st$ (up to polylogarithmic factors) in Corollary \ref{cor:sanitizing-formal} is best possible up to polylogarithmic factors.

\section{Conclusions}
\label{sec:conclusion}
In this paper we showed that it is possible to privately and properly learn binary hypothesis classes of Littlestone dimension $d$ with sample complexity polynomial in $d$. As a corollary we showed that such classes have sanitizers with sample complexity polynomial in $d$ and the dual Littlestone dimension $d^\st$. 
A central open question remaining (see, e.g.,~\cite[Section 1.6]{beimel_characterizing_2019}) is to determine a characterization of the sample complexity of (proper and improper) PAC learning with approximate differential privacy, up to (ideally) a constant factor, much like the VC dimension provides such a characterization for (non-private) PAC learning~\cite{vapnik_statistical_1998}, the Littlestone dimension provides such a characterization for online learning~\cite{littlestone_learning_1987,ben-david_agnostic_2009}, and the probabilistic representation dimension~\cite{beimel_characterizing_2019} and the one-way public coin communication complexity~\cite{feldman_sample_2014} both provide such a characterization for improper PAC learning with pure differential privacy. As noted by~\cite{alon_closure_2020}, current lower bounds even allow for the possibility that the sample complexity of (proper or improper) PAC learning with approximate differential privacy is linear in $\vc(\MF) + \log^\st (\Ldim(\MF))$. Below we list some intermediate questions which may be useful in attacking this question and the related question of characterizing the sample complexity of sanitization. (Throughout by ``private'' we mean $(\ep, \delta)$-differentially private with $\delta$ negligible in the number of users $n$.)
\begin{enumerate}
\item {\bf Sample complexity linear in Littlestone dimension.} The most immediate open question is to reduce the exponent of $d$ from the current value of 6 in Theorem~\ref{thm:pap-pac-informal}. In particular, one could hope for sample complexity that scales linearly with the Littlestone dimension $d$ (see the discussion following Theorem~\ref{thm:pap-pac-informal}). 
\item \label{it:learn-poly-char} {\bf Polynomial characterization of private learnability.} One could also attempt to show bounds with sublinear dependence on the Littlestone dimension, as long as there is at least linear dependence on the VC dimension. Rather optimistically, we ask: is the sample complexity of (properly or improperly) PAC learning a class $\MF$ with $(\ep, \delta)$-differential privacy at most $n = \poly \left( \vc(\MF), \log^\st(\Ldim(\MF))\right)$? (Here we omit dependence on $1/\alpha, 1/\ep, \log 1/\delta$, for which the dependence should be polynomial as well.) 
In light of the lower bound of $\Omega(\vc(\MF) + \log^\st(\Ldim(\MF)))$ by Alon et al.~\cite{alon_private_2019} on the sample complexity, this would give a characterization for the sample complexity of private PAC learning up to a polynomial factor. 
\item {\bf Proper vs.~improper learning.} Is there a family of hypothesis classes for which the sample complexity of proper private learning is asymptotically larger than the sample complexity of improper private learning? The answer to this question is ``yes'' for the case of pure privacy (e.g., exhibited by the class of point functions~\cite{beimel_bounds_2014}), but it remains open for approximate privacy to the best of our knowledge.
\item \label{it:san-direct-proof} {\bf Direct proof of Corollary~\ref{cor:sanitizing-formal}.} The current proof of Corollary~\ref{cor:sanitizing-formal} is quite long: it consists of first proving the existence of an improper private learner (Theorem~\ref{thm:poly-pri-learn}), then showing how to make it proper (Corollary~\ref{cor:agnostic-proper}), and finally applying Theorem~\ref{thm:bousquet-sanitizer} of Bousquet et al.~\cite{bousquet_passing_2019}, which itself has two fairly involved parts, the first of which shows that $\MF$ is ``Sequentially-Foolable''~\cite[Theorem 2]{bousquet_passing_2019}, and the second of which shows that $\MF$ is sanitizable~\cite[Theorem 1]{bousquet_passing_2019}. It would be interesting to find a more direct proof of Corollary~\ref{cor:sanitizing-formal}, namely one that does not ``go through'' a proper learner.
\item {\bf Improved bounds for sanitization.} Finally, it would be interesting to improve quantitatively upon the upper bound for sanitization of Corollary~\ref{cor:sanitizing-formal}. In particular, analogously to item~\ref{it:learn-poly-char}, it is natural to ask: is the sample complexity of sanitizating a class $\MF$ (with approximate privacy) at most $n = \poly \left( \vc(\MF), \vc^\st(\MF), \log^\st(\Ldim(\MF))\right)$? By \cite[Corollary 3.6]{bun_fingerprinting_2014}, Theorem~\ref{thm:vc-lb}, and~\cite[Theorems 3.2 \& 4.6]{bun_differentially_2015}, the sample complexity of sanitization is at least $\tilde \Omega \left( \vc(\MF) + \sqrt{\vc^\st(\MF)} + \log^\st(\Ldim(\MF))\right)$,\footnote{\cite[Corollary 3.6]{bun_fingerprinting_2014} gives a lower bound of $\tilde \Omega(\sqrt{d})$ on the sample complexity of private release of 1-way marginals on $\{-1,1\}^d$; the $\tilde \Omega(\vc^\st(\MF))$ lower bound on the sample complexity of sanitization in any class $\MF$ follows since a class of 1-way marginals on a copy of $\{-1,1\}^{\vc^\st(\MF)}$ may be embedded in any class $\MF$. Similarly, \cite{bun_differentially_2015} gives a $\Omega(\log^\st |\MX|)$ lower bound on the sample complexity of release of threshold functions on a domain $\MX$; the $\Omega(\log^\st |\MX|)$ lower bound on the sample complexity of sanitization in any class $\MF$ follows since $\log \Ldim(\MF)$ thresholds may be embedded in $\MF$.} so this would provide a characterization for the sample complexity of sanitization up to a polynomial factor. Since our approach of using the results of~\cite{bousquet_passing_2019} seems to necessarily incur at least a {\it square-root} dependence on the dual Littlestone dimension $\Ldim^\st(\MF)$, any positive answer to this question would likely involve a positive answer to the question in item~\ref{it:san-direct-proof}.
\end{enumerate}

\appendix

\section{Private proper learner for infinite $\MF$ and $\MX$}
\label{sec:infinite-spaces}
In this section we extend the arguments from Section~\ref{sec:proper-learning-finite} to cover the case where $\MX, \MF$ are allowed to be countably infinite. The techniques closely follow those in~\cite{bousquet_passing_2019}. 
\subsection{Preliminaries}
\label{sec:topology}
\paragraph{Product topology}
Let $\MV$ be an arbitrary set, and let $\{-1,1\}$ have the discrete topology. The {\it product topology} on the space $\{-1,1\}^\MV$ of functions $f : \MV \ra \{-1,1\}$ is defined to be the coarsest topology so that the functions $\pi_v : \{-1,1\}^\MV \ra \{-1,1\}$, defined by $\pi_v(f) := f(v)$ are all continuous. It is known that this topology is Hausdorff. The following fact is an immediate consequence of Tychanoff's theorem:
\begin{theorem}[Tychanoff's theorem; e.g.,~\cite{munkres_topology_2000}, Chapter 5, Theorem 1.1]
  \label{thm:tychanoff}
The space $\{-1,1\}^\MV$ is compact (under the product topology). 
\end{theorem}

\paragraph{Compactness}
Let $\MW$ be a compact Hausdorff topological space. Let $C(\MW)$ denote the space of real-valued continuous functions on $\MW$. Let $\rca(\MW)$ denote the space of Borel measures on $\MW$, and let $\Delta(\MW)$ denote the space of Borel {\it probability} measures on $\MW$ (a measure $\mu$ on $\MW$ is a probability measure if for all measurable subsets $A \subset \MW$, $\mu(A) \in [0,1]$, and $\mu(\MW) = 1$). The weak* topology on $\rca(\MW)$ (also known as the {\it vague} topology) is defined to be the coarsest topology so that all of the mappings $\mu \mapsto \int_{f \in \MW} \omega(f) d\mu(f)$, where $\omega \in C(\MW)$, are continuous. The following lemma is a consequence of the Banach-Alaoglu theorem (see, e.g.,~\cite[Theorem IV.21]{reed_functional_1981}) and the Riesz-Markov theorem which states that the dual space of the Banach space $C(\MW)$ is the space $\rca(\MW)$ of Borel measures on $\MW$ (see, e.g.,~\cite[Theorem IV.14]{reed_functional_1981}, and also~\cite[Claim 2]{bousquet_passing_2019}):
\begin{lemma}
  \label{lem:deltaf-compact}
The space $\Delta(\MW)$ is compact in the weak* topology.
\end{lemma}

\paragraph{Spaces of distributions} Next, recall that $\Sigma$ is a $\sigma$-algebra on the data space $\MX$. We consider the product topology on the space $\{-1,1\}^\MX$; by Tychonoff's theorem (Theorem~\ref{thm:tychanoff}), $\{-1,1\}^\MX$ is compact (and Hausdorff). 
Let $\MF \subset \{-1,1\}^\MX$ have the subspace topology, so that $\MF$ is also compact. By Lemma~\ref{lem:deltaf-compact}, $\Delta(\MF)$ is compact in the weak* topology.

Following~\cite{bousquet_passing_2019}, let $\BR_\fin^\MX$ to be the space of real-valued functions $p : \MX \ra \BR$ so that there are only finitely many $x \in \MX$ so that $p(x) \neq 0$. Give $\BR_\fin^\MX$ the topology induced by the $\ell_1$ norm; more formally, a basis of open sets is given by the balls $\MB_{q,a}$, for $q \in \BR_\fin^\MX$, $a > 0$, where:
$$
\MB_{q,a} := \left\{ p \in \BR_\fin^\MX : \sum_{x \in \MX} | p(x) - q(x) | < a \right\}.
$$

Let $\Delta_\fin(\MX)$ be the subspace of $\BR_\fin^\MX$ consisting of functions $p$ so that for all $x \in \MX$, $p(x) \geq 0$ and $\sum_{x \in \MX} p(x) =1$. 
We will often identify $\Delta_\fin(\MX)$ with the space of probability measures on $\MX$ with {\it finite support}. In particular, for some $p : \MX \ra \BR$, the corresponding measure $P$ is the one defined by, for $A \in \Sigma$, 
$$
P(A) = \sum_{x \in \MX} p(x) \cdot \delta_{x}(A) = \sum_{x \in \MX} p(x) \cdot \One[x \in A].
$$

\paragraph{Semi-continuity, Sion's minimax theorem}
Let $\MW$ be a topological space. A function $f : \MW \ra \BR$ is {\it upper semi-continuous (u.s.c)} if for every $r \in \BR$, the set $\{ w : f(w) \geq r \}$ is closed. Similarly, $f$ is {\it lower semi-continuous (l.s.c)} if for every $r \in \BR$, the set $\{ w : f(w) \leq r \}$ is closed. We will use the following fact:
\begin{lemma}[\cite{bousquet_passing_2019}, Claim 3]
  \label{lem:measure-usc}
Let $\MW$ be a compact hausdorff space, and let $\MK \subset \MW$ be a closed subset. Consider the mapping $T_\MK : \Delta(\MW) \ra [0,1]$, defined by $T_\MK(\mu) := \mu(\MK)$. Then $T_\MK$ is u.s.c.~with respect to the weak* topology on $\Delta(\MW)$. 
\end{lemma}

Sion's minimax theorem, stated below, is a generalization of the von Neumann minimax theorem.
\begin{theorem}[\cite{sion_general_1958}]
  \label{thm:sion}
  Let $\MW$ be a compact and convex subset of a topological vector space and $\MU$ be a convex subset of a topological vector space. Suppose $F : \MW \times \MU \ra \BR$ is a real-valued function so that:
  \begin{itemize}
  \item For all $u \in \MU$, the function $w \mapsto F(w,u)$ is l.s.c.~and convex on $\MW$. 
  \item For all $w \in \MW$, the function $u \mapsto F(w,u)$ is u.s.c.~and concave on $\MU$.
  \end{itemize}
  Then
  $$
\inf_{w \in \MW} \sup_{u \in \MU} F(w,u) = \sup_{u \in \MU} \inf_{w \in \MW} F(w,u).
  $$
  
\end{theorem}

\subsection{Modifications to the finite case}
In this section we detail the modifications that it is necessary to make to the proofs in Section~\ref{sec:proper-learning-finite} to establish Theorem~\ref{thm:poly-pri-learn-proper} (and thus Corollary~\ref{cor:agnostic-proper}) for the case that $\MX, \MF$ are countably infinite.

We begin with Lemma~\ref{lem:game-value}; notice that nowhere in the proof of Lemma~\ref{lem:game-value} do we use that $\MX, \MF$ are finite; i.e., it holds if $\MX, \MF$ are allowed to be infinite. 
Corollary~\ref{cor:game-value-infinite} is then an immediate corollary of Lemma~\ref{lem:game-value} (with infinite $\MX, \MF$), since $\Delta_\fin(\MX) \subset \Delta(\MX)$.  
\begin{corollary}
  \label{cor:game-value-infinite}
There is a constant $C > 0$ so that the following holds. Fix any $\alpha \in (0,1)$ and $\MG \subset \MF$ which is $\left\lceil \frac{C( d +  \log1/\alpha)}{\alpha^2} \right\rceil$-irreducible and suppose $\vc(\MF) \leq d$ for some $d \in \BN$. 
  Then it holds that
  \begin{equation}
    \label{eq:sup-p}
\sup_{P \in \Delta_\fin(\MX)} \inf_{D \in \Delta(\MF)} \E_{x \sim P, h \sim D} \left[ \One[\soa{\MG}{x} \neq h(x)] \right] \leq \alpha.
  \end{equation}
\end{corollary}

Lemma~\ref{lem:like-soa-infinite} is a generalization of Lemma~\ref{lem:like-soa} to the case that $\MX, \MF$ are infinite; the main technical portion of the proof departing from that of Lemma~\ref{lem:like-soa} is the verification that the preconditions of Sion's minimax theorem hold.
\begin{lemma}
  \label{lem:like-soa-infinite}
  There is a constant $C > 0$ so that the following holds. Fix any $\alpha \in (0,1)$ and $\MG \subset \MF$ which is $\left\lceil \frac{C(d+ \log1/\alpha)}{\alpha^2} \right\rceil$-irreducible and suppose $\vc(\MF) \leq d$ and $\vc^\st(\MF) \leq d^\st$ for some $d, d^\st \in\BN$. Then there is a set $\MH \subset \MF$, depending only on the function $\soaf{\MG}: \MX \ra \{-1,1\}$, and of size $|\MH| \leq \left\lceil \frac{C(d^\st+ \log 1/\alpha)}{\alpha^2} \right\rceil$, so that for any distribution $P \in \Delta(\MX)$, it holds that
  \begin{equation}
    \label{eq:minh-ub-infinite}
\min_{h \in \MH} \E_{X \sim P} \left[ \One[h(X) \neq \soa{\MG}{X}] \right]  \leq 3\alpha.
  \end{equation}
\end{lemma}
\begin{proof}
We will first use Sion's minimax theorem (Theorem~\ref{thm:sion}) to argue that
  \begin{align}
    \label{eq:infd-infinite}
    & \inf_{D \in \Delta(\MF)} \sup_{P \in \Delta_{\fin}(\MX)}  \E_{x \sim P, h \sim D} \left[ \One[\soa{\MG}{x} \neq h(x)] \right] \\
=&  \sup_{P \in \Delta_{\fin}(\MX)} \inf_{D \in \Delta(\MF)} \E_{x \sim P, h \sim D} \left[ \One[\soa{\MG}{x} \neq h(x)] \right] \leq \alpha.\label{eq:infd-infinite-2}
  \end{align}
  (Notice that the inequality in (\ref{eq:infd-infinite-2}) is from Corollary~\ref{cor:game-value-infinite}; below we argue that the equality in the above display.) 
  In particular, we will have $\MW = \Delta(\MF), \MU = \Delta_{\fin}(\MX)$, and for $D \in \MW, P \in \MU$,
  $$
F(D, P) := \E_{x \sim P, h \sim D} \left[ \One[\soa{\MG}{x} \neq h(x)] \right].
$$
Notice that $\MW, \MU$ are subsets of the topological vector spaces $\rca(\MF), \BR_\fin^\MX$, respectively. Moreover, it is immediate that both $\MW, \MU$ are convex, and by Theorem~\ref{thm:tychanoff} and Lemma~\ref{lem:deltaf-compact} we have that $\MW$ is compact. To check the u.s.c.~and l.s.c.~preconditions of Theorem~\ref{thm:sion}, we argue as follows:
\begin{itemize}
\item Fix any $D \in \MW$. Notice that the function $P \mapsto F(D, P)$ may be written as
  $$
F(D,P) = \sum_{x \in \MX} P(x) \cdot \E_{h \sim D} \left[ \One[\soa{\MG}{x} \neq h(x) ] \right].
$$
It is evident that $P \mapsto F(D,P)$ is a linear function, hence convex. Moreover, it is continuous (hence l.s.c.) since for each $x \in \MX$, $| \E_{h \sim D}\left[ \One[\soa{\MG}{x} \neq h(x)]\right]| \leq 1$ and since the topology on $\MW$ is induced by the $\ell_1$ norm on $\BR_\fin^\MX$.
\item Fix any $P \in \MU$. It is evident that $D \mapsto F(D,P)$ is a linear function, hence concave. Note that for any $x \in \MX$, by definition of the product topology, the map from $\rca(\MF)$ to $\BR$ that sends $h \mapsto \One[\soa{\MG}{x} \neq h(x)]$ is continuous. Thus $\{ h \in \MF : \soa{\MG}{x} \neq h(x) \}$ is a closed subset of $\MF$. 
By Lemma~\ref{lem:measure-usc}, the mapping $D \mapsto \E_{h \sim D} \left[ \One[\soa{\MG}{x} \neq h(x)] \right]$ is u.s.c.~with respect to the weak* topology on $\Delta(\MF)$. That $D \mapsto F(D,P)$ is u.s.c.~follows since a finite sum of u.s.c.~functions is u.s.c.
\end{itemize}
We have verified that all of the conditions of Theorem~\ref{thm:sion} hold, and thus we may conclude that the equality (\ref{eq:infd-infinite}) holds.

Fix any $P \in \Delta(\MX)$. 
Since $\vc(\MF)$ is finite, by Theorem~\ref{thm:unif-conv}, there is some $P' \in \Delta_\fin(\MX)$ so that
\begin{equation}
  \label{eq:delta-fin-x}
\sup_{h \in \MF}  \left| \E_{x \sim P} \left[ \One[\soa{\MG}{x} \neq h(x)] \right] - \E_{x \sim P'} \left[ \One[\soa{\MG}{x} \neq h(x)] \right] \right| \leq \alpha.
\end{equation}
  
Fix an arbitrary $\alpha' > 0$.  Fix some $D \in \Delta(\MF)$ obtaining a value of at most $\alpha + \alpha'$ in (\ref{eq:infd-infinite}). 
Note that $D$ depends only on $\soaf{\MG} \in \{0,1\}^\MX$, i.e., it can be written as a function of $\soaf{\MG}$. 
By Theorem~\ref{thm:unif-conv}, for a sufficiently large $C > 0$, for $m = \left\lceil \frac{C(d^\st +\log1/\alpha)}{\alpha^2} \right\rceil$, then with probability at least $1/2$ over an i.i.d.~sample $h_1, \ldots, h_m \sim D$, we have that
  \begin{align}
    \label{eq:h-conc-infinite}
    & \sup_{x \in \MX} \left| \E_{h \sim D} \left[ \One[\soa{\MG}{x} \neq h(x)] \right] - \frac{1}{m} \sum_{j=1}^m \One[\soa{\MG}{x} \neq h_j(x)] \right| \\
    =& \sup_{x \in \MX} \left| \E_{h \sim D} \left[\One[h(x) \neq 1] \right] - \frac{1}{m} \sum_{j=1}^m\One[h_j(x) \neq 1] \right| \leq \alpha.\nonumber
\end{align}
Fix any $h_1, \ldots, h_m$ so that (\ref{eq:h-conc}) holds, and set $\MH := \{ h_1, \ldots, h_m \}$. Write $h \sim_U \MH$ to mean that $h$ is drawn uniformly from $\MH$. Then by (\ref{eq:infd-infinite}), (\ref{eq:delta-fin-x}), and (\ref{eq:h-conc-infinite}), we have, for the given $P \in \Delta(\MX)$,
\begin{align}
  & \E_{x \sim P, h \sim_U \MH} \left[ \One[\soa{\MG}{x} \neq h(x)] \right] \nonumber\\
  \leq^{(\ref{eq:delta-fin-x})} & \E_{x \sim P', h \sim_U \MH} \left[ \One[\soa{\MG}{x} \neq h(x)] \right] + \alpha \nonumber\\
  \leq^{(\ref{eq:h-conc-infinite})}& \E_{x \sim P, h \sim D} \left[ \One[\soa{\MG}{x} \neq h(x)] \right] + 2 \alpha \nonumber\\
  \leq^{(\ref{eq:infd-infinite}) \& (\ref{eq:infd-infinite-2})} & 3\alpha + \alpha'.\nonumber
\end{align}
Since $\alpha' > 0$ is arbitrary, (\ref{eq:minh-ub-infinite}) is an immediate consequence of the above.
\end{proof}
No modifications to the algorithm \PPPLearn (Algorithm~\ref{alg:poly-pri-prop-learn}) are necessary to deal with the case of infinite $\MX, \MF$, except the reference to Lemma~\ref{lem:like-soa} on Step~\ref{it:choose-h} should instead be to Lemma~\ref{lem:like-soa-infinite}. 
That Theorem~\ref{thm:poly-pri-learn-proper} holds for the case that $\MX, \MF$ are countably infinite follows without any modifications to its proof. 

\section{On the Littlestone dimension of SOA classes}
\label{sec:ldim-soa}
Given a class $\MF$, the algorithm \PolyPriLearn (Algorithm~\ref{alg:poly-pri-learn}) will output with high probability a hypothesis of the form $\soaf{\MG}$ for some $\MG \subset \MF$. To address the question of whether this hypothesis has small population error, we used Lemma~\ref{lem:gen-soa} to upper bound the Littlestone dimension (and thus VC dimension) of the class of hypotheses $\soaf{\MG}$ for which $\MG \subset \MF$ is $(d+1)$-irreducible, where $d$ is the Littlestone dimension of $\MF$. In this section, we show that one cannot drop the requirement that $\MG$ be $(d+1)$-irreducible; in particular, the VC dimension of
\begin{equation}
\label{eq:tildef-inf}
\tilde \MF := \{ \soaf{\MG} : \MG\subset \MF \}
\end{equation}
can be infinite even if $\vc(\MF)$ is finite. 

Let $\MF^{\negpt}$ be the class of point functions and negated point functions on an infinite set $\MX$. In particular, for $x \in \MX$, write $\delta_x$ to be the point function for $x$, defined by $\delta_x(y) = 1$ if $y = x$ and else $\delta_x(y) = -1$. Then:
$$
\MF^{\negpt} := \{ \delta_x : x \in \MX \} \cup \{ - \delta_x : x \in \MX\}.
$$
It is straightforward to check that $\Ldim(\MF^{\negpt}) = \vc(\MF^{\negpt}) = 3$. However, the Littlestone dimension of the class $\tilde\MF^{\negpt}$ defined in (\ref{eq:tildef-inf}) is infinite, as shown in the following proposition:
\begin{proposition}
It holds that $\vc(\tilde \MF^{\negpt}) = \Ldim(\tilde \MF^{\negpt}) = \infty$.
\end{proposition}
\begin{proof}
  We show that for any $d \in \BN$, $d \geq 2$, and distinct points $x_1, \ldots, x_d \in \MX$, there is some $h \in \tilde \MF^{\negpt}$ so that $h(x_1) = \cdots = h(x_d) = 1$ and $h(x) = -1$ for all $x \not \in \{x_1, \ldots, x_d \}$.

  To do so, fix $d \geq 2$ and the points $x_1, \ldots, x_d$. Define
  $$
\MG := \{ f \in \MF^{\negpt} : \exists j \in [d] \text{ s.t. } f(x_j) = 1 \}.
$$
We first show that for all $x \not \in \{x_1, \ldots, x_d\}$, it holds that $\soa{\MG}{x} = -1$. This in turn follows from the following two facts:
\begin{itemize}
\item $\Ldim(\MG|_{(x,1)}) = 1$. To see this, note first that any $f \in \MG|_{(x,1)}$ must be of the form $f(y) = -\delta_z(y)$ for $z,y \in \MX$. The class of such $f$ has Littlestone dimension at most 1. The Littlestone dimension is exactly 1 since $-\delta_{x_1}, -\delta_{x_2} \in \MG|_{(x,1)}$. 
\item $\Ldim(\MG|_{(x,-1)}) = 2$. To see that the Littlestone dimension is at most 2, note that $\Ldim(\MF^{\negpt}|_{(x,-1)}) =2$ and $\MG \subset \MF^{\negpt}$. To see that the Littlestone dimension is at least 2, consider the tree $\bx$ of depth 2 defined by
  $$
\bx_1 = x_1, \qquad  \bx_2(-1) = \bx_2(1) = x_2.
$$
This tree is shattered by $\MG|_{(x,-1)}$ since:
\begin{align*}
  \delta_{x_3}(x_1) = \delta_{x_3}(x_1) = -1 \\
  \delta_{x_1}(x_1) = 1, \delta_{x_2}(x_2) = -1 \\
  \delta_{x_2}(x_1) = -1, \delta_{x_2}(x_2) = 1 \\
  -\delta_x(x_1) = -\delta_x(x_2) = 1,
\end{align*}
and $\delta_{x_1}, \delta_{x_2}, \delta_{x_3}, -\delta_x \in \MG|_{(x,-1)}$.
\end{itemize}

We next show that for all $x \in \{x_1, \ldots, x_d \}$, it holds that $\soa{\MG}{x} = 1$. This in turn follows from the following facts. We note that by symmetry we may assume without loss of generality that $x = x_1$:
\begin{itemize}
\item $\Ldim(\MG|_{(x_1,1)}) = 2$. Since $\MG \subset \MF^{\negpt}$, we have that $\MG|_{(x_1,1)} \subset \MF^{\negpt}|_{(x_1,1)}$. On the other hand, any $f \in \MF^{\negpt}$ with $f(x_1) = 1$ necessarily lies in $\MG$, so $\MF^{\negpt}|_{(x_1,1)} \subset \MG|_{(x_1,1)}$. Thus $\MF^{\negpt}|_{(x_1,1)} = \MG|_{(x_1,1)}$, so we have that $\Ldim(\MG|_{(x_1,1)}) = \Ldim(\MF^{\negpt}|_{(x_1,1)}) = 2$.
\item $\Ldim(\MG|_{(x_1,-1)}) = 2$. The Littlestone dimension is at most 2 since $\Ldim(\MF^{\negpt}|_{(x_1,-1)}) = 2$. A similar argument as the one used above to establish that $\Ldim(\MG|_{(x,-1)}) = 2$ may be used to show that here $\Ldim(\MG|_{(x_1,-1)}) = 2$; however, we note that this argument isn't necessary to show that $\soa{\MG}{x_1} = 1$. 
\end{itemize}
We have shown that $\soa{\MG}{x} = 1$ if and only if $x \in \{x_1, \ldots, x_d\}$, which completes the proof of the proposition.
\end{proof}

\section{Proof of Theorem \ref{thm:bousquet-sanitizer}}
\label{sec:sanitizer-quant}
In this section we sketch how the quantitative bound in Theorem \ref{thm:bousquet-sanitizer} (restated below for convenience) may be derived from the argument in \cite{bousquet_passing_2019}.
\BLMSanitizer*
\begin{proof}[Proof of Theorem \ref{thm:bousquet-sanitizer} using \cite{bousquet_passing_2019}]
  We assume familiarity with the notation and terminology of \cite{bousquet_passing_2019}. We first remark that by Proposition 2 of \cite{bousquet_passing_2019}, it suffices to show the existence of a DP-fooling algorithm (for an arbitrary distribution $p_{real}$ over $\MX$) with sample complexity given by (\ref{eq:n-sanitizer}). In turn, the proof of existence of a DP-fooling algorithm is a slight modification of the proof of Proposition 1 of \cite{bousquet_passing_2019} with the parameter $\kappa$ therein taken to be equal to 1; the main difference is the use of the advanced composition lemma for differential privacy as opposed to the basic composition lemma used in \cite{bousquet_passing_2019}.

  For any $\alpha, \beta, \delta \in (0,1)$, let the number of samples in the input dataset $S$ to the DP-fooling algorithm 
  be the quantity on the right-hand side of (\ref{eq:n-sanitizer}). As in \cite{bousquet_passing_2019}, we let $G$ be a generator that fools $\MF$ with round complexity $T(\alpha') \leq O \left( \frac{d_{\LL}^\st \log(d_{\LL}^\st / \alpha')}{(\alpha')^2}\right)$, for any $\alpha' \in (0,1)$. (Such a $G$ is guaranteed by \cite[Theorem 2]{bousquet_passing_2019}.) Let $D$ be the discriminator used in \cite[Figure 2]{bousquet_passing_2019}. 
  We use exactly the same fooling algorithm as in \cite{bousquet_passing_2019}, except with the number of rounds set to $T_0 = T(\alpha/4)$, and set $\tau_0 = 1/\sqrt{T_0\log(1/\delta)}$ (in \cite{bousquet_passing_2019} the settings were instead $T_0 = \min \{ |S|^\kappa, T(\alpha/4) \}$ and $\tau_0 = 1/T_0$). Thus there is some constant $C$ so that $T_0 \leq C \cdot \left(\frac{d_{\LL}^\st \log(d_{\LL}^\st / \alpha)}{\alpha^2} \right)$. 
  To analyze privacy and utility, we use \cite[Lemma 6]{bousquet_passing_2019}, which uses as a black-box a proper PAC learner with sample complexity $n_0(\alpha', \beta', \ep', \delta'$): 
  in particular, in our usage of Lemma 6, the privacy parameters $(\ep', \delta')$ of this PAC learner, which are referred to as $(\alpha(\tau|S|), \beta(\tau|S|))$ in \cite{bousquet_passing_2019}, are $(\ep', \delta') := (1, \delta)$. Moreover, the parameter $\alpha'$ (referred to as $\ep$ in \cite[Lemma 6]{bousquet_passing_2019}) is set to $\alpha' := \alpha/8$, the parameter $\beta'$ (referred to as $\delta$ in \cite{bousquet_passing_2019}) is set to $\beta' := \beta\tau_0/2$.\footnote{The sample complexity bound in \cite[Lemma 6]{bousquet_passing_2019} includes an additional factor of $\tau_0$, explaining the dependence of $\tau_0^2 \beta/2$ in (\ref{eq:n-sanitizer}).} The number of samples in (\ref{eq:n-sanitizer}) satisfies Eq.~(12) of \cite{bousquet_passing_2019}, thus allowing us to apply Lemma 6 therein.

  To analyze privacy of this algorithm, note that \cite[Lemma 6]{bousquet_passing_2019} gives that the discriminator $D$ is $(6\tau_0 \ep' + \tau_0, 4e^{6\tau_0 \ep'} \tau_0 \delta')$-differentially private, where $\ep', \delta'$ are the privacy parameters used in the proper PAC learner for $\MF$ which has sample complexity $n_0(\alpha', \beta', \ep', \delta')$. By the advanced composition lemma (e.g., \cite[Theorem 3.20]{DworkRothBook}), the overall algorithm is
  $$
\left( \sqrt{2T_0 \ln(1/\delta')} \cdot (6 \tau_0 \ep' + \tau_0) + 3T_0 (6 \tau_0 \ep' + \tau_0)^2,\  T_0 \cdot 4 e^{6 \tau_0 \ep'} \tau_0 \delta' + \delta' \right)
$$
differentially private. Using our choice of $\tau_0$, as well as our settings $\ep' = 1, \delta' = \delta$ (recall that $\delta$ is the target approximate privacy parameter), the overall algorithm is $(O(1), O(\sqrt{T_0} ) \cdot \delta)$-differentially private.

To analyze utility of the overall algorithm, the argument is essentially identical as in \cite{bousquet_passing_2019}: the choice of $T_0 = T(\alpha/4)$ and the fact that the generator $G$ has round complexity $T(\cdot)$ implies that if the guarantee of \cite[Lemma 6]{bousquet_passing_2019} holds for all $T_0$ iterations (which is the case with probability at least $1 - T_0 \cdot (\tau_0^2 \beta/2) \geq 1-\beta/2$), then the distribution $p_{syn}$ output by the generator $G$ at its termination satisfies $\sup_{f \in \MF} \left| \E_{x \sim p_{syn}}[f(x)] - \E_{x \sim p_S}[f(x)] \right| \leq \alpha/2$. Moreover, we observe that since the number of samples in (\ref{eq:n-sanitizer}) is at least $\Omega \left( \frac{d_{\VV} + \log(1/\beta)}{\alpha^2} \right)$ , with probability at least $1-\beta/2$, it will also hold that $\sup_{f \in \MF} \left| \E_{x \sim p_S}[f(x)] - \E_{x \sim p_{real}}[f(x)] \right| \leq \alpha/2$. Thus, with probability at least $1-\beta$, we have $\sup_{f \in \MF} \left| \E_{x \sim p_{syn}}[f(x)] - \E_{x \sim p_{real}}[f(x)]\right|$, which establishes the desired DP-foolability property.
\end{proof}

\bibliographystyle{alpha}
\bibliography{privacy}

\newcommand{\etalchar}[1]{$^{#1}$}
\begin{thebibliography}{NBW{\etalchar{+}}18}

\bibitem[ABMS20]{alon_closure_2020}
Noga Alon, Amos Beimel, Shay Moran, and Uri Stemmer.
\newblock Closure properties for private classification and online prediction.
\newblock In {\em COLT}, pages 119--152, 2020.

\bibitem[AJL{\etalchar{+}}19]{abernethy_online_2019}
Jacob~D. Abernethy, Young~Hun Jung, Chansoo Lee, Audra McMillan, and Ambuj
  Tewari.
\newblock Online learning via the differential privacy lens.
\newblock In {\em NeurIPS}, pages 8892--8902, 2019.

\bibitem[ALMM19]{alon_private_2019}
Noga Alon, Roi Livni, Maryanthe Malliaris, and Shay Moran.
\newblock Private {PAC} learning implies finite {Littlestone} dimension.
\newblock In {\em STOC}, page 852–860, 2019.

\bibitem[AS17]{agarwal_price_2017}
Naman Agarwal and Karan Singh.
\newblock The price of differential privacy for online learning.
\newblock In {\em ICML}, page 32–40, 2017.

\bibitem[Ass83]{assouad_densite_1983}
Patrick Assouad.
\newblock Densit\'e et dimension.
\newblock {\em Annales de l'Institut Fourier}, 33(3):233--282, 1983.

\bibitem[BBKN14]{beimel_bounds_2014}
Amos Beimel, Hai Brenner, Shiva~Prasad Kasiviswanathan, and Kobbi Nissim.
\newblock Bounds on the sample complexity for private learning and private data
  release.
\newblock {\em Machine Learning}, 94:401--437, 2014.

\bibitem[BBNS19]{blasiok_towards_2019}
Jaroslaw Blasiok, Mark Bun, Aleksandar Nikolov, and Thomas Steinke.
\newblock Towards instance-optimal private query release.
\newblock In {\em SODA}, page 2480–2497, 2019.

\bibitem[BDKT12]{bhaskara_unconditional_2012}
Aditya Bhaskara, Daniel Dadush, Ravishankar Krishnaswamy, and Kunal Talwar.
\newblock Unconditional differentially private mechanisms for linear queries.
\newblock In {\em STOC}, page 1269–1284, 2012.

\bibitem[BDRS18]{bun_composable_2018}
Mark Bun, Cynthia Dwork, Guy~N. Rothblum, and Thomas Steinke.
\newblock Composable and versatile privacy via truncated {CDP}.
\newblock In {\em STOC}, page 74–86, 2018.

\bibitem[Ben15]{ben-david_notes_2015}
Shai Ben{-}David.
\newblock 2 notes on classes with {Vapnik--Chervonenkis} dimension 1.
\newblock {\em arXiv:1507.05307}, 2015.

\bibitem[Bha17]{bhaskar_thicket_2017}
Siddharth Bhaskar.
\newblock Thicket density.
\newblock {\em arXiv:1702.03956}, 2017.

\bibitem[BLM20a]{bousquet_passing_2019}
Olivier Bousquet, Roi Livni, and Shay Moran.
\newblock Synthetic data generators: Sequential and private.
\newblock In {\em NeurIPS}, 2020.
\newblock \url{https://arxiv.org/abs/1902.03468}, v3.

\bibitem[BLM20b]{bun_equivalence_2020}
Mark Bun, Roi Livni, and Shay Moran.
\newblock An equivalence between private classification and online prediction.
\newblock In {\em FOCS}, 2020.

\bibitem[BLR08]{blum2008learning}
Avrim Blum, Katrina Ligett, and Aaron Roth.
\newblock A learning theory approach to non-interactive database privacy.
\newblock In {\em STOC}, pages 609--618, 2008.

\bibitem[BM03]{bartlett_rademacher_2003}
Peter~L. Bartlett and Shahar Mendelson.
\newblock Rademacher and gaussian complexities: Risk bounds and structural
  results.
\newblock {\em JMLR}, 3:463–482, 2003.

\bibitem[BMNS19]{beimel_private_2019}
Amos Beimel, Shay Moran, Kobbi Nissim, and Uri Stemmer.
\newblock Private center points and learning of halfspaces.
\newblock In {\em COLT}, pages 269--282, 2019.

\bibitem[BNS14]{beimel_private_2014}
Amos Beimel, Kobbi Nissim, and Uri Stemmer.
\newblock Private learning and sanitization: Pure vs. approximate differential
  privacy.
\newblock {\em Theory of Computing}, 12, 07 2014.

\bibitem[BNS15]{beimel_learning_2015}
Amos Beimel, Kobbi Nissim, and Uri Stemmer.
\newblock Learning privately with labeled and unlabeled examples.
\newblock In {\em SODA}, pages 461--477, 2015.

\bibitem[BNS16]{bun_simultaneous_2016}
Mark Bun, Kobbi Nissim, and Uri Stemmer.
\newblock Simultaneous private learning of multiple concepts.
\newblock In {\em ITCS}, page 369–380, 2016.

\bibitem[BNS19]{beimel_characterizing_2019}
Amos Beimel, Kobbi Nissim, and Uri Stemmer.
\newblock Characterizing the sample complexity of pure private learners.
\newblock {\em JMLR}, 20(146):1--33, 2019.

\bibitem[BNSV15]{bun_differentially_2015}
Mark Bun, Kobbi Nissim, Uri Stemmer, and Salil~P. Vadhan.
\newblock Differentially private release and learning of threshold functions.
\newblock In {\em FOCS}, pages 634--649, 2015.

\bibitem[BPS09]{ben-david_agnostic_2009}
Shai Ben{-}David, D{\'{a}}vid P{\'{a}}l, and Shai Shalev{-}Shwartz.
\newblock Agnostic online learning.
\newblock In {\em COLT}, 2009.

\bibitem[BST14]{bassily_erm}
Raef Bassily, Adam~D. Smith, and Abhradeep Thakurta.
\newblock Private empirical risk minimization: Efficient algorithms and tight
  error bounds.
\newblock In {\em FOCS}, pages 464--473, 2014.

\bibitem[Bun20]{bun_computational_2020}
Mark Bun.
\newblock A computational separation between private learning and online
  learning.
\newblock {\em arXiv:2007.05665}, 2020.

\bibitem[BUV14]{bun_fingerprinting_2014}
Mark Bun, Jonathan Ullman, and Salil Vadhan.
\newblock Fingerprinting codes and the price of approximate differential
  privacy.
\newblock In {\em STOC}, page 1–10, 2014.

\bibitem[CN20]{cohen_towards_2020}
Aloni Cohen and Kobi Nissim.
\newblock Towards formalizing the {GDPR}’s notion of singling out.
\newblock {\em PNAS}, 117(15), 2020.

\bibitem[DMNS06]{dwork2006calibrating}
Cynthia Dwork, Frank McSherry, Kobbi Nissim, and Adam Smith.
\newblock Calibrating noise to sensitivity in private data analysis.
\newblock In {\em TCC}, pages 265--284, 2006.

\bibitem[DNR{\etalchar{+}}09]{dwork_complexity_2009}
Cynthia Dwork, Moni Naor, Omer Reingold, Guy~N. Rothblum, and Salil Vadhan.
\newblock On the complexity of differentially private data release: Efficient
  algorithms and hardness results.
\newblock In {\em STOC}, page 381–390, 2009.

\bibitem[DR14]{DworkRothBook}
Cynthia Dwork and Aaron Roth.
\newblock {\em The Algorithmic Foundations of Differential Privacy}.
\newblock Now Publishers Inc., 2014.

\bibitem[DRV10]{dwork_boosting_2010}
Cynthia {Dwork}, Guy~N. {Rothblum}, and Salil {Vadhan}.
\newblock Boosting and differential privacy.
\newblock In {\em FOCS}, pages 51--60, 2010.

\bibitem[Dud99]{dudley_uniform_1999}
Richard~M. Dudley.
\newblock {\em Uniform Central Limit Theorems}.
\newblock Cambridge University Press, 1999.

\bibitem[Dwo06]{Dwork2006DP}
Cynthia Dwork.
\newblock Differential privacy.
\newblock In {\em ICALP}, pages 1--12, 2006.

\bibitem[ENU20]{edmonds_power_2020}
Alexander Edmonds, Aleksandar Nikolov, and Jonathan Ullman.
\newblock The power of factorization mechanisms in local and central
  differential privacy.
\newblock In {\em STOC}, page 425–438, 2020.

\bibitem[FX14]{feldman_sample_2014}
Vitaly Feldman and David Xiao.
\newblock Sample complexity bounds on differentially private learning via
  communication complexity.
\newblock In {\em COLT}, pages 1--20, 2014.

\bibitem[GHM19]{gonen_private_2019}
Alon Gonen, Elad Hazan, and Shay Moran.
\newblock Private learning implies online learning: An efficient reduction.
\newblock In {\em NeurIPS}, pages 8702--8712, 2019.

\bibitem[GKM20]{ghazi_differentially_2020}
Badih Ghazi, Ravi Kumar, and Pasin Manurangsi.
\newblock Differentially private clustering: Tight approximation ratios.
\newblock In {\em NeurIPS}, 2020.

\bibitem[HLM12]{hardt_simple_2012}
Moritz Hardt, Katrina Ligett, and Frank McSherry.
\newblock A simple and practical algorithm for differentially private data
  release.
\newblock In {\em NIPS}, pages 2339--2347, 2012.

\bibitem[HR10]{hardt2010multiplicative}
Moritz Hardt and Guy~N. Rothblum.
\newblock A multiplicative weights mechanism for privacy-preserving data
  analysis.
\newblock In {\em FOCS}, pages 61--70, 2010.

\bibitem[HT10]{hardt_geometry_2010}
Moritz Hardt and Kunal Talwar.
\newblock On the geometry of differential privacy.
\newblock In {\em STOC}, pages 705--714, 2010.

\bibitem[KLM{\etalchar{+}}20]{kaplan_privately_2020}
Haim Kaplan, Katrina Ligett, Yishay Mansour, Moni Naor, and Uri Stemmer.
\newblock Privately learning thresholds: Closing the exponential gap.
\newblock In {\em COLT}, pages 2263--2285, 2020.

\bibitem[KLN{\etalchar{+}}08]{kasiviswanathan2008what}
Shiva~Prasad Kasiviswanathan, Homin~K. Lee, Kobbi Nissim, Sofya Rashkodnikova,
  and Adam Smith.
\newblock What can we learn privately?
\newblock In {\em FOCS}, pages 531--540, 2008.

\bibitem[KMST20]{kaplan_private_2020}
Haim Kaplan, Yishay Mansour, Uri Stemmer, and Eliad Tsfadia.
\newblock Private learning of halfspaces: Simplifying the construction and
  reducing the sample complexity.
\newblock In {\em NeurIPS}, 2020.

\bibitem[KSS20]{kaplan_how_2020}
Haim Kaplan, Micha Sharir, and Uri Stemmer.
\newblock {How to Find a Point in the Convex Hull Privately}.
\newblock In {\em SoCG}, pages 52:1--52:15, 2020.

\bibitem[{Lit}87]{littlestone_learning_1987}
Nick {Littlestone}.
\newblock Learning quickly when irrelevant attributes abound: A new
  linear-threshold algorithm.
\newblock In {\em FOCS}, pages 68--77, 1987.

\bibitem[Mun00]{munkres_topology_2000}
J.R. Munkres.
\newblock {\em Topology}.
\newblock Featured Titles for Topology. Prentice Hall, Inc., 2000.

\bibitem[MY16]{moran_sample_2016}
Shay Moran and Amir Yehudayoff.
\newblock Sample compression schemes for {VC} classes.
\newblock {\em J. ACM}, 63(3), June 2016.

\bibitem[NBW{\etalchar{+}}18]{nissim_bridging_2016}
Kobbi Nissim, Aaron Bembenek, Alexandra Wood, Mark Bun, Marco Gaboardi, Urs
  Gasser, David~R. O{\textquoteright}Brien, and Salil Vadhan.
\newblock Bridging the gap between computer science and legal approaches to
  privacy.
\newblock {\em Harvard Journal of Law \& Technology}, 31:687--780, 2016 2018.

\bibitem[Nik15]{nikolov_improved_2015}
A.~Nikolov.
\newblock An improved private mechanism for small databases.
\newblock In {\em ICALP}, pages 1010--1021, 2015.

\bibitem[NRW19]{neel_heuristics_2019}
Seth Neel, Aaron Roth, and {Zhiwei Steven} Wu.
\newblock How to use heuristics for differential privacy.
\newblock In {\em FOCS}, pages 72--93, 2019.

\bibitem[NTZ12]{nikolov_geometry_2012}
Aleksandar Nikolov, Kunal Talwar, and Li~Zhang.
\newblock The geometry of differential privacy: the sparse and approximate
  cases.
\newblock {\em STOC}, pages 351--360, 2012.

\bibitem[Par14]{article29}
Article 29 Data Protection~Working Party.
\newblock Opinion 05/2014 on anonymisation techniques, 2014.

\bibitem[RK19]{kearns_ethical_2019}
Aaron Roth and Michael Kearns.
\newblock {\em The Ethical Algorithm: The Science of Socially Aware Algorithm
  Design}.
\newblock Oxford University Press, 2019.

\bibitem[RR10]{roth_interactive_2011}
Aaron Roth and Tim Roughgarden.
\newblock Interactive privacy via the median mechanism.
\newblock In {\em STOC}, page 765–774, 2010.

\bibitem[RS81]{reed_functional_1981}
M.~Reed and B.~Simon.
\newblock {\em I: Functional Analysis}.
\newblock Methods of Modern Mathematical Physics. Elsevier Science, 1981.

\bibitem[SB14]{shalev-shwartz_understanding_2014}
Shai Shalev{-}Shwartz and Shai Ben{-}David.
\newblock {\em Understanding Machine Learning - From Theory to Algorithms}.
\newblock Cambridge University Press, 2014.

\bibitem[{Sha}12]{shalev-shwartz_online_2012}
S.~{Shalev-Shwartz}.
\newblock {\em Online Learning and Online Convex Optimization}.
\newblock Foundations and Trends in Machine Learning, 2012.

\bibitem[Sio58]{sion_general_1958}
Maurice Sion.
\newblock On general minimax theorems.
\newblock {\em Pacific J. Math.}, 8(1):171--176, 1958.

\bibitem[Vad17]{vadhan2017complexity}
Salil Vadhan.
\newblock The complexity of differential privacy.
\newblock In {\em Tutorials on the Foundations of Cryptography}, pages
  347--450. Springer, 2017.

\bibitem[Vap98]{vapnik_statistical_1998}
Vladimir Vapnik.
\newblock {\em Statistical Learning Theory}.
\newblock Wiley-Interscience, 1998.

\end{thebibliography}
\end{document}